\documentclass[11pt]{article}

\usepackage{macros}
\bluehyperref
\usemedgeometry
\usepackage[numbers]{natbib}
\usepackage{xr}
\usepackage{overpic}

\usepackage{cleveref}

\usepackage{graphicx} 

\title{On Privately Estimating a Single Parameter}
\author{Hilal Asi$^1$ ~~~~ John C.\ Duchi$^2$ ~~~~ Kunal Talwar$^1$ \\
$^1$Apple ~~~~ $^2$Stanford University}

\begin{document}

\maketitle


\begin{abstract}
  We investigate differentially private estimators for individual
  parameters within larger parametric models.
  While generic private estimators exist, the estimators we provide repose
  on new local notions of estimand stability,
  and these notions allow procedures that provide private certificates
  of their own stability.
  By leveraging these private certificates, we provide computationally and
  statistical efficient mechanisms that release private statistics that are,
  at least asymptotically in the sample size, essentially unimprovable: they
  achieve instance optimal bounds.
  Additionally, we investigate the practicality of the algorithms both in
  simulated data and in real-world data from the American Community Survey
  and US Census, highlighting scenarios in which the new procedures
  are successful and identifying areas for future work.
\end{abstract}



\section{Introduction}

The challenges of privately estimating high-dimensional objects are myriad:
dimension dependent costs make private estimation of even simple models
notoriously challenging~\cite{DuchiJoWa18, SteinkeUl15, CaiWaZh21}; optimal
methods require sophisticated algorithmic strategies and
analyses~\cite{DworkRo14,AsiDu20}; the practicality of the methods can be
dubious~\cite{AsiDu20, DuchiHaKu23, ChadhaDuKu24}; until recently, we did
not even have methods that could computationally efficiently estimate a mean
vector with error commensurate with the covariance of the observed
data~\cite{DuchiHaKu23, BrownHoSm23}. We instead take a complementary goal,
developing methodology for estimating a single parameter in a parametric
statistical model.
While this may seem pedestrian---how hard could it be to estimate a single
scalar?---such problems and questions of their efficiency motivate both
substantial applied work, where estimating a (single scalar) causal
treatment effect motivates hundreds of thousands of
studies~\cite{ImbensRu15}, as well as deep theoretical work delineating what
functionals can and cannot be estimated~\cite{BickelKlRiWe98,
  VanDerVaart02semiparametric, LaanRo11}.
Less prosaically, how can we expect any applied work to leverage the
insights of differential privacy if we cannot even efficiently estimate
a single parameter?

To set the stage, consider the classical M-estimation
problem~\cite{VanDerVaart98, HuberRo09}.  For a population $P$ on data points
$z \in \mc{Z}$, we wish to estimate the minimizer of the population
(expected) loss
\begin{equation*}
  \poploss(\theta) \defeq
  P \loss_\theta = \int \loss_\theta(z) dP(z),
\end{equation*}
where $\loss_\theta(z)$ measures the loss of the parameter $\theta$ on
observation $z$ and we use the empirical process notation that $P f = \int
f(z) dP(z)$.  Given a sample of $n$ observations $(z_1, \ldots, z_n)$ and
associated empirical distribution $P_n = \frac{1}{n} \sum_{i = 1}^n
\pointmass_{z_i}$ placing a point mass on each $z_i$, classical
M-estimators release $\theta(P_n) = \argmin_\theta \{P_n \loss_\theta\}$.
We augment this slightly to incorporate $\ell_2$-regularization around
a point $\theta_0 \in \R^d$, considering private release of
\begin{equation}
  \label{eqn:basic-m-estimator}
  \theta(P_n) \defeq \argmin \left\{P_n \loss_\theta
  + \frac{\lambdareg}{2} \ltwo{\theta - \theta_0}^2 \right\},
\end{equation}
where $\lambdareg \ge 0$. Taking as motivation estimating treatment effects or
other individual scalars,
we also carefully consider
estimating linear functionals
\begin{equation*}
  u^T \theta(P_n)
\end{equation*}
of the parameter, where $u$ is (without loss of generality) a unit vector.

We develop differentially private~\cite{DworkMcNiSm06, DworkKeMcMiNa06}
estimators for these tasks, adopting notation that is a bit different from
standard formulations but more convenient for estimation problems.
Let $\mc{P}_n$ denote the collection of probability measures supported on at
most $n$ points in $\mc{Z}$, where $P_n(\{z\}) \in \{0, 1/n, 2/n, \ldots,
1\}$; we can identify a sample $\{z_1, \ldots, z_n\}$ by its associated
empirical distribution $P_n$.
We say that two samples $P_n, P_n'$ are
\emph{neighboring} if they differ in only a single observation,
equivalently, that their variation distance satisfies
\begin{equation*}
  \tvnorm{P_n - P_n'} \defeq \sup_A |P_n(A) - P_n'(A)| \le \frac{1}{n}.
\end{equation*}
We develop mechanisms, meaning a randomized functions on $\mc{P}_n$,
satisfying
\begin{definition}
  A mechanism $M$ is
  $(\diffp,\delta)$-\emph{differentially private} if
  \begin{equation*}
    \P(M(P_n) \in A) \le e^\diffp \P(M(P_n') \in A) + \delta
  \end{equation*}
  for all neighboring empirical distributions $P_n, P_n'$ and all measurable
  sets $A$.
\end{definition}

For M-estimators of the form~\eqref{eqn:basic-m-estimator}, the
most natural seeming approach
is to add noise commensurate with the \emph{sensitivity},
or the modulus of continuity, of the statistic of interest with respect
to changes in the sample $P_n$. If we wish to release
a statistic $\theta(P_n)$, then we consider the local modulus of continuity
\begin{equation}
  \label{eqn:modulus-continuity}
  \modcont_\theta(P_n; k) \defeq \sup\left\{\ltwo{\theta(P_n)
    - \theta(P_n')} \mid n \tvnorm{P_n - P_n'} \le k\right\}
\end{equation}
of $\theta$ for the $\ell_2$-distance at $P_n$, where the supremum is taken
over samples $P_n' \in \mc{P}_n$ differing by at most $k$ observations from
$P_n$.
Adding noise scaling as $\modcont_\theta(P_n; 1)$ is, essentially,
the best we could possibly hope to achieve in private
estimation~\cite{AsiDu20}.
However, this local modulus is sensitive to the underying sample $P_n$, so
naively using it cannot work, which motivates Nissim et al.'s smooth
sensitivity~\cite{NissimRaSm07}.

Numerous other strategies for privately computing
M-estimators~\eqref{eqn:basic-m-estimator} exist, and
we touch briefly on a few here before turning to our own development.
\emph{Objective perturbation} strategies
add a random linear term to the objective~\eqref{eqn:basic-m-estimator}, and
they appear to be among the most practical private estimators, though their
adaptivity to particular problems is unclear~\cite{ChaudhuriMoSa11,
  RedbergKoWa23}.
Other general approaches for convex M-estimation include
(stochastic) gradient approaches, which perturb data within
a gradient descent method~\cite{BassilySmTh14, BassilyFeTaTh19},
enjoy some worst-case guarantees, but they also
do not appear adaptive to local stability~\eqref{eqn:modulus-continuity}.
\citet{AsiDu20} take a different approach and focus on
low-dimensional quantities, introducing the \emph{inverse sensitivity
mechanism}.
This mechanism is essentially instance optimal for releasing one-dimensional
quantities, but appears challenging to compute except in certain special
cases, as more sophisticated problems require high-dimensional
integrals.
Our investigation takes as a departing point insights from both
\citet{AsiDu20} and \citet{NissimRaSm07}, but then carefully investigates
particular stability properties inherent in M-estimators.

\subsection{Heuristic development, motivating approach, and main results}
\label{sec:heuristic-motivation}

To motivate all our following development, we begin with a quite heuristic
derivation of the estimators we develop, and this overview allows us to
highlight a few of the challenges along the way. The basic idea is
straightforward: we would like to add noise commensurate with the local
modulus~\eqref{eqn:modulus-continuity}.
A Taylor approximation provides the starting point for our heuristic.  Let
$P_n$ and $P_n'$ be neighboring samples satisfying $\tvnorm{P_n - P_n'} \le
\frac{1}{n}$, and let $\theta = \theta(P_n)$ and $\theta' = \theta(P_n')$ be
the associated empirical minimizers. Then via a Taylor approximation, we
should have
\begin{align*}
  0 = P_n' \dot{\loss}_{\theta'}
  + \lambdareg (\theta'- \theta_0)
  & = (P_n' - P_n) \dot{\loss}_{\theta'} + P_n \dot{\loss}_{\theta'}
  + \lambdareg(\theta' - \theta_0) \\
  & = (P_n' - P_n) \dot{\loss}_{\theta'} + (P_n\ddot{\loss}_\theta
  + \lambdareg I + E)
  (\theta' - \theta),
\end{align*}
where $E = o(\norm{\theta' - \theta})$ is an error matrix and $P_n
\dot{\loss}_\theta + \lambdareg (\theta - \theta_0) = 0$. Inverting, we
obtain
\begin{equation*}
  \theta(P_n') - \theta(P_n) = (P_n \ddot{\loss}_\theta
  + \lambdareg I + E)^{-1}
  (P_n - P_n') \dot{\loss}_{\theta'}
  = (P_n \ddot{\loss}_\theta + \lambdareg I)^{-1}
  (P_n - P_n') \dot{\loss}_{\theta'}
  + o(1/n),
\end{equation*}
where we cavalierly assumed $P_n \ddot{\loss}_\theta + \lambdareg I + E$ is
invertible and $(P_n - P_n') \dot{\loss}_{\theta'} = O(1/n)$.

Continuing with this heuristic motivation, we define the collection of
possible gradients
\begin{equation*}
  \mc{G} \defeq \left\{\dot{\loss}_\theta(z) \mid z \in \mc{Z},
  \theta \in \R^d \right\}.
\end{equation*}
Providing privacy relies on controlling the amount changing a single
example can modify a parameter of interest;
upon changing a single example, we have
$\theta(P_n') \approx \theta(P_n)
+ \frac{1}{n} (P_n \ddot{\loss}_\theta + \lambdareg I)^{-1} (g_0 - g_1)$
for gradients $g_0, g_1 \in \mc{G}$.
Thus, to within higher order error terms, the
most the parameter $\theta(P_n)$ may change
is the local \emph{parameter}
sensitivity
\begin{subequations}
  \label{eqn:both-sensitivities}
  \begin{equation}
    \label{eqn:non-directional-difference}
    \diffu(P_n)
    \defeq
    \frac{1}{n}
    \sup_{g_0, g_1 \in \mc{G}}
    \ltwo{(P_n \ddot{\loss}_{\theta(P_n)} + \lambdareg I)^{-1} (g_0 - g_1)}.
  \end{equation}
  If instead we wish to estimate the linear functional $u^T \theta(P_n)$,
  then the \emph{directional} sensitivity
  \begin{equation}
    \label{eqn:directional-difference}
    \diffu(P_n, u)
    \defeq
    \frac{1}{n} \sup_{g_0, g_1 \in \mc{G}}
    u^T (P_n \ddot{\loss}_{\theta(P_n)} + \lambdareg I)^{-1} (g_0 - g_1)
  \end{equation}
\end{subequations}
bounds the change. That is,
we should obtain the guarantees
\begin{equation*}
  \ltwo{\theta(P_n) - \theta(P_n')}
  \le
  \diffu(P_n)
  ~~ \mbox{and} ~~
  |u^T (\theta(P_n) - \theta(P_n'))| \le \diffu(P_n, u).
\end{equation*}
The sensitivies~\eqref{eqn:both-sensitivities} asymptotically capture
\emph{exactly} the amount that $\theta(P_n)$ or $u^T \theta(P_n)$ may change
in substituting a single example in $P_n$ when $n$ is large (that is, the
local sensitivity or local modulus of continuity); for example, for robust
regression (see Example~\ref{example:robust-regression} to come) with
covariate vectors $x$ drawn from any compact set, we have
\begin{equation*}
  \frac{\modcont_\theta(P_n; 1)}{\diffu(P_n)}
  \to 1
  ~~ \mbox{and} ~~
  \frac{\sup\{|u^T (\theta(P_n)  - \theta(P_n'))| ~\mbox{s.t.}~
    n \tvnorm{P_n - P_n'}
    \le 1\}}{\diffu(P_n, u)} \to 1
\end{equation*}
with probability $1$ as $n \to \infty$ under i.i.d.\ sampling.

Making these heuristics rigorous
will require privately certifying a lower
and upper bounds on the minimal (respectively, maximal)
eigenvalues
\begin{equation*}
  \lambdamin(P_n) \defeq \lambdamin(P_n \ddot{\loss}_{\theta(P_n)})
  ~~ \mbox{and} ~~
  \lambdamax(P_n) \defeq \lambdamax(P_n \ddot{\loss}_{\theta(P_n)}).
\end{equation*}
Our
sharpest insights thus revolve around developing conditions that guarantee
$P_n \ddot{\loss}_\theta$ itself provides explicit control over the error
matrix $E$, making an elegant and practical use case for the theory of
self-concordant functions~\cite{NesterovNe94, BoydVa04}.

We can summarize our approach as follows:
\textbf{(i)} Privately certify an estimate $\what{\lambda}$ of
$\lambdamin(P_n)$ that satisfies
$\what{\lambda} \le \lambdamin(P_n)$ with high probability.
\textbf{(ii)} Given such an estimate, demonstrate that the
sensitivities~\eqref{eqn:both-sensitivities} or a proxy for them are stable
to changes in $P_n$ when $\lambdamin(P_n)$ is far from 0, and that they
bound (respectively) the local sensitivities $\ltwos{\theta(P_n) -
  \theta(P_n')}$ and $|u^T \theta(P_n) - u^T \theta(P_n')|$.
Then \textbf{(iii)} we release $\theta(P_n)$ or $u^T \theta(P_n)$ with
additive Gaussian noise whose variance scales roughly with the
sensitivies~\eqref{eqn:both-sensitivities}.
To obtain $(\diffp,
\delta)$-differential privacy, we will essentially show the following: if we
wish to release $\theta(P_n)$, (effectively) release
\begin{subequations}
  \label{eqn:heuristic-release}
  \begin{equation}
    \what{\theta} \defeq
    \theta(P_n) + \normal\left(0, O(1) \frac{\log \delta^{-1}}{\diffp^2}
    \diffu(P_n)^2
    \right),
  \end{equation}
  and if we wish to instead release the linear functional
  $u^T \theta(P_n)$, then we (effectively) release
  \begin{equation}
    \what{T} \defeq u^T \theta(P_n) + \normal\left(0,
    O(1) \frac{\log \delta^{-1}}{\diffp^2} \diffu(P_n, u)^2\right).
  \end{equation}
\end{subequations}
Combining each of these steps with appropriate
privacy composition guarantees then gives a full procedure whose error
scales roughly as
\begin{equation*}
  \ltwobig{\what{\theta} - \theta(P_n)}^2
  = O(1) \frac{\log \delta^{-1}}{\diffp^2} \diffu(P_n)^2 \cdot d
\end{equation*}
or
\begin{equation*}
  |\what{T} - u^T \theta(P_n)| = O(1)
  \frac{\sqrt{\log \frac{1}{\delta}}}{\diffp}
  \diffu(P_n, u),
\end{equation*}
each holding with high probability.
Each of these is unimprovable~\cite{CaiWaZh21, AsiDu20}.
To the extent that we can achieve these---which we make precise
in the sequel---we obtain error scaling precisely with the local
sensitivity (modulus of continuity) of the parameter of interest.


\section{Preliminaries: loss classes and private mechanisms}
\label{sec:prelim-ideas}

We describe the classes of losses we study in
problem~\eqref{eqn:basic-m-estimator} and provide privacy building blocks.

\subsection{Loss classes of interest}

Recalling problem~\eqref{eqn:basic-m-estimator},
the smoothness and related properties of the loss function $\loss_\theta$
will determine much of the difficulty of the problems we consider.
We consider both general smooth losses and a more
nuanced perspective tied to generalized linear models.

\subsubsection{Generic smooth losses}

The first class of losses we consider are Lipschitzean of up to
second order. In particular, for each $z \in \mc{Z}$,
we assume that $\theta \mapsto \loss_\theta(z)$ is
$\lipobj$-Lipschitz, has $\lipgrad$-Lipschitz gradient, and $\liphess$-Lipschitz
Hessian, all with respect to the
$\ell_2$-norm, meaning (respectively) that
\begin{equation*}
  \ltwos{\dot{\loss}_\theta} \le \lipobj,
  ~~
  \opnorms{\ddot{\loss}_\theta} \le \lipgrad,
  ~~
  \opnorms{\ddot{\loss}_\theta - \ddot{\loss}_{\theta'}}
  \le \liphess \ltwo{\theta - \theta'},
\end{equation*}
where we leave the observation $z$ implicit. For $d$-dimensional problems,
we typically expect the scaling that $\lipobj \lesssim \sqrt{d}$, $\lipgrad
\lesssim d$, and $\liphess \lesssim d^{3/2}$.

We take two working examples, arising from typical applications in
robust statistics and estimation in which the data is in pairs
$z = (x, y) \in \R^d \times \mc{Y}$. Both are generalized linear
model losses, where
$\loss_\theta(z) = h(\<\theta, x\>, y)$ for a function $h$ convex in its
first argument, so that
\begin{equation*}
  \dot{\loss}_\theta(x, y) = h'(\<\theta, x\>, y) x,
  ~~ \mbox{and} ~~
  \ddot{\loss}_\theta(x, y) = h''(\<\theta, x\>, y) xx^T.
\end{equation*}
The Lipschitz constants then must scale with the $\ell_2$-diameter
of $x \in \mc{X}$, that is,
\begin{equation*}
  \lipobj = \sup_{\theta, x \in \mc{X}, y \in \mc{Y}}
  \ltwos{\dot{\loss}_\theta(x, y)}
  = \linf{h'(\cdot, \cdot)} \sup_{x \in \mc{X}} \ltwo{x}
\end{equation*}
and similarly
$\lipgrad = \linfs{h''} \sup_{x \in \mc{X}} \ltwo{x}^2$ and
$\liphess = \linfs{h^{(3)}} \sup_{x \in \mc{X}} \ltwo{x}^3$.


\begin{example}[Robust regression]
  \label{example:robust-regression}
  The standard approaches
  to robust regression~\cite{HuberRo09} use either absolute
  error or Huber's robust loss, neither of which is $\mc{C}^2$---making
  private estimation quite challenging---so we
  consider a smoother variant. Consider
  \begin{equation*}
    h(t) = \log(1 + e^t) + \log(1 + e^{-t}),
  \end{equation*}
  and for $x \in \R^d$ and $y \in \R$ define
  the loss $\loss_\theta(y \mid x) = h(y - \<x, \theta\>)$.
  We have
  $h'(t) = \frac{e^t - 1}{e^t + 1} \in (-1, 1)$,
  $0 < h''(t)
  = \frac{2 e^t}{(e^t + 1)^2} \le \half$,
  while
  $h'''(t) = \frac{2 e^t}{(e^t + 1)^2}
  - \frac{4 e^{2t}}{(e^t + 1)^3}
  = \frac{2 e^t - 2 e^{2t}}{(e^t + 1)^3} \in (-\frac{1}{5}, \frac{1}{5})$.
  Different settings of the domain $x \in \mc{X}$ yield
  different Lipschitz constants, assuming $y$ may take on any real value.
  When $\mc{X} = [-1, 1]^d$, we thus obtain
  \begin{equation*}
    \lipobj = \sqrt{d},
    ~~~
    \lipgrad = \frac{d}{2},
    ~~~
    \liphess = \linfs{h^{(3)}} \sup_{x \in \mc{X}} \ltwo{x}^3
    \approx .19245 \cdot d^{3/2}.
  \end{equation*}
  Taking $\mc{X}$ to be the $\ell_2$-ball of radius $r \sqrt{d}$
  gives $\lipobj = r \sqrt{d}$,
  $\lipgrad = \frac{r^2 d}{2}$,
  $\liphess < \frac{r^3 d^{3/2}}{5}$.
\end{example}

\begin{example}[Binary logistic regression]
  \label{example:logistic-regression}
  For binary logistic regression, we assume the data
  $(x, y) \in \R^d \times \{-1, 1\}$, and for
  \begin{equation*}
    h(t) = \log(1 + e^{-t})
    ~~ \mbox{define} ~~
    \loss_\theta(x, y) = h(y \<x, \theta\>).
  \end{equation*}
  For
  $\sigma(t) = \frac{1}{1 + e^t}$ we have
  $h'(t) = -\sigma(t) \in (-1, 0)$,
  $0 < h''(t) = \sigma(t) (1 - \sigma(t)) \le \frac{1}{4}$, and
  $h'''(t) = \sigma(t)(1 - \sigma(t)) (1 - 2 \sigma(t)) \in (-.0963, .0963)$.
  So as in Example~\ref{example:robust-regression},
  if we assume that $x \in [-r, r]^d$, then we have
  \begin{equation*}
    \lipobj = r \sqrt{d},
    ~~~
    \lipgrad = \frac{r^2 d}{4},
    ~~~
    \liphess \le \frac{r^3 d^{3/2}}{10}
  \end{equation*}
  for binary logistic regression.
\end{example}

\subsubsection{Quasi-self-concordant losses and generalized linear models}

Combining the generalized linear model setting with some mild restrictions
on the loss $h$ allows us to obtain stronger results.
To use our heuristic derivation in
Section~\ref{sec:heuristic-motivation} to guarantee privacy, we need to
provide \emph{and} privately certify
fairly precise control over the error matrix $E$. This suggests considering loss
functions whose second derivatives appropriately bound themselves or for
which second derivatives control the third derivatives, leading us
to consider families of (approximately) self-concordant
losses~\cite{BoydVa04, OstrovskiiBa21}.

Classical self-concordant functions~\cite{NesterovNe94, BoydVa04}
satisfy
\begin{equation*}
  |f'''(t)| \le 2 f''(t)^{3/2}
\end{equation*}
for all $t$. We will use variations on this classical condition, saying
\begin{subequations}
  \label{eqn:qsc}
  that a convex
  function $f : \R \to \R$ is
  \emph{$\concordantfunc$-quasi-self-concordant (q.s.c.)} if
  \begin{equation}
    \label{eqn:functional-self-bounding}
    f''(t) \hinge{1 - \concordantfunc(|s|)}
    \le f''(t + s) \le f''(t)(1 + \concordantfunc(|s|))
  \end{equation}
  for all $t, s \in \R$. In some cases, we only require the lower bound on the
  second derivative~\eqref{eqn:functional-self-bounding},
  so we say that $f : \R \to \R$ is
  \emph{$\concordantfunc$-lower q.s.c.} if
  \begin{equation}
    \label{eqn:lower-functional-self-bounding}
    f''(t) \hinge{1 - \concordantfunc(|s|)}
    \le f''(t + s)
  \end{equation}
  for all $t, s \in \R$.
\end{subequations}
For a loss function $h : \R \times \mc{Y} \to \R$, we say that $h$ (or the
induced loss $\loss_\theta(x, y) = h(\<\theta, x\>, y)$) is
$\concordantfunc$-quasi-self-concordant if $t \mapsto h(t, y)$ is for each
$y$.  Several important properties derive from these self concordance
definitions, including that self-concordance implies the
quasi-self-concordance condition~\eqref{eqn:functional-self-bounding}.
Before returning to Examples~\ref{example:robust-regression}
and~\ref{example:logistic-regression}, we thus collect a few properties of
quasi-self-concordance and self-concordance.
\begin{lemma}[Self-concordance properties]
  \label{lemma:self-concordance}
  The following properties hold.
  \begin{enumerate}[(i)]
  \item \label{item:quasi-self-concordant}
    If for some $c < \infty$,
    the function $f$ satisfies $|f'''(t)| \le c f''(t)$ for all $t$,
    then
    \begin{equation*}
      e^{-c|s|} f''(t) \le f''(t + s) \le e^{c|s|} f''(t)
    \end{equation*}
    for all $t, s$, and so is $\concordantfunc$-q.s.c.\ for
    $\concordantfunc(s) = e^{cs} - 1$, or for $\concordantfunc(s) = (e^c -
    1)s$ for $s \le 1$ and $\concordantfunc(s) = \infty$ otherwise.
  \item \label{item:self-concordant-hessian}
    Let $f$ be self-concordant with $t \in \dom f$. Then
    \begin{equation*}
      \frac{f''(t)}{(1 + |s| f''(t)^{1/2})^2}
      \le f''(t + s)
      \le \frac{f''(t)}{\hinge{1 - |s| f''(t)^{1/2}}^2},
    \end{equation*}
    where
    the lower bound holds when $t + s \in \dom f$.
  \item If $f$ is self-concordant, then
    it is $\concordantfunc$-q.s.c.\ with
    $\concordantfunc(s) = \hinge{1 - s \sup_t f''(t)^{1/2}}^{-2} - 1$.
  \end{enumerate}
\end{lemma}
\noindent
The proofs of these results are standard; we include them
in Appendix~\ref{sec:proof-self-concordance} for completeness.

We now revisit our examples above in the context of
quasi-self-concordance~\eqref{eqn:functional-self-bounding}.

\begin{example}[Robust regression with log loss;
    Example~\ref{example:robust-regression} continued]
  \label{example:robust-regression-sc}
  For robust regression with the log loss, recall that
  for $y \in \R$ we have $h(t, y) =
  \log(1 + e^{t - y}) + \log(1 + e^{y - t})$. Setting $\phi(t) = \log(1 +
  e^t) + \log(1 + e^{-t})$ yields
  \begin{equation*}
    \frac{\phi'''(t)}{\phi''(t)}
    = \frac{e^t - e^{2t}}{e^t (e^t + 1)}
    = -\frac{e^{2t} - e^t}{e^{2t} + e^t} \in [-1, 1].
  \end{equation*}
  Leveraging
  Lemma~\ref{lemma:self-concordance}.\eqref{item:quasi-self-concordant}, we
  thus obtain
  \begin{equation}
    \label{eqn:binary-logistic-sc}
    \begin{split}
      h''(t, y) \hinge{1 - |s|} \le
      e^{-|s|} h''(t, y)
      & \le h''(t + s, y) \\
      & \le
      e^{|s|} h''(t, y)
      \stackrel{(\star)}{\le}
      h''(t, y) \left(1 + (e - 1) |s|\right),
    \end{split}
  \end{equation}
  where inequality~$(\star)$ holds for $|s| \le 1$.
  Robust regression with the log loss
  is $\concordantfunc$-lower q.s.c.\ with
  $\concordantfunc(s) = 1 - e^{-s}$ and
  is $\concordantfunc$-q.s.c.\ with
  $\concordantfunc(s) = (e^s - 1)$,
  as $e^{-s} \ge 1 - (e^s - 1)$ for $s \ge 0$.
\end{example}

\begin{example}[Binary logistic regression;
    Example~\ref{example:logistic-regression} continued]
  \label{example:binary-logistic-sc}
  For binary logistic regression, we have $h(t, y) =
  \log(1 + e^{-t y})$. Then
  for the sigmoid function $\sigma(t) = 1 / (1 + e^t)$,
  the derivative calculations in Example~\ref{example:logistic-regression}
  give
  \begin{equation*}
    \frac{|h'''(t, y)|}{h''(t, y)} = |1 - 2 \sigma(t)| \le 1,
  \end{equation*}
  so the bounds~\eqref{eqn:binary-logistic-sc} hold
  as in Example~\ref{example:robust-regression-sc};
  logistic regression is $\concordantfunc$-q.s.c.\ with
  $\concordantfunc(s) = (e^s - 1)$.
\end{example}

\noindent
\citet{OstrovskiiBa21} give additional examples of self-concordant
loss functions satisfying the classical self-concordance definitions.

\subsection{Composition, test-and-release mechanisms, and privacy
  building blocks}
\label{sec:composition-test-release}

We record a few building block results on privacy that
form the basis for our guarantees in the sequel. The first instantiates
propose-test-release approaches~\cite{DworkLe09} and composition of
$(\diffp, \delta)$-differentially private mechanisms.  The application of
the results we present will be to estimate a quantity approximating the
local sensitivity of a statistic $\theta$ of interest, with a (private)
certificate that the quantity upper bounds the local sensitivity
$\modcont_\theta(P_n; 1)$; we can then release the statistic with noise added
commensurate to this bound, as in the
motivation~\eqref{eqn:heuristic-release}. We will use the shorthand
\begin{equation*}
  Z_0 \eqdiffp Z_1
\end{equation*}
to mean that $\P(Z_0 \in A) \le e^\diffp \P(Z_1 \in A) + \delta$
for any measurable sets $A$.

\subsubsection{Composition}

We begin with the basic composition bound, which considers
drawing a (private) random variable conditional on $P_n$, and then
conditional on this value, releasing another statistic.
To formalize this,
assume we have a random variable $W \sim \mu(\cdot \mid P_n)$
taking values in a set $\mc{W}$
and a (randomized) mechanism $M$ mapping $\mc{P}_n \times \mc{W}
\to \mc{T}$ for a target set $\mc{T}$,
where for each sample distribution $P_n$ there exists
a good set $G = G(P_n)$ for which
\begin{equation*}
  \P(M(P_n, w) \in A) \le
  e^\diffp \P(M(P_n', w) \in A) + \delta
\end{equation*}
for all $w \in G(P_n)$.  We have the following minor extension of standard
compositional guarantees~\cite{DworkRo14} (which require that the good set
is the full space, $G(P_n) = \mc{W}$).
\begin{lemma}
  \label{lemma:conditional-composition}
  Assume that $W$ is $(\diffp_0, \delta_0)$-differentially private and $\P(W
  \in G(P_n) \mid P_n) \ge 1 - \gamma$ for all $P_n \in \mc{P}_n$.  Then
  the composed pair
  \begin{equation*}
    \left(M(P_n, W), W\right)
  \end{equation*}
  is $(\diffp + \diffp_0, \delta + \delta_0 + \gamma)$ differentially
  private.
\end{lemma}
\noindent
See Appendix~\ref{sec:proof-conditional-composition} for a proof
of this lemma.

As a typical application
of Lemma~\ref{lemma:conditional-composition}, we demonstrate a Gaussian
mechanism. Letting $\Phi(\cdot)$ denote the standard normal c.d.f.,
define the $(\diffp, \delta)$-variance
\begin{equation}
  \label{eqn:private-normal-variance}
  \varpriv
  \defeq \inf \left\{\sigma^2 \mid \Phi(-\sigma \diffp - 1/2\sigma)
  + \Phi(-\sigma \diffp + 1/2\sigma) \le \delta \right\}.
\end{equation}
As $\Phi(-t) \le e^{-t^2/2}$ for $t \ge 0$,
it suffices to choose
$\sigma$ large enough that
$-\sigma \diffp + \frac{1}{2 \sigma} \le - \sqrt{2 \log\frac{2}{\delta}}$,
so solving the quadratic in $\sigma$ guarantees
\begin{equation*}
  \sigma(\diffp, \delta)
  \le
  \sigma_{\textup{naive}}(\diffp, \delta) \defeq
  \frac{\sqrt{2 \log \frac{2}{\delta}}}{2 \diffp}
  + \frac{\sqrt{2 \log \frac{2}{\delta} + 2 \diffp}}{2 \diffp}.
\end{equation*}
But as $\Phi(-t) \asymp \frac{1}{t} e^{-t^2/2}$ for $t$ large,
the formulation~\eqref{eqn:private-normal-variance} is tighter.

The following lemma, which we prove for completeness in
Appendix~\ref{sec:proof-private-normal-variance}, shows that this quantity
is sufficient to guarantee privacy (see also~\cite{DworkKeMcMiNa06}).
\begin{lemma}
  \label{lemma:private-normal-variance}
  Let $\Phi(\cdot)$ denote the standard normal c.d.f., and
  let $\mu_0, \mu_1 \in \R^d$ and
  $\Delta^2 \ge \ltwo{\mu_0 - \mu_1}^2$.
  Then $Z_i \sim \normal(\mu_i, \Delta^2
  \sigma^2(\diffp, \delta) I_d)$, $i = 0, 1$,
  satisfy $Z_0 \eqdiffp Z_1$.
\end{lemma}

Now let $f : \mc{P}_n \to \R^d$ be a function of interest, and
let $W$ be an $(\diffp_0, \delta_0)$-differentially private estimate of the
local modulus $\modcont_f(P_n; 1)$ satisfying
$\P(W \ge \modcont_f(P_n; 1) \mid P_n) \ge 1 - \gamma$ for
all sample distributions $P_n$.
Define
the mechanism
\begin{equation*}
  M(P_n) = f(P_n) + \normal(0, W^2 \cdot \sigma^2(\diffp, \delta) I_d).
\end{equation*}
Then the following observation is an immediate consequences of
Lemmas~\ref{lemma:conditional-composition}
and~\ref{lemma:private-normal-variance}.
\begin{observation}
  \label{observation:comp-comp}
  The mechanism $M(P_n)$ above is
  $(\diffp + \diffp_0, \delta + \delta_0 + \gamma)$-differentially
  private.
\end{observation}

In our most sophisticated functional estimation problems, we will
require a bit more subtlety in the closeness of Gaussian distributions;
we defer such discussion until then.

\subsubsection{Test and release}
\label{sec:test-release}

Lemma~\ref{lemma:conditional-composition} allows us to present variants
of the test and release framework~\cite{DworkLe09}, which privately tests
that a sample $P_n$ is ``good enough,'' then uses a separate mechanism that
is private on ``good'' samples.
Thus, consider two mechanisms:
the first, $M_0$, computes a statistic
$(\diffp_0, \delta_0)$-differentially privately.
If $M_0(P_n)$ satisfies some condition, we execute
$M_1(P_n)$.
We make the following abstract assumption:
\begin{enumerate}
  \renewcommand{\theenumi}{A.\arabic{enumi}}
\item \label{item:deterministic-good-set}
  There is a
  statistic $\lambda : \mc{P}_n \to \Lambda$
  and a (deterministic) good set $G \subset \Lambda$ such
  that if $\lambda(P_n) \in G$, then
  $M_1(P_n) \eqdiffp M_1(P_n')$ for all $P_n'$ neighboring $P_n$.
\item \label{item:accurate-acceptance-set}
  There is an acceptance set $A$ such that
  if $\lambda(P_n) \not \in G$, then
  $\P(M_0(P_n) \in A) \le \delta_0$.
\end{enumerate}
Loosely, we have
the probabilistic implication that
whenever $M_0(P_n) \in A$, excepting an event with probability
$\delta_0$, the statistic $\lambda(P_n) \in G$.
We instantiate the following test/release scheme:
\algbox{
  \label{alg:test-release}
  The Test/Release Scheme
}{%
  \textbf{Require:}
  $(\diffp_0, \delta_0)$
  and $(\diffp, \delta)$-differentially private
  mechanisms $M_0$ and $M_1$ satisfying
  Assumptions~\ref{item:deterministic-good-set}
  and~\ref{item:accurate-acceptance-set},
  along with associated acceptance set $A$.
  \begin{enumerate}[i.]
  \item Release $M_0(P_n)$.
  \item If $M_0(P_n) \in A$, then
    release $M_1(P_n)$. Otherwise,
    release $\perp$.
  \end{enumerate}
}

\noindent
Let $M(P_n)$ be the final output of the procedure~\ref{alg:test-release},
Then in
Appendix~\ref{sec:proof-test-release},
we provide a proof of the following guarantee that $M(P_n)$ is private.
\begin{lemma}
  \label{lemma:test-release}
  The mechanism $M(P_n)$ is $(\diffp_0 + \diffp, e^{\diffp_0} \delta_0
  + \delta)$-differentially private.
\end{lemma}


\section{Parameter release algorithms for GLMs}
\label{sec:main-release}

We present the two main algorithms that apply our ideas, first to release
the full vector $\theta(P_n)$, and second to release individual linear
functionals $u^T \theta(P_n)$. The latter is our main interest, but the
former exhibits the same techniques.  As we outline in
Section~\ref{sec:prelim-ideas}, this consists of two phases: first, we
privately release a (putative) lower bound $\what{\lambda}$ on
$\lambdamin(P_n)$, which is both accurate and differentially private.  Given
such an estimate $\what{\lambda}$, we can then use (recall the
definition~\eqref{eqn:non-directional-difference}) that $\ltwos{\theta(P_n)
  - \theta(P_n')} \le (1 + o(1)) \frac{2 \lipobj}{\lambdamin(P_n) +
  \lambdareg}$
to release an estimate $\what{\theta}$ with appropriate
noise.  We use this approach in Section~\ref{sec:intro-release-full-params};
in the subsequent Section~\ref{sec:intro-release-individual}, we extend the
ideas to release individual parameters. In
Section~\ref{sec:dimension-dependence}, we discuss the implied dimension
dependence and accuracy guarantees of the main results here, especially in
relation to the underlying geometry of the data, with some commentary on
optimality as well.

To develop the ideas, we focus on generalized linear model losses $\loss$
that are $\concordantfunc$-quasi-self-concordant, meaning that
$\loss_\theta(x, y) = h(x^T \theta, y)$, where $h$ satisfies
inequality~\eqref{eqn:qsc}.
We make a few restrictions to allow concrete algorithms,
tacitly assuming these throughout this section:
we require that for a constant
$\selftwoconst \ge 0$ and $\rho \in (0, 1)$,
the self-bounding functional $\concordantfunc$ satisfies
\begin{equation}
  \label{eqn:bound-sc-by-linear}
  \concordantfunc(t) \le \selftwoconst t
  ~~ \mbox{if} ~~
  t \le \frac{1 - \rho}{\selftwoconst}.
\end{equation}
Recalling Examples~\ref{example:robust-regression-sc} (robust regression)
and~\ref{example:binary-logistic-sc} (binary logistic regression), the
choice $\concordantfunc(t) = (e^t - 1)$ guarantees
inequality~\eqref{eqn:bound-sc-by-linear} holds whenever $t = \frac{1 -
  \rho}{\selftwoconst}$ satisfies $e^t - 1 \le 1 - \rho$; the choices
$\alpha = 1.2332 \le 1.234$ and $\rho = \half$ suffice.

\subsection{Releasing full parameter vectors}
\label{sec:intro-release-full-params}

We preview the stability guarantees we
prove in Section~\ref{sec:parameter-stability}.
Let $0 \le \selftwoconst < \infty$ and
$\rho \in (0, 1)$ be the constants in
the linear bound~\eqref{eqn:bound-sc-by-linear} on the self-concordance
function $\concordantfunc$.  Define the condition
\begin{equation}
  \tag{C1}
  \label{eqn:lambda-self-bounded-big-enough}
  \lambdamin(P_n)
  + \frac{1}{\rho} \lambdareg
  \ge \frac{4 \lipobj \selftwoconst \radius(\mc{X})}{
    \rho (1 - \rho) n} + \frac{\lipgrad}{\rho n},
\end{equation}
which guarantees that $\lambdamin(P_n)
+ \lambdareg$ is large enough to certify
stability: as a consequence of
Proposition~\ref{proposition:advanced-parameter-self-bounding} in
Sec.~\ref{sec:qsc-glm-stability}, we have
\begin{corollary}
  \label{corollary:advanced-parameter-self-bounding}
  Let condition~\eqref{eqn:lambda-self-bounded-big-enough} hold.
  Define
  \begin{equation}
    \label{eqn:t-param-change}
    \tparamchange(\lambda)
    \defeq \frac{1 - \sqrt{1 - \frac{8 \selftwoconst \radius(\mc{X})
          \lipobj}{\lambda n}}}{2 \selftwoconst \radius(\mc{X})}.
  \end{equation}
  Then for any neighboring samples $P_n$ and $P_n'$ and
  $0 \le \lambda \le \lambdamin(P_n) + \lambdareg$, we have
  \begin{equation*}
    \ltwo{\theta(P_n) - \theta(P_n')} \le \tparamchange(\lambda).
  \end{equation*}
\end{corollary}

So any guarantee that the (regularized) minimal eigenvalue $\lambdamin(P_n)
+ \lambdareg \ge \lambda$ implies a stability guarantee on the parameters
via the parameter change bound $\tparamchange(\lambda)$
equation~\eqref{eqn:t-param-change} defines.  The bound satisfies
$\tparamchange(\lambda) \le \frac{3 \lipobj}{n \lambda}$ under
Condition~\eqref{eqn:lambda-self-bounded-big-enough} (see the discussion
following Proposition~\ref{proposition:advanced-parameter-self-bounding}),
it is monotonically decreasing in $\lambda$, and satisfies the asymptotic
$\tparamchange(\lambda) = \frac{2 \lipobj}{n \lambda}(1 + o(1))$ as $n \to
\infty$.  The bound in
Corollary~\ref{corollary:advanced-parameter-self-bounding} is thus sharp, in
that for large $n$ it converges to local
sensitivity~\eqref{eqn:non-directional-difference} whenever the gradients
$\mc{G}$ are a scaled $\ell_2$-ball.

To leverage Corollary~\ref{corollary:advanced-parameter-self-bounding}
and the test/release framework (Alg.~\ref{alg:test-release}), we
thus seek to privately release the minimal eigenvalue $\lambdamin(P_n)$. For
this, we develop a new family of techniques for releasing parameters whose
stability one can control recursively.  Deferring the full development to
Section~\ref{sec:private-recursions}, let $\recurse$ be a
``recursion'' function satisfying the following: for
some nonnegative quantity $\lambda(P_n)$, we have the bound
\begin{equation*}
  \lambda(P_n') \ge \recurse(\lambda(P_n))
  ~~ \mbox{and} ~~
  |\lambda(P_n') - \lambda(P_n)| \le \lambda(P_n) - \recurse(\lambda(P_n)),
\end{equation*}
for all neighboring $P_n, P_n'$,
so that $\recurse$ lower bounds $\lambda(P_n')$ and
upper bounds the change in the parameter of interest: it provides
a (local) guarantee of stability of $\lambda(P_n)$.
For the $N$-fold composition $\recurse^N$ of $\recurse$,
we can
calculate the smallest $N$ yielding
$\recurse^N(\lambda(P_n)) = 0$; this value $N$ is stable with respect
to the sample $P_n$. By releasing a privatized variant $\what{N}$ of $N$
and inverting the recursion, we may then release a private version
of the quantity $\lambda(P_n)$ of interest.

To work in the context of GLMs, define the recursion
\begin{equation}
  \label{eqn:self-bounded-lambda-recursion}
  \recurse(\lambda)
  \defeq \begin{cases}
    \lambda \hinge{1 - \concordantfunc\left(\radius(\mc{X})
      \tparamchange(\lambda + \lambdareg)\right)}
    - \frac{\lipgrad}{n}
    & \mbox{if~} \lambda ~
    \mbox{satisfies~\eqref{eqn:lambda-self-bounded-big-enough}} \\
    0 & \mbox{otherwise}.
  \end{cases}
\end{equation}
Corollary~\ref{corollary:eigenvalue-qsc-change} in
Section~\ref{sec:eigenvalue-stability} to come shows this recursion provides
the stability guarantee that $\recurse(\lambdamin(P_n)) \le
\lambdamin(P_n')$ for any $P_n'$ neighboring $P_n$. The following
algorithm instantiates our discussion
and releases (with high-probability) a lower bound on
$\lambdamin(P_n)$.

\algbox{
  \label{alg:self-bounded-lambda-release}
  Privately lower bounding $\lambdamin(P_n)$ for
  quasi-self-concordant GLMs.
}{
  \textbf{Require:} A $\concordantfunc$-quasi-self-concordant loss
  where $\concordantfunc$ locally satisfies the
  linear upper bound~\eqref{eqn:bound-sc-by-linear},
  privacy parameters $\diffp \ge 0$ and $\delta \in (0, 1)$
  
  \begin{enumerate}[i.]
  \item Set the recursion $\recurse$ as
    in~\eqref{eqn:self-bounded-lambda-recursion}.
  \item Set
    \begin{equation*}
      \what{N} \defeq \min\left\{N \in \N
      \mid \recurse^N(\lambdamin(P_n))  = 0 \right\}
      + \frac{1}{\diffp} \laplace(1).
    \end{equation*}
  \item Set $k(\diffp, \delta) = \frac{1}{\diffp} \log \frac{1}{2 \delta}$,
    then return $\what{N}$ and
    \begin{equation*}
      \what{\lambda} \defeq
      \sup\left\{\lambda \ge 0
      \mid \recurse^{\hinges{\what{N} - k(\diffp, \delta)}}(\lambda) = 0 \right\}.
    \end{equation*}
  \end{enumerate}
}
\noindent
Proposition~\ref{proposition:good-lambda}
in Section~\ref{sec:private-one-dim-recursions} then implies the
following privacy guarantee.
\begin{corollary}
  \label{corollary:lambda-qsc-accuracy-privacy}
  Let the loss $\loss$ be $\concordantfunc$-q.s.c.\ for $\concordantfunc(t)
  = (e^t - 1)$. 
  Then Algorithm~\ref{alg:self-bounded-lambda-release}
  is $\diffp$-differentially private,
  and with
  probability at least $1 - \delta$, $\what{\lambda}$ satisfies
  $\what{\lambda} \le \lambdamin(P_n)$. Additionally,
  there exists a numerical constant $C < \infty$ such that
  if $C \frac{\lipobj \radius(\mc{X})}{n}
  \le \lambdamin(P_n) + \lambdareg$,
  \begin{equation*}
    \what{\lambda} \ge \lambdamin(P_n) - O(1)
    \frac{1}{\diffp} \log \frac{1}{\delta}
    \left[\frac{\lipobj \radius(\mc{X})}{n}
      \frac{\lambdamin(P_n)}{\lambdamin(P_n) + \lambdareg}
      + \frac{\lipgrad}{n} \right]
  \end{equation*}
  with the same probability.
\end{corollary}
\noindent
See Section~\ref{sec:proof-lambda-qsc-accuracy-guarantee} for a proof
of the corollary.

With Corollary~\ref{corollary:lambda-qsc-accuracy-privacy} in hand,
Corollary~\ref{corollary:advanced-parameter-self-bounding}
coupled with the privacy
composition results we enumerate in
Section~\ref{sec:composition-test-release}
(Lemma~\ref{lemma:private-normal-variance} and
Observation~\ref{observation:comp-comp}), this guarantees that
Algorithm~\ref{alg:self-concordant-glm-lambda-release} is private,
as the next theorem captures.

\algbox{
  \label{alg:self-concordant-glm-lambda-release}
  Local output perturbation for releasing
  $\theta(P_n)$
}{%
  \textbf{Require:}
  A $\concordantfunc$-quasi-self-concordant
  GLM loss $h : \R \times \mc{Y} \to \R$
  satisfying the self-bounding condition~\eqref{eqn:functional-self-bounding},
  privacy parameters $\diffp \ge 0$ and $\delta \in (0, 1)$

  \begin{enumerate}[i.]
  \item Let $\what{\lambda}$ be the output of
    Algorithm~\ref{alg:self-bounded-lambda-release}
  \item If $\what{\lambda} + \lambdareg = 0$, return
    $\perp$
  \item Otherwise, let $\sigma^2(\diffp, \delta)$ be the
    normal variance~\eqref{eqn:private-normal-variance}.
    Return
    \begin{equation*}
      \what{\theta} \defeq
      \theta(P_n) + \normal\left(0, \sigma^2(\diffp, \delta)
      \cdot \tparamchange^2(\what{\lambda} + \lambdareg) I_d\right).
    \end{equation*}
  \end{enumerate}
  
}

\begin{theorem}
  \label{theorem:self-concordant-release}
  The output $\what{\theta}$ of
  Alg.~\ref{alg:self-concordant-glm-lambda-release} is
  $(2\diffp, 2\delta)$-differentially private. Additionally,
  there exists a numerical constant $C < \infty$ such that
  if
  \begin{equation*}
    \lambdamin(P_n) + \lambdareg
    \ge C \left(\frac{1}{\diffp}
    \log \frac{1}{\delta} \cdot
    \left[\frac{\lipgrad}{n} +
      \frac{\lipobj \radius(\mc{X})}{n}\right]
    + \frac{\lipobj \radius(\mc{X})}{n}
    + \frac{\lipgrad}{n}\right),
  \end{equation*}
  then with probability at least $1 - \delta - \gamma$,
  \begin{equation*}
    \ltwobig{\what{\theta} - \theta(P_n)}
    \le C
    \frac{\lipobj}{n (\lambdamin(P_n) + \lambdareg)}
    \frac{1}{\diffp} \sqrt{\log\frac{1}{\delta}}
    \left[\sqrt{d} + \sqrt{\log\frac{1}{\gamma}}\right].
  \end{equation*}
\end{theorem}
\noindent
See Section~\ref{sec:proof-self-concordant-release} for the full proof.

\subsection{Releasing individual model parameters}
\label{sec:intro-release-individual}

One of our main desiderata is to release a single
coordinate of the vector $\theta(P_n)$, or, more generally,
to release
\begin{equation*}
  u^T \theta(P_n)
\end{equation*}
for a unit vector $u$. The key is that for different problem
geometries---relating to the gradient set $\mc{G} = \{\dot{\loss}_\theta(x,
y) \mid x \in \mc{X}, y \in \mc{Y}, \theta \in \R^d\}$---the minimal and
maximal eigenvalues $\lambdamin(P_n)$ and $\lambdamax(P_n)$ of
$P_n\ddot{\loss}_{\theta(P_n)}$ certify bounds on the stability of $u^T
\theta(P_n)$.  The following corollary
(Lemma~\ref{lemma:perturbation-with-errors} in the proof of
Proposition~\ref{proposition:advanced-parameter-self-bounding}) captures
this for self-concordant losses~\eqref{eqn:bound-sc-by-linear}.
\begin{corollary}
  \label{corollary:directional-modulus-bound}
  For $u \in \R^d$ and $\lambda \ge 0$, let $\diffu(P_n, u)$ be the
  directional difference~\eqref{eqn:directional-difference} and
  $\tparamchange(\lambda)$ be the parameter change
  bound~\eqref{eqn:t-param-change}. Define
  \begin{equation*}
    \gamma(\lambda) \defeq \selftwoconst \cdot \tparamchange(\lambda)
    \radius_2(\mc{X})
  \end{equation*}
  and
  \begin{equation}
    \label{eqn:directional-modulus-bound}
    \stdonedim \defeq
    \diffu(P_n, u)
    + \frac{2 \lipobj}{n (\lambdamin(P_n) + \lambdareg)}
    \cdot \frac{\gamma(\lambdamin(P_n) + \lambdareg)}{1
      - \gamma(\lambdamin(P_n) + \lambdareg)}
    \ltwo{u}.
  \end{equation}
  Then
  \begin{equation*}
    \left|u^T (\theta(P_n') - \theta(P_n))\right|
    \le \stdonedim.
  \end{equation*}
\end{corollary}

Corollary~\ref{corollary:directional-modulus-bound} allows us to use the
propose-test-release scheme to argue that releasing
\begin{equation*}
  u^T \theta(P_n) + \stdonedim \cdot Z
\end{equation*}
for a Gaussian $Z$ with variance scaling as $\frac{1}{\diffp^2} \log
\frac{1}{\delta}$ is private so long as we can privately certify that
$\lambdamin(P_n)$ is large enough and $\lambdamax(P_n)$ is small enough,
because these combined imply that the ratio
$\stdonedim / \stdonedim[P_n']$ is close to one whenever $P_n$ and
$P_n'$ are neighboring.

\subsubsection{Releasing the maximal eigenvalue}

To address that we must certify that $\lambdamax(P_n)$ is not too large, we
adapt Algorithm~\ref{alg:self-bounded-lambda-release} to release an
approximation to $\lambdamax(P_n)$, and follows here.  Recalling the
self-concordance function $\concordantfunc(t) \le \selftwoconst t$, for a
fixed value $\what{\lambda}$, we define the increasing recursion
\begin{equation}
  \label{eqn:increasing-recursion}
  \recurse_{\what{\lambda}}(\lambda) \defeq
  \min\left\{
  \lambda \cdot
  \left(
  1 + \concordantfunc(\tparamchange(\what{\lambda}) \cdot \radius(\mc{X}))
  \right)
  + \frac{\lipgrad}{n},
  \lipgrad \right\}.
\end{equation}
Then via a derivation and justification completely parallel to that we have
done for the lower eigenvalues, so that we find the smallest $N$ such that
$\recurse^N(\lambdamax(P_n)) \ge \lipgrad$ (recalling that the Lipschitz
constant $\lipgrad$ of the gradients upper bounds $\lambdamax(P_n)$), we
obtain that the following algorithm is $\diffp$-differentially private.

\algbox{
  \label{alg:upper-lambda-release} A private upper bound on $\lambdamax(P_n)$
}{%
  \textbf{Require:}
  A $\concordantfunc$-quasi-concordant loss where $\concordantfunc$
  locally satisfies the linear upper bound~\eqref{eqn:bound-sc-by-linear},
  privacy parameters $\diffp \ge 0$ and $\delta \in (0, 1)$,
  $\diffp$-differentially private estimate $\what{\lambda}_{\min}$
  satisfying $\what{\lambda}_{\min} \le \lambdamin(P_n)$ with
  probability at least $1 - \delta$.
  
  \begin{enumerate}[i.]
  \item Set the recursion $\recurse = \recurse_{\what{\lambda}_{\min} + \lambdareg}$
    as
    in~\eqref{eqn:increasing-recursion}.
  \item Set
    \begin{equation*}
      \what{N} \defeq \min\left\{N \in \N \mid \recurse^N(\lambdamax(P_n))
      \ge \lipgrad \right\} + \frac{1}{\diffp} \laplace(1).
    \end{equation*}
  \item Set $k(\diffp, \delta) = \frac{1}{\diffp} \log\frac{1}{2\delta}$,
    then return $\what{N}$ and
    \begin{equation*}
      \what{\lambda} =
      \min\left\{\inf\left\{\lambda \mid \recurse^{N - k(\diffp,\delta)}(\lambda)
      \ge \lipgrad \right\},
      \lipgrad\right\}.
    \end{equation*}
  \end{enumerate}
}

\begin{corollary}
  \label{corollary:lambda-max-qsc-guarantee}
  Let the loss $\loss$ be $\concordantfunc$-q.s.c.\ for $\concordantfunc(t)
  = e^t - 1$. Let $\what{\lambda}_{\min}$ be the output of
  Algorithm~\ref{alg:self-bounded-lambda-release}
  and
  $\what{\lambda}_{\max}$ be the output of
  Algorithm~\ref{alg:upper-lambda-release} with this input. Then
  the pair
  \begin{equation*}
    (\what{\lambda}_{\min}, \what{\lambda}_{\max})
  \end{equation*}
  is $(2 \diffp, \delta)$-differentially private and satisfies
  $\what{\lambda}_{\min} \le \lambdamin(P_n)$ and
  $\what{\lambda}_{\max} \ge \lambdamax(P_n)$ with probability
  at least $1 - 2 \delta$. Additionally,
  there exists a numerical constant $C < \infty$ such that if
  $C \frac{\lipobj \radius(\mc{X})}{n}
  \le  \lambdamin(P_n) + \lambdareg$
  then with the same probability
  \begin{equation*}
    \what{\lambda}_{\max}
    \le \lambdamax(P_n) + O(1) \frac{1}{\diffp} \log\frac{1}{\delta}
    \left[\frac{\lipobj \radius(\mc{X})}{n}
      \frac{\lambdamax(P_n)}{\lambdamin(P_n) + \lambdareg}
      + \frac{\lipgrad}{n} \right].
  \end{equation*}
\end{corollary}
\noindent
See Section~\ref{sec:proof-lambda-max-qsc-guarantee} for a proof of this
corollary.

\subsubsection{Releasing the linear functional}

Now that Algorithms~\ref{alg:self-bounded-lambda-release}
and~\ref{alg:upper-lambda-release} demonstrate that
accurately releasing minimal and maximal eigenvalues is possible,
we can provide a test-release scheme that first checks whether
one can certify that the local moduli of continuity
$\stdonedim$ are similar for $P_n'$ neighboring $P_n$, and then---assuming
they are---releases a noisy version of $u^T \theta(P_n)$.
The key are stability guarantees on the ratio
$\stdonedim / \stdonedim[P_n']$, which in turn imply
that $\normal(0, \stdonedim)$ and $\normal(0, \stdonedim[P_n'])$ are
appropriately close distributions so that
the release $u^T \theta(P_n) + \stdonedim \cdot Z$ is private.
These rely on a series of constants,
all implicitly dependent on the estimated
minimal eigenvalue $\lambda_0 \approx \lambdamin(P_n) + \lambdareg$
and maximal eigenvalue $\lambda_1 \approx \lambdamax(P_n) + \lambdareg$,
and that (to actually describe the algorithm) we define here:
\begin{equation}
  \label{eqn:all-the-ratio-constants}
  \begin{split}
    \tparamchange & \defeq \tparamchange(\lambda_0),
    ~~~
    r \defeq \radius_2(\mc{X}),
    ~~~ \annoyingconst
    \defeq \frac{\linf{h''}}{\hinge{1 - \selftwoconst \tparamchange}}
    \frac{r^2}{n \lambda_0},
    ~~~
    \gamma \defeq \selftwoconst r \cdot \tparamchange,
    ~~~
    \gamma' \defeq \selftwoconst r \cdot
    \tparamchange(\recurse(\lambda_0)) \\
    \simconst_1 & \defeq \frac{1}{\hinge{1 - \selftwoconst
        r \tparamchange}} - 1,
    ~~~
    \simconst_2 \defeq \frac{1}{n (1 - \annoyingconst)}
    \frac{\linf{h''}}{\hinge{1 - \selftwoconst r \tparamchange}},
    ~~~
    \kappa \defeq \frac{\lambda_1}{\lambda_0}.
  \end{split}
\end{equation}
Then Propositions~\ref{proposition:certifiable-ratio-bound}
and~\ref{proposition:certifiable-linf-ratio} in
Section~\ref{sec:modulus-ratio-bounds} imply the following corollary.

\begin{corollary}
  \label{corollary:ratio-stabilities}
  For $2 \le p \le \infty$, let the gradient set $\mc{G}_p \defeq \{g \in
  \R^d \mid \norm{g}_p \le d^{\frac{1}{p} - \half} \lipobj \}$ and
  $\recurse$ be the recursion~\eqref{eqn:self-bounded-lambda-recursion}.
  Then
  for $p = 2$,
  \begin{equation*}
    \left(1 + \kappa (\simconst_1 + \simconst_2 r)
      + \frac{\kappa \lambda_0}{\recurse(\lambda_0)}
      \frac{\gamma'}{1 - \gamma'}\right)^{-1}
    \le
    \frac{\stdonedim}{\stdonedim[P_n']}
    \le
    \frac{1 + \kappa \frac{\gamma}{1 - \gamma}}{
      1 - \kappa (\simconst_1 + \simconst_2 r)
    }.
  \end{equation*}
  For $p > 2$, let $d_p = d^{1 - 2/p}$. Then
  \begin{equation*}
    \left(
    \left(1 + \sqrt{d_p} {\simconst}_1 {\kappa}
      + \frac{2 {\simconst}_2 d_p}{{\lambda}_0}\right)
      + \frac{\sqrt{d_p} \kappa {\lambda}_0 }{\recurse({\lambda}_0)}
      \frac{{\gamma}'}{1 - {\gamma}'}\right)^{-1}
    \le
    \frac{\stdonedim}{\stdonedim[P_n']}
    \le \frac{1 + \sqrt{d_p} \kappa \frac{\gamma}{1 - \gamma}}{
      1 - \sqrt{d_p} {\simconst}_1 {\kappa}
      - 2 d_p {\simconst}_2 / {\lambda}_0}.
  \end{equation*}
\end{corollary}

We shall see that $(\frac{\stdonedim}{\stdonedim[P_n']})^2 - 1 \lesssim
\diffp / \log\frac{1}{\delta}$ and
$(\frac{\stdonedim[P_n']}{\stdonedim[P_n]})^2 - 1 \lesssim \diffp /
\log\frac{1}{\delta}$ is enough to guarantee that releasing $u^T \theta(P_n)
+ \normal(0, \varpriv \cdot \stdonedim^2)$ is private.  Let $\Phi^{-1}$
denote the standard inverse Gaussian cumulative distribution function, so
that $\Phi^{-1}(1 - \delta)^2 \le \log \frac{1}{\delta}$.  Recalling
the generalized linear models in Examples~\ref{example:robust-regression}
and~\ref{example:logistic-regression}, where
the radius of the covariate vectors
$x \in \mc{X}$ governs smoothness properties, we consider
two cases.
In the case that gradients belong to the $p =
2$-norm ball, we check that
\begin{subequations}
  \label{eqn:ratio-for-epsilon}
  \begin{equation}
    \label{eqn:ratio-for-epsilon-l2}
    \max\left\{
    \frac{1 + \kappa \frac{\gamma}{1 + \gamma}}{
      1 - \kappa(\simconst_1 + \simconst_2 \radconst)},
    1 + \kappa (\simconst_1 + \simconst_2 \radconst)
    + \frac{\lambda_1}{\recurse(\lambda_0)} \frac{\gamma'}{1 - \gamma'}
    \right\}^2
    - 1 \le 
    \frac{2 \diffp}{1 + \Phi^{-1}(1 - \delta/2)^2}.
  \end{equation}
  For $p > 2$, let $d_p = d^{1 - \frac{2}{p}}$ and
  check that
  \begin{equation}
    \label{eqn:ratio-for-epsilon-lp}
    \begin{split}
      & \max\left\{\frac{1 + \sqrt{d_p} \kappa \frac{\gamma}{1 + \gamma}}{
        1 - \sqrt{d_p} \kappa \simconst_1 - 2 d_p \simconst_2 / \lambda_0},
      1 + \sqrt{d_p} \kappa \simconst_1 + \frac{2 d_p \simconst_2}{
        \lambda_0}
      + \frac{\sqrt{d_p} \lambda_1}{\recurse(\lambda_0)}
      \frac{\gamma'}{1 - \gamma'}    
      \right\}^2 - 1 \\
      & \qquad\qquad\qquad\qquad\qquad\qquad\qquad\qquad
      \qquad\qquad\qquad\qquad
      \le \frac{2 \diffp}{1 + \Phi^{-1}(1 - \delta/2)^2}.
    \end{split}
  \end{equation}
\end{subequations}
With these definitions,
Algorithm~\ref{alg:release-u-t-theta} then
privately releases a version of $u^T \theta(P_n)$.

\algbox{
  \label{alg:release-u-t-theta} Releasing a one-dimensional statistic
}{%
  \textbf{Require:}
  A $\concordantfunc$-quasi-concordant loss where $\concordantfunc$
  locally satisfies the linear upper bound~\eqref{eqn:bound-sc-by-linear},
  privacy parameters $\diffp \ge 0$ and $\delta \in (0, 1)$

  \begin{enumerate}[i.]
  \item Let $\what{\lambda}_{\min}$ and $\what{\lambda}_{\max}$
    be the outputs of Algorithms~\ref{alg:self-bounded-lambda-release}
    and~\ref{alg:upper-lambda-release}, respectively

  \item If $\diffp$ fails to satisfy the appropriate
    inequality~\eqref{eqn:ratio-for-epsilon}
    return $T = \perp$.
    
  \item So long as the pair $\what{\lambda}_{\min}$ and $\lambdareg$
    satisfy
    condition~\eqref{eqn:lambda-self-bounded-big-enough},
    return
    \begin{equation*}
      T \defeq u^T \theta(P_n)
      + \normal\left(0, \varpriv \cdot \stdonedim^2\right)
    \end{equation*}
    Otherwise return $T = \perp$.
  \end{enumerate}
}

Then as a corollary of the main results in
Section~\ref{sec:releasing-linear-functionals}
(see Section~\ref{sec:proof-u-t-theta-private} for the proof),
we have the following result.
\begin{corollary}
  \label{corollary:u-t-theta-private}
  Let $\diffp \ge 0$ and $\delta \in (0, 1)$. Then $T$
  is $(3 \diffp, (1 + e^\diffp + e^{2 \diffp}) \delta)$-differentially
  private.
\end{corollary}
\noindent
Unpacking Corollary~\ref{corollary:u-t-theta-private}, we see the following:
as soon as we can guarantee the conditions~\eqref{eqn:ratio-for-epsilon}
hold, we can then release the actual statistic
$u^T \theta(P_n)$ with noise scaling precisely as the
(slightly enlarged)
local modulus of continuity~\eqref{eqn:directional-modulus-bound}.

In passing, we note a slight but practically important
improvement that we employ in our experiments.
Because Algorithm~\ref{alg:release-u-t-theta} performs
three private operations: two using Laplace mechanisms and
the last with a Gaussian, we can use
privacy loss random variables and explicit calculations
to programatically achieve sharper privacy bounds than
that in the corollary~\cite{GopiLeWu21}.

\subsection{Dimension dependence and commentary}
\label{sec:dimension-dependence}

While the algorithms we have presented are relatively straightforward to
implement, we still must discuss the dimension and sample-size scaling they
require and accuracy guarantees they provide, especially in the context of
the necessary dimension-dependent penalties privacy
enforces~\cite{SteinkeUl15, BarberDu14a, CaiWaZh21}. Let us focus on the
scaling that arises in Corollary~\ref{corollary:u-t-theta-private} for large
sample sizes $n$ (see also
Proposition~\ref{proposition:certifiable-linf-ratio} to come).  For
simplicity, we assume the self-bounding
constants~\eqref{eqn:bound-sc-by-linear} satisfy $\selftwoconst = O(1)$
(e.g., as in Examples~\ref{example:robust-regression-sc}
and~\ref{example:binary-logistic-sc}), and that the q.s.c.\ function $h$ has
$O(1)$-Lipschitz zeroth, first, and second derivatives. Then the problem
scaling all boils down to the $\ell_2$ radius $\radconst =
\radius_2(\mc{X})$ of the covariates $\mc{X}$, so that $\lipletter_i = O(1)
\cdot \radconst^{1 + i}$.  In this case,
Condition~\eqref{eqn:lambda-self-bounded-big-enough} becomes that
$\lambdamin(P_n) + \lambdareg \gtrsim \frac{\radconst^2}{n}$, while
letting $\lambda = \lambdamin(P_n)$ be shorthand for the smallest eigenvalue,
the
parameter change~\eqref{eqn:t-param-change} satisfies
\begin{equation*}
  \tparamchange(\lambda) = \frac{2 \lipobj}{\lambda n}(1 + o(1))
  \asymp \frac{\radconst}{\lambda n}
\end{equation*}
Substituting $\radconst = \radius_2(\mc{X})$ and the above values into the
constants~\eqref{eqn:all-the-ratio-constants}, we obtain $\simconst_1 \asymp
\frac{1}{1 - \radconst^2 /\lambda n} - 1 \asymp \frac{\radconst^2}{\lambda
  n}$, $\simconst_2 \asymp \frac{1}{n}$, and $\gamma \asymp
\frac{\radconst^2}{\lambda n}$.  Finally, we specialize a bit to the case
that $\mc{X}$ is contained in a scaled $\ell_p$-ball for some $2 \le p \le
\infty$.  In this case, condtion~\eqref{eqn:ratio-for-epsilon-lp}
essentially subsumes condition~\eqref{eqn:ratio-for-epsilon-l2}. Letting
$d_p = d^{1 - \frac{2}{p}}$ and $\kappa = \frac{\lambdamax(P_n) +
  \lambdareg}{ \lambdamin(P_n) + \lambdareg}$,
condition~\eqref{eqn:ratio-for-epsilon-lp} thus (for large $n$) becomes
equivalent to
\begin{equation*}
  \kappa \frac{\sqrt{d_p} \radconst^2}{\lambda n}
  \lesssim \frac{\diffp}{\log\frac{1}{\delta}}.
\end{equation*}
Once this scaling holds, then
Algorithm~\ref{alg:release-u-t-theta} releases
$T = u^T \theta(P_n) + \stdonedim \cdot \sigma(\diffp,\delta) \normal(0,1)$,
where we recall the definition~\eqref{eqn:directional-modulus-bound}
of $\stdonedim = \diffu(P_n, u) + O(1) \frac{d^{3/2}}{n^2}$,
the optimal scaling.
We summarize this with the following theorem.
\begin{theorem}
  Let the losses $\loss_\theta$ satisfy the smoothness conditions of
  Algorithm~\ref{alg:release-u-t-theta}, and assume that the covariate
  domain $\mc{X}$ has $\ell_2$-radius $\radconst = \radius_2(\mc{X})$. Then
  the output $T$ of Algorithm~\ref{alg:release-u-t-theta} is $(3 \diffp, (1 +
  e^\diffp + e^{2 \diffp}) \delta)$-differentially private. Let
  the notation above hold. Then additionally,
  there is a numerical constant $C < \infty$ such that if
  \begin{equation*}
    C \cdot \kappa \frac{\sqrt{d_p} \radconst^2}{(\lambdamin(P_n) + \lambdareg)
      n} \le \frac{\diffp}{\log \frac{1}{\delta}},
  \end{equation*}
  then with probability at least $1 - \delta - \gamma$,
  \begin{equation*}
    \left|T - u^T \theta(P_n)\right|
    \le C \, \frac{1}{n} \cdot
    \sup_{x \in \mc{X}} \left|u^T (P_n \ddot{\loss}_{\theta(P_n)} +
    \lambdareg I)^{-1} x \right| \cdot
    \frac{1}{\diffp} \sqrt{\log\frac{1}{\delta}}
    \sqrt{\log\frac{1}{\gamma}}.
  \end{equation*}
\end{theorem}

So we see a somewhat interesting behavior, which we believe is worth
investigating, though leave to future work: once the sample size is large
enough, then Algorithm~\eqref{alg:release-u-t-theta} releases $u^T
\theta(P_n)$ with noise scaling exactly (up to the higher-order term) as the
local modulus of continuity $\diffu(P_n, u)$. But until we have sufficient
sample size to dominate the dimension, the presented procedures are likely
impractical. There appear to be two condition-number-like quantities: the
actual condition number $\kappa = \frac{\lambdamax(P_n) +
  \lambdareg}{\lambdamin(P_n) + \lambdareg}$, and one relating the scale of
the covariates $\mc{X}$ to the curvature $\lambdamin$ of the problem, with a
dimension-dependent penalty.

Reifying the theorem by considering the
cases that $\mc{X}$ is an $\ell_2$-ball of radius $\sqrt{d}$ or the
hypercube $\{-1, 1\}^d$, we see that once the sample size is large
enough---i.e., it satisfies $\kappa d \ll n$ in the former case and $\kappa
d^{3/2} \ll n$ in the latter---we achieve optimal private
estimation.\footnote{We note in passing that these scalings appear to be
unimprovable using our analyses, though plausibly a much cleaner treatment
is possible.}
It would be interesting to understand if this dimensional scaling is
fundamental in some sense. In statistical problems in the ``high-dimensional
asymptotic'' regime that $d/n \to c$ for a constant $0 < c < \infty$ (or
even $d^2 / n \to 0$), certain interesting functionals are estimable, such
as the mean-squared error of a predictor~\cite{Dicker14}.  With privacy,
these questions appear to be subtle.

\section{Experiments}

We complement our theoretical work with experimental results that compare
the proposed algorithm to existing algorithms for privately estimating a
single parameter.
We consider two settings.
In the first, we evaluate the procedures here,
along with alternative private algorithms,
on a simulated robust regression dataset (as
in Example~\ref{example:robust-regression}), where we create a
synthetic dataset to allow us to test different aspects of the
algorithms here and their relationship with others.
In the second, we consider the Folktables dataset~\cite{DingHaMiSc21},
which consists of datasets derived from the US Census.
In each experiment, we consider five procedures:
\begin{enumerate}[1.]
\item Localized output perturbation (the methods in this paper).
  When estimating the entire parameter $\theta(P_n)$, this corresponds
  to Alg.~\ref{alg:self-concordant-glm-lambda-release}, while
  estimating the linear functional $u^T \theta(P_n)$ corresponds to
  Algorithm~\ref{alg:release-u-t-theta}.
\item \label{item:non-private-local}
  A non-private and idealized variant of the procedure above,
  where we release the parameter of interest with noise
  scaling as its local sensitivity, that is,
  \begin{equation*}
    \theta(P_n) + \normal(0, \diffu^2(P_n) \cdot \varpriv I_d)
    ~~ \mbox{or} ~~
    u^\top \theta(P_n) + \normal(0, \diffu^2(P_n, u) \cdot \varpriv I_d),
  \end{equation*}
  where $\diffu(P_n)$ and $\diffu(P_n, u)$ are the
  sensitivies~\eqref{eqn:both-sensitivities}
  and $\varpriv$ is
  the private variance~\eqref{eqn:private-normal-variance}.
\item Differentially private stochastic gradient descent
  (DP-SGD)~\cite{BassilySmTh14,BassilyFeTaTh19}, where
  we have non-privately searched for hyper-parameters to select
  batch sizes and total iterations.
\item Naive output perturbation, which for
  $\theta_{\lambdareg}(P_n) \defeq \argmin_\theta \{\poploss_n(\theta)
  + \frac{\lambdareg}{2} \ltwo{\theta}^2\}$ releases
  \begin{equation*}
    \theta(P_n) + \normal\left(0, \frac{4 \lipobj^2}{n^2 \lambdareg^2}
    \cdot \varpriv I_d\right),
  \end{equation*}
  as $\ltwos{\theta_{\lambdareg}(P_n) - \theta_{\lambdareg}(P_n')} \le
  \frac{2 \lipobj}{n\lambdareg}$ is trivially stable.
  We set
  $\lambdareg = 10^{-2}$.
\item Objective perturbation~\cite{ChaudhuriMoSa11} with
  optimized parameter settings~\cite{RedbergKoWa23}.
  For linear models (logistic or robust regression) as in this paper, this
  releases
  \begin{equation}
    \label{eqn:obj-pert}
    \what{\theta}(P_n)
    = \argmin_\theta \left\{\poploss_n(\theta)
    + W^T \theta
    + \frac{\lambdareg}{2} \ltwo{\theta}^2
    \right\},
  \end{equation}
  where $\lambdareg = \frac{4 \lipgrad}{n \diffp}$
  and $W \sim \normal(0, \sigma^2)$ for
  $\sigma^2 = \frac{2 \radius(\mc{X})}{n \diffp}
  \sqrt{2 \log \frac{4}{\delta}}
  + \sqrt{2 \diffp + 2 \log \frac{4}{\delta}}$.
  These choices guarantee $(\diffp, \delta)$-differential privacy.
\end{enumerate}

\subsection{Robust Regression (synthetic experiments)}

In our simulated data, we experiment with robust regression.
To generate data for the experiments, we fix a sample size $n$ and dimension
$d$, then generate $\theta\opt \sim r \uniform(\sphere^{d-1})$, varying
the radius $r = \ltwo{\theta\opt}$.
We sample $x_i \simiid \uniform[-1, 1]^d$, and draw
$y_i = \<\theta\opt, x_i\> + z_i$ for $z_i \simiid \sigma \cdot \laplace(1)$,
$i = 1, \ldots, n$.
With this setting, we consider either estimating
$\theta_1\opt$, the first coordinate of $\theta\opt$, or the
vector $\theta\opt$.
Each of our reported results corresponds to average results over 25 such
experiments, where we provide (approximate) 95\% confidence intervals
of $\pm 2$ standard errors.
We consider four distinct experimental settings:
(1) releasing eigenvalues,
(2) error in releasing the functional $e_1^T \theta(P_n)$
versus sample size $n$,
and (3) and (4) evaluating error versus privacy $\diffp$
in releasing $e_1^T \theta(P_n)$ or the full vector $\theta(P_n)$.

Because Algorithm~\ref{alg:release-u-t-theta} outputs $\perp$ when it cannot
certify $\lambdamin(P_n) > 0$, our first experiment investigates the
relative error $|\what{\lambda} - \lambdamin(P_n)| / \lambdamin(P_n)$ in the
eigenvalue Algorithm~\ref{alg:self-bounded-lambda-release} releases.
We show results in Figure~\ref{fig:eigenvalue-estimates},
varying the sample size ratio $n / d$ and for dimensions
$d = 5, 10, 20$ and fixing $\diffp = 1$, $\delta = 10^{-6}$.
This figure makes clear that, while the procedure eventually
achieves quite small error, the sample sizes necessary may be
quite large for most tasks---hence, in the sequel, our focus
on census data with relatively small numbers of covariates.
We will return to this point in the discussion, when we suggest
future work.

\begin{figure}
  \centering
  \begin{tabular}{cc}
    \hspace{-.5cm}
    \includegraphics[width=.75\columnwidth]{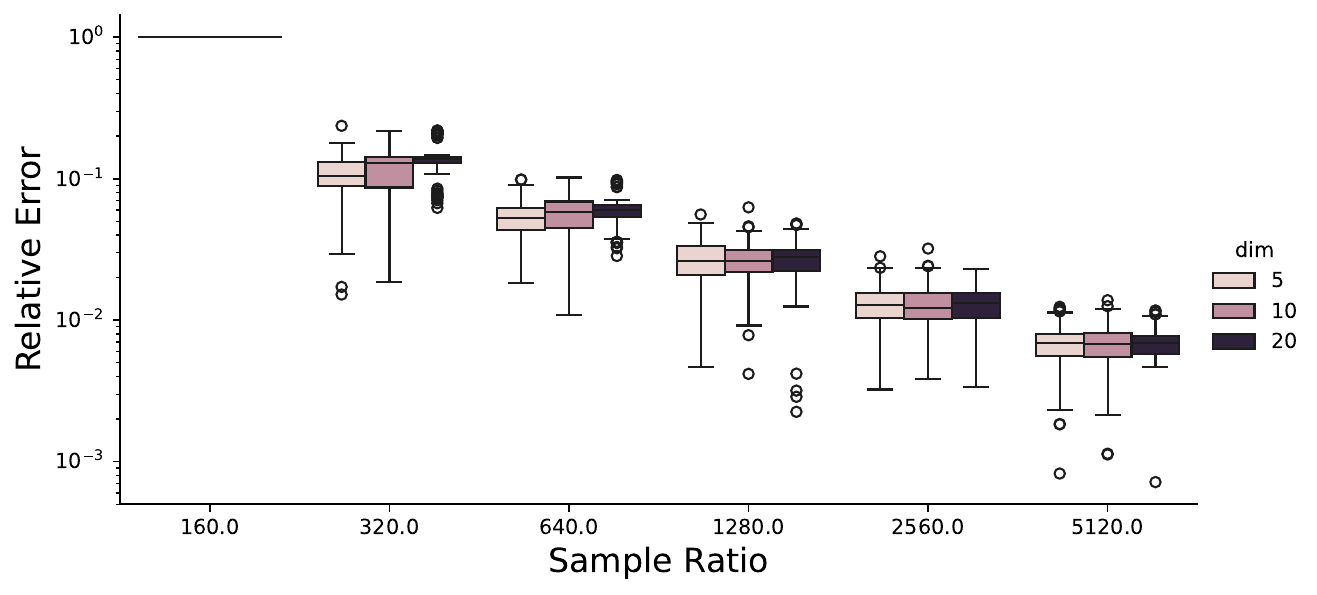} &
    \hspace{-1cm}
    \begin{minipage}{.3\columnwidth}
      \vspace{-6cm}
      \caption{\label{fig:eigenvalue-estimates}  Relative error
        $\frac{|\what{\lambda} - \lambdamin(P_n)|}{\lambdamin(P_n)}$
        in eigenvalue estimate versus sampling ratio $r = n / d$ for
        dimensions $d \in \{5, 10, 20\}$ on simulated robust regression. For
        ratios $r \le 160$, Alg.~\ref{alg:self-bounded-lambda-release}
        releases $\what{\lambda} = 0$.}
    \end{minipage}
  \end{tabular}
\end{figure}

For our second set of experiments, we fix $\diffp = 4$ to investigate the
scaling of errors with the sample size $n$ for estimating $e_1^T
\theta(P_n)$, the first coordinate of $\theta(P_n)$.
Based on our theoretical results and Fig.~\ref{fig:eigenvalue-estimates},
we expect that
Algorithm~\ref{alg:release-u-t-theta} should exhibit a type of thresholding
behavior: when $n$ is too small, we cannot certify that $\lambdamin(P_n)$
is large enough to guarantee stability, and so must release statistics
with substantial noise.
When $n$ is large enough, we expect to achieve error near that of the
non-private procedure adding noise scaling exactly as the
local sensitivity (see item~\ref{item:non-private-local} above).
While prior work suggests objective
perturbation~\eqref{eqn:obj-pert} should be a competitive and easy-to-use
algorithm~\cite{ChaudhuriMoSa11, RedbergKoWa23},
that it regularizes its parameter $\theta$ around 0 suggests that
there should be a gap between its performance and the methods here as
$\ltwo{\theta\opt}$ grows.
Figures~\ref{fig:sample-size-scaling}
and~\ref{fig:sample-size-scaling-5} show the results of these experiments
for different dimensions $d$;
the results are consistent with our expectations:
the non-private method has (by far) the best performance, while
Algorithm~\ref{alg:release-u-t-theta} eventually achieves errors
near the ``best possible'' non-private release.

\begin{figure}[ht]
  \begin{center}
    \includegraphics[width=.95\columnwidth]{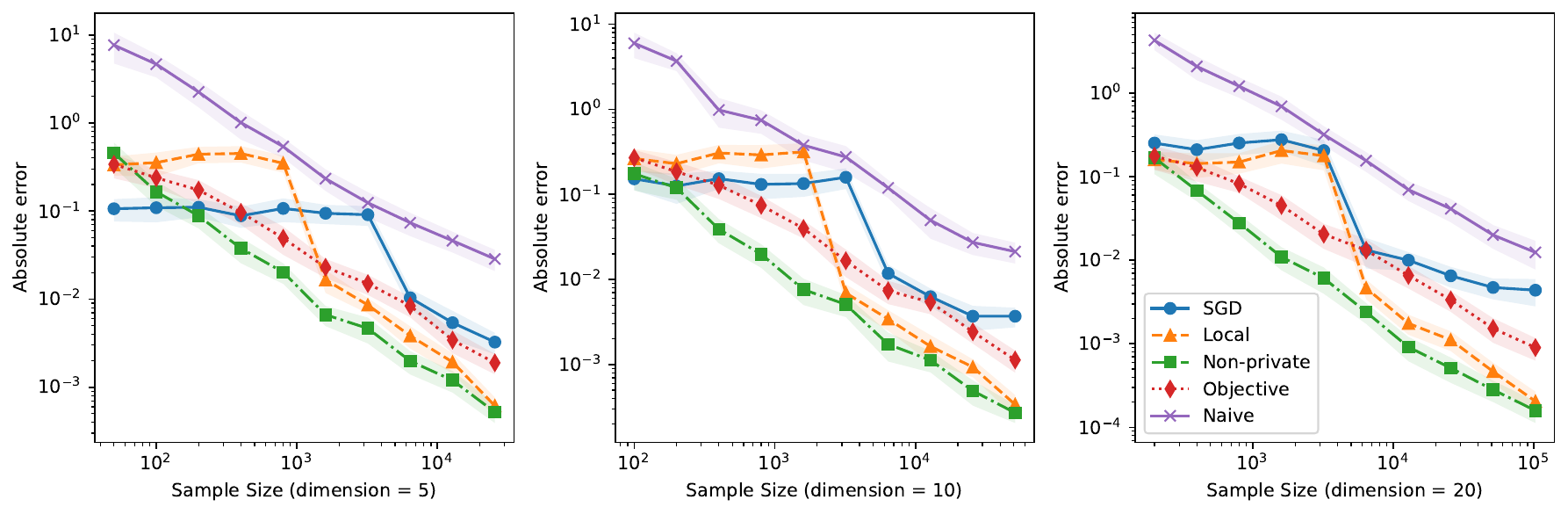}
    \caption{\label{fig:sample-size-scaling} Error $|\theta_1\opt -
      \what{\theta}_1(P_n)|$ in the first-coordinate of the target
      $\theta(P_n)$ versus sample size for varying
      dimensions $d = 5, 10, 20$ in a robust regression
      experiment, where $\ltwo{\theta\opt} =
      1$. The method Local (orange
      triangle) is Algorithm~\ref{alg:release-u-t-theta}; SGD is DPSGD,
      non-private is the non-private idealized version of the methods here
      (item~\ref{item:non-private-local}),
      objective is objective perturbation~\eqref{eqn:obj-pert},
      and naive is the naive output perturbation estimator. Objective
      perturbation and the methods here exhibit the best performance,
      with Alg.~\ref{alg:release-u-t-theta} exhibiting a noticeable
      improvement at sufficiently large sample size.}
  \end{center}
\end{figure}

\begin{figure}
  \begin{center}
    \includegraphics[width=.95\columnwidth]{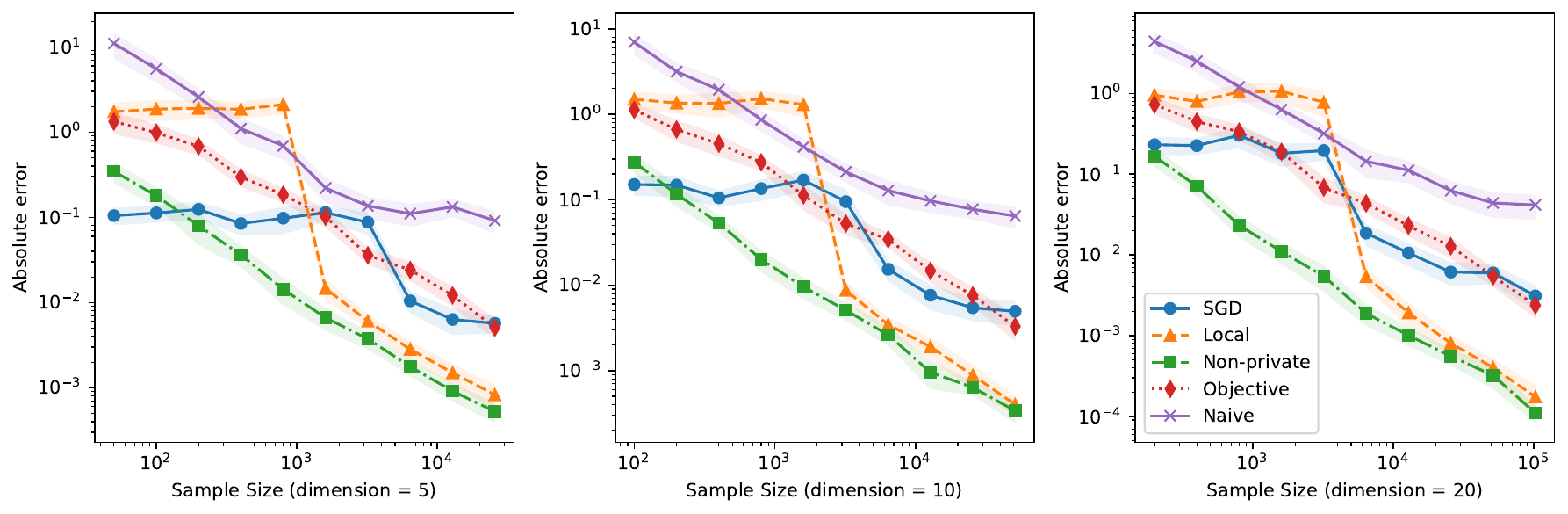}
    \caption{\label{fig:sample-size-scaling-5}
      Identical to Fig.~\ref{fig:sample-size-scaling}, except that
      $\ltwo{\theta\opt} = 5$. Note that
      the gap in performance between objective perturbation and
      Alg.~\ref{alg:release-u-t-theta} is larger than in the case
      that $\ltwo{\theta\opt} = 1$.}
  \end{center}
\end{figure}

Finally, Figures~\ref{fig:rob-reg-coord} and~\ref{fig:rob-reg-full}
investigate the error $\what{\theta} - \theta(P_n)$ for single coordinates
(Fig.~\ref{fig:rob-reg-coord}) and the full parameter vector $\theta(P_n)$,
respectively, as $\diffp$ increases, for fixed
sample size $n = 10^5$ and dimension $d = 10$.
The plots are consistent with our observations and expectations to
this point:
differentially private SGD and objective perturbation both exhibit
reasonable performance, but are worse than
the local release Algorithm~\ref{alg:release-u-t-theta}
for estimating a single coordinate.
On the other hand, objective perturbation is very competitive when
$\theta\opt$ is small, but has some degradation as the norm of $\theta\opt$
increases (plots (b) in each figure).

\begin{figure}
  \centering
  \begin{tabular}{cc}
    \begin{overpic}[width=.46\textwidth]{
        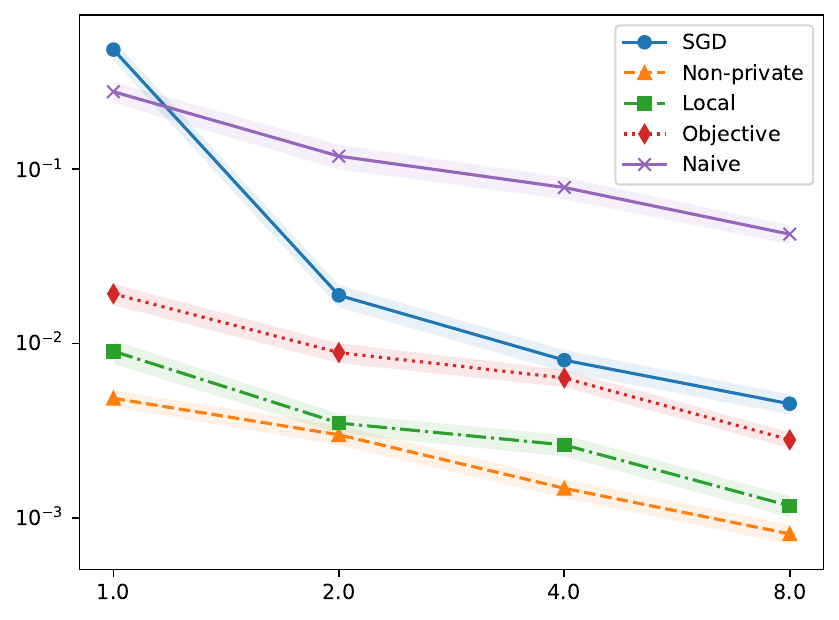}
      \put(50,0){$\diffp$}
    \end{overpic}
    &
    \begin{overpic}[width=.46\textwidth]{
        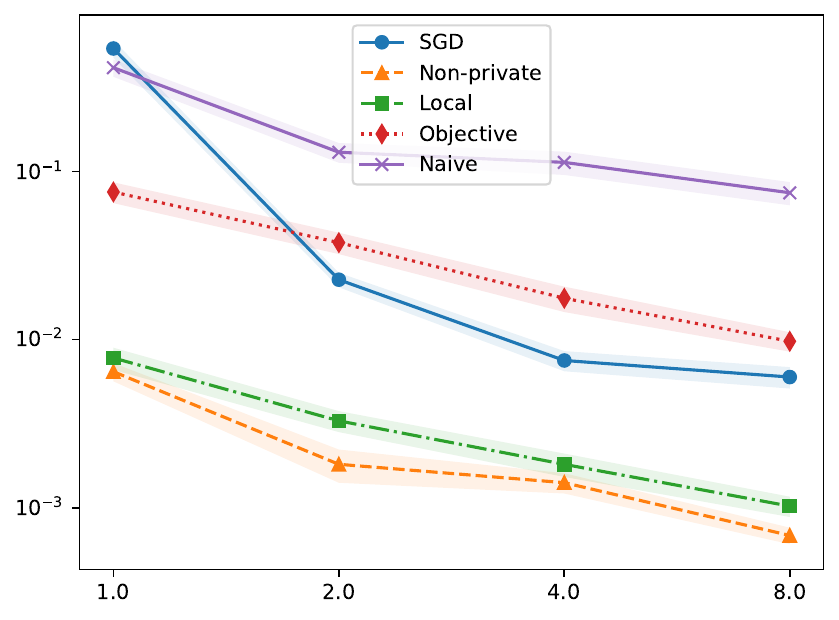}
      \put(50,0){$\diffp$}
    \end{overpic}
    \\
    (a) & (b)
  \end{tabular}
  \caption{ \label{fig:rob-reg-coord} Error $|\what{\theta}_j - \theta_j\opt|$ as
    a function of the privacy parameter $\diffp$ for simulated robust
    regression estimating a single (random) coordinate $j$.  (a) Small norm
    $\ltwo{\theta\opt} = 1$ (b) Larger norm $\ltwo{\theta\opt} =
    6$. Dimension $d = 10$ in both and sample size $n = 10^5$.}
\end{figure}

\begin{figure}
  \centering
  \begin{tabular}{cc}
    \begin{overpic}[width=.46\textwidth]{
        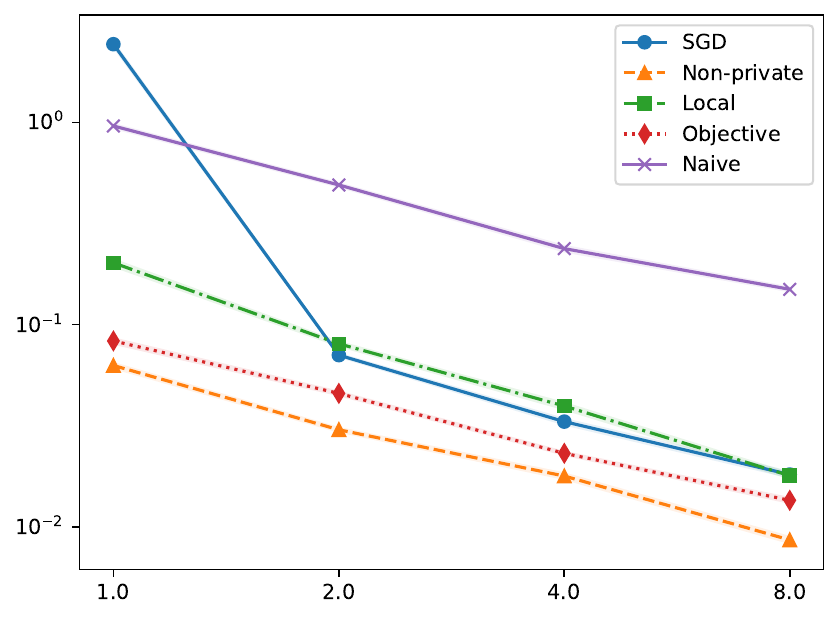}
      \put(50,0){$\diffp$}
    \end{overpic}
    &
    \begin{overpic}[width=.46\textwidth]{
        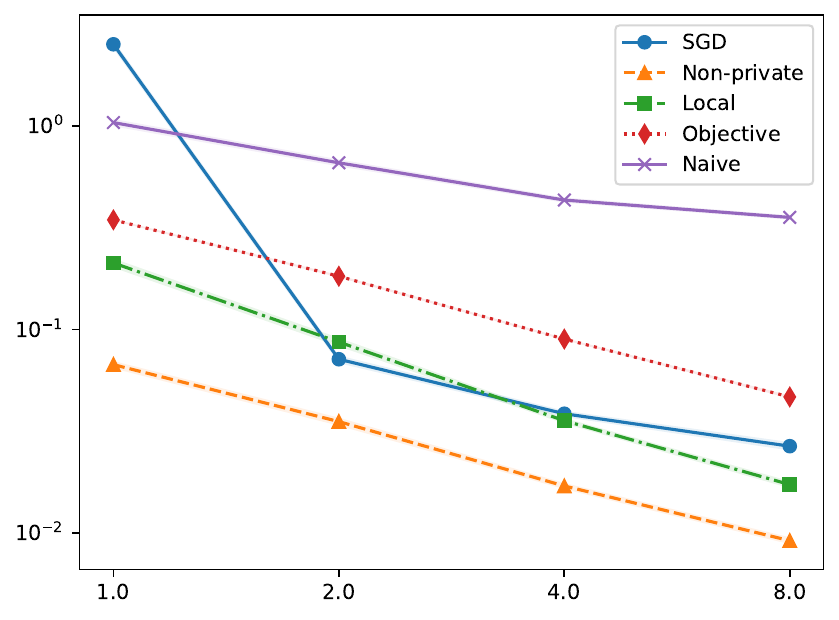}
      \put(50,0){$\diffp$}
    \end{overpic}
    \\
    (a) & (b)
  \end{tabular}
  \caption{ \label{fig:rob-reg-full} Error $\ltwos{\theta - \theta\opt}$ as a
    function of the privacy parameter $\diffp$ for simulated robust
    regression estimating entire parameter $\theta\opt \in \R^d$, where $d =
    10$, and sample size $n = 10^5$.  (a) Small norm $\ltwo{\theta\opt} = 1$
    (b) Larger norm $\ltwo{\theta\opt} = 6$.}
\end{figure}

\subsection{Logistic Regression (folktables)}

We also investigate the performance of the methods we develop on an
income prediction with data from the American Community Survey (ACS),
part of the US Census, as implemented in the FolkTables
datasets~\cite{DingHaMiSc21}.
We use state-level data drawn from the 2018 edition of the survey,
fitting logistic regression predictors of income,
where $Y = 1$ if the income of an individual is above \$40,000 and $Y = -1$
otherwise.
As features we take the following covariates:
an indicator of working age (between 18 and 60 years old);
hours-per week of work over the past year (normalized to the
interval $[-1, 1]$); schooling level from $-1$ (no grade school)
to $1$ (graduate degree); an indicator of whether an individual is white;
a 1-hot encoding of occupation mapped to 8 distinct
areas\footnote{These correspond to top-level occupation
codes from the 2018 census and are
``Management, Business, and Financial'',
``Computer, Engineering, and Science'',
``Education, Legal, Community Service, Arts, and Media'',
``Healthcare Practitioners and Technical'',
``Service'',
``Sales and Office'',
``Natural Resources, Construction, and Maintenance'',
and the union of
the categories
``Production, Transportation, and Material Moving'' and
``Military Specific'' occupations};
an indicator of marital status; an indicator of sex;
and a 1-hot encoding of whether an individual is
employed in a private corporation, government, self-employed,
or has unknown employment.
Including an intercept term and eliminating linear dependence in the
features, this yields $d = 17$-dimensional problem data, with
covariate vectors $x_i \in [-1, 1]^d$.

\begin{figure}[ht]
  \centering
  \begin{overpic}[width=.95\columnwidth]{
      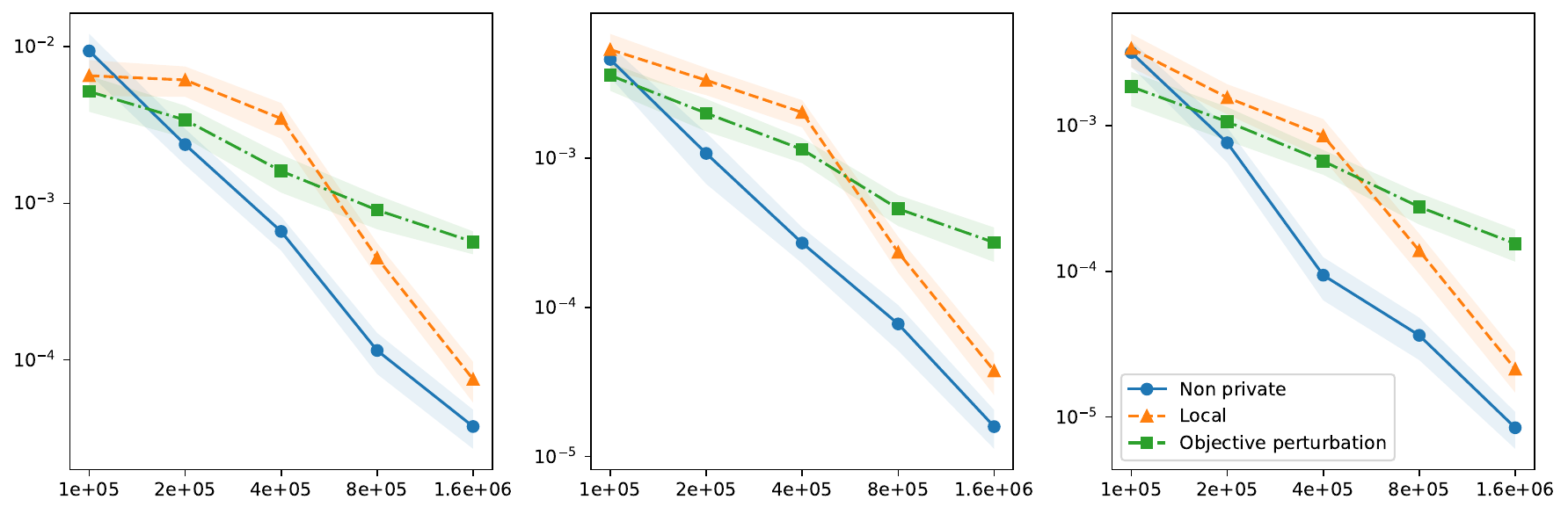}
    \put(-.2,1){\footnotesize $n = $}
    \put(-3,6){\rotatebox{90}{Error $|\what{\theta}_j - \theta_j(P_n)|$}}
  \end{overpic}
  \caption{\label{fig:ca-folktables} Estimation error versus sample size $n$
    for the parameter $\theta_j$ corresponding to the \texttt{SEX} indicator
    in predicting income levels using California survey data with
    approximate 95\% confidence intervals based on 25 experiments. Left: privacy
    level $\diffp = 2$. Middle: privacy level $\diffp = 4$.  Right: privacy
    level $\diffp = 8$. The true parameter $\theta_j(P_n) \approx -.27$.
    The ``Non-private'' estimator is the
    non-private idealized version (item~\ref{item:non-private-local}),
    objective perturbation corresponds to~\eqref{eqn:obj-pert},
    and Local is Algorithm~\ref{alg:release-u-t-theta}.
  }
\end{figure}

\begin{figure}[ht]
  \centering
  \begin{overpic}[width=.95\columnwidth]{
      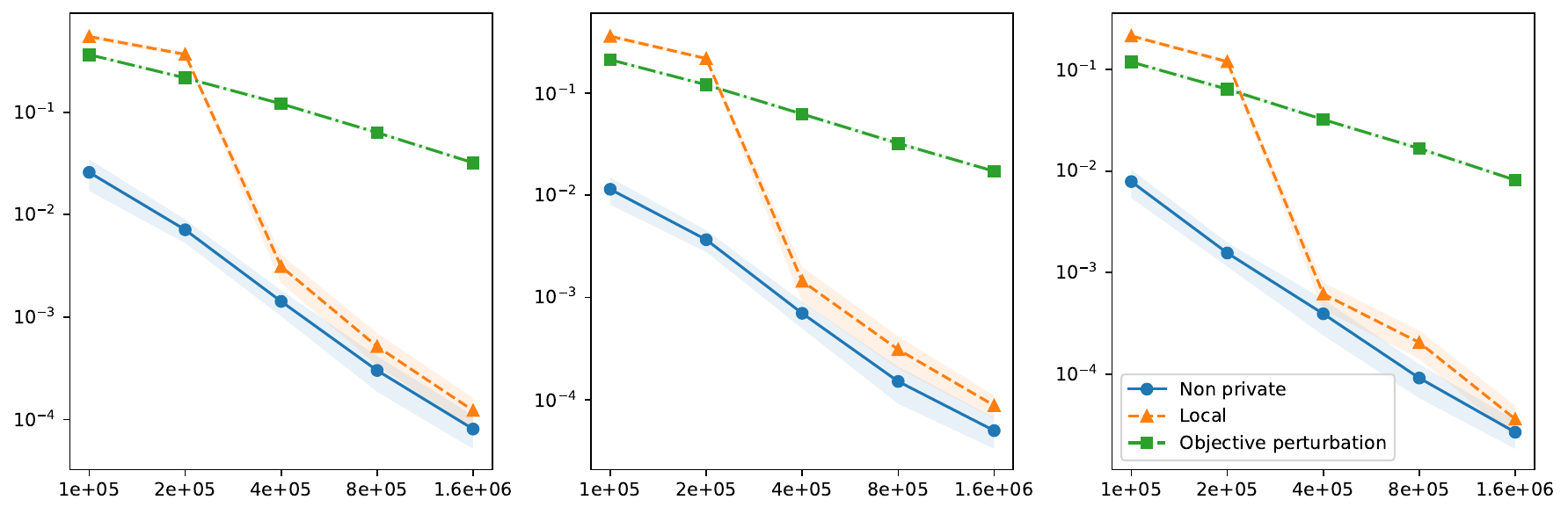}
    \put(-.2,1){\footnotesize $n = $}
    \put(-3,6){\rotatebox{90}{Error $|\what{\theta}_j - \theta_j(P_n)|$}}
  \end{overpic}
  \caption{\label{fig:mi-folktables} Estimation error versus sample size $n$
    for the parameter $\theta_j$ corresponding to amount of schooling
    in predicting income levels using Michigan survey data with
    approximate 95\% confidence intervals based on 25 experiments. Left: privacy
    level $\diffp = 2$. Middle: privacy level $\diffp = 4$.  Right: privacy
    level $\diffp = 8$. The true parameter $\theta_j(P_n) \approx 3.0$.
    The legend matches that in Fig.~\ref{fig:ca-folktables}.}
\end{figure}

We present results for experiments on data from California and Michigan;
results on other states are similar.
For each set of experiments, we treat the available data for the state as
the population, then draw a sample (with replacement) of size $n$ to give an
empirical distribution $P_n$, investigating estimators on the sample $P_n$.
We modify Algorithm~\ref{alg:release-u-t-theta} slightly,
so that if the estimated minimal eigenvalue $\what{\lambda} = 0$,
the method instead uses objective perturbation with privacy
parameter $\diffp/2$ (as we use half of the privacy budget
$\diffp$ to estimate $\what{\lambda}$).
On the data from California, the Hessian $\nabla^2 \poploss(\theta\opt)$ has
condition number $300 \pm 4$, while the Michigan data yields condition
number $450 \pm 6$, making the problems moderately poorly conditioned (these
condition numbers were large enough that even tuned DP-SGD had
error $\ge 10^{-1}$, so we do not include its results in the experiments).
Based on our results in simulation, we expect two main results in our
experiments: first, for small sample sizes, we expect objective perturbation
to outperform the private local modulus procedures
(Alg.~\ref{alg:release-u-t-theta}), but there ought to be
a transition once the sample size $n$ is large enough that
Alg.~\ref{alg:self-bounded-lambda-release} can effectively estimate
$\lambdamin(P_n)$.
Second, when the parameter $\theta_j(P_n)$ of interest is large,
we expect to see a larger gap in performance between
objective perturbation and the local noise addition procedures we develop
here.

In Figures~\ref{fig:ca-folktables} and \ref{fig:mi-folktables},
we see results roughly consistent with these expectations.
Both figures exhibit the transition somewhere in the neighborhood
of $n = 4 \cdot 10^5$ datapoints, where the error in using
Algorithm~\ref{alg:release-u-t-theta} drops substantially,
reflecting that it is possible to certify stability of the
local modulus of continuity $\sup_{x \in \mc{X}}
u^T \nabla^2 \poploss_n(\theta(P_n))^{-1} x$.
%
%
The plots also make clear that there remains a substantial gap
between methods that can explicitly
leverage the local modulus of continuity of the estimand of interest
and those that cannot.

\section{Parameter and eigenvalue stability guarantees}
\label{sec:parameter-stability}

The building blocks out of which all of our theoretical
results follow are stability
analyses that demonstrate that if the empirical minimizer of a (smooth
enough) loss $P_n \loss_\theta$ has Hessian with minimal eigenvalue $\lambda
> 0$, then (i) the empirical minimizers associated with neighboring samples
$P_n'$ are close, and (ii) the Hessians at these empirical minimizers have
minimal eigenvalues $\lambda' \ge \lambda - o(1)$, where the $o(1)$ term
depends in somewhat nontrivial ways on $P_n$ and $P_n'$.  Accordingly, in
this section, we present several results in this vein. The first set, in
Section~\ref{sec:parameter-stability}, gives quantitative bounds on the
stability of empirical minimizers for both general smooth losses
(Sec.~\ref{sec:generic-parameter-stability}) and quasi-self-concordant
generalized linear model losses (Sec.~\ref{sec:qsc-glm-stability}).
We then build off of these stability results to provide
Hessian and associated eigenvalue perturbation bounds in
Section~\ref{sec:eigenvalue-stability}.
Throughout this section we let
\begin{equation*}
  \theta(P_n) = \argmin_{\theta \in \Theta} P_n \loss_\theta
  + \frac{\lambdareg}{2} \ltwo{\theta - \theta_0}^2
  ~~~ \mbox{and} ~~~
  \lambdamin(P_n) = \lambdamin \left(P_n \ddot{\loss}_{\theta(P_n)}
  \right),
\end{equation*}
and $P_n'$ denotes the empirical distribution of a sample
satisfying $n \tvnorm{P_n - P_n'} \le 1$.

\subsection{Stability bounds for the full parameter}

We collect the main bounds on the deviation
$\ltwo{\theta(P_n) - \theta(P_n')}$, making
the heuristic development in Sec.~\ref{sec:heuristic-motivation}
rigorous and giving the appropriate numerical constants necessary to
implement our associated private algorithms.
We defer proofs of the results to Section~\ref{sec:proofs-parameter-stability}.

\subsubsection{Generic smooth losses}
\label{sec:generic-parameter-stability}

When the losses have $\lipletter_i$-Lipschitz continuous $i$th derivative
for $i = 0, 1, 2$, we have the following two propositions;
the first gives a ``basic'' guarantee, while the second sharpens it
by a particular recursive bound.
In each, we require that $\lambdamin(P_n)$ is large enough (highlighting
the importance of privately certifying lower bounds on $\lambdamin(P_n)$);
we thus require the condition
\begin{equation}
  \tag{C2}
  \lambdamin(P_n)
  + \lambdareg \ge \max\left\{\frac{3 \lipgrad}{n},
  \sqrt{\frac{12 \lipobj \liphess}{n}}\right\}.
  \label{eqn:lambda-is-large-enough-to-start}
\end{equation}
As a brief remark, we can compare
Condition~\eqref{eqn:lambda-is-large-enough-to-start} to
Condition~\eqref{eqn:lambda-self-bounded-big-enough}.  Let us assume a
``typical'' scenario, where the Lipschitz constants exhibit the scalings
$\lipletter_i \propto d^{(i + 1)/2}$ (recall
Examples~\ref{example:robust-regression}
and~\ref{example:logistic-regression}), and we expect that $\lambdamin(P_n)$
should be roughly of constant order.  (Think of classical linear regression,
where $P_n \ddot{\loss}_\theta = \frac{1}{n} \sum_{i = 1}^n x_i x_i^T$, so
that if $x_i \simiid \uniform(\{\pm 1\}^d)$ then $P_n \ddot{\loss}_\theta
\approx I_d$.)  Then Condition~\eqref{eqn:lambda-is-large-enough-to-start}
requires a sample size scaling at least as $n \gtrsim d^2$ or that
$\lambdareg \gtrsim d /\sqrt{n}$, while
Condition~\eqref{eqn:lambda-self-bounded-big-enough} requires only that $n
\gtrsim d$ or $\lambdareg \gtrsim \frac{d}{n}$, a quadratic difference in
required sample size.

\begin{proposition}
  \label{proposition:basic-parameter-stability}
  Let Condition~\eqref{eqn:lambda-is-large-enough-to-start} hold.
  Then
  for any empirical distribution
  $P_n'$ with $\tvnorm{P_n - P_n'} \le 1/n$,
  the minimizer $\theta(P_n')$ exists and satisfies
  \begin{equation*}
    \ltwo{\theta(P_n) - \theta(P_n')}
    \le \frac{12 \lipobj}{n (\lambdamin(P_n) + \lambdareg)}.
  \end{equation*}
\end{proposition}
\noindent
See Section~\ref{sec:proof-basic-parameter-stability} for a proof.

When the problem is unconstrained (so that $\Theta = \R^d$), we can provide
a sharper recursive
bound. Proposition~\ref{proposition:basic-parameter-stability} guarantees
the existence of a solution $\theta(P_n')$ for all $P_n'$ neighboring $P_n$
whenever Condition~\eqref{eqn:lambda-is-large-enough-to-start} holds, and so
we can perform a Taylor expansion to yield sharper guarantees.  (We defer
the proof to Section~\ref{sec:proof-double-generic-parameter-recursion}).
\begin{proposition}
  \label{proposition:double-generic-parameter-recursion}
  Let the conditions of Proposition~\ref{proposition:basic-parameter-stability}
  hold, but assume $\Theta = \R^d$.
  Then for all $P_n'$ neighboring $P_n$, we have
  \begin{align*}
    \ltwo{\theta(P_n) - \theta(P_n')}
    \le \frac{1}{2 \liphess}
    \left[\lambdamin(P_n) + \lambdareg - \sqrt{(\lambdamin(P_n) + \lambdareg)^2
        - \frac{8 \lipobj \liphess}{n}}\right].
  \end{align*}
\end{proposition}

Proposition~\ref{proposition:double-generic-parameter-recursion}
is always sharper than
Proposition~\ref{proposition:basic-parameter-stability} and
is (for $n$ large) asymptotically tight. Indeed, letting $\lambdareg = 0$
for simplicity and assuming $n$ is
large, a Taylor expansion of $\sqrt{a^2 + \delta} = a + \delta / 2a +
O(\delta^2)$ gives
\begin{equation*}
  \frac{1}{2 \liphess}
  \left[\lambdamin(P_n) - \sqrt{\lambdamin^2(P_n)
      - \frac{8 \lipobj \liphess}{n}}\right]
  = \frac{2 \lipobj}{\lambdamin(P_n) n} + O(n^{-2}),
\end{equation*}
which is essentially as sharp as we could expect in these generic settings;
recall the definition~\eqref{eqn:directional-difference} of the (asymptotic)
local sensitivity.


\subsubsection{Quasi-self-concordant GLMs}
\label{sec:qsc-glm-stability}

Given more conditions on the losses at play, we
can obtain sharper stability guarantees for $\theta(P_n)$; we provide a
few of these here.  Recall the definitions~\eqref{eqn:qsc} of
quasi-self-concordance (q.s.c.), so that we consider generalized linear
model (GLM) losses of the form
\begin{equation*}
  \loss_\theta(x, y) = h(\<\theta, x\>, y).
\end{equation*}
Notably, for any such GLM-type loss, we have
\begin{equation*}
  \dot{\loss}_\theta(x, y)
  = h'(\<\theta, x\>, y) x
  ~~ \mbox{and} ~~
  \ddot{\loss}_\theta(x, y)
  = h''(\<\theta, x\>, y) xx^T.
\end{equation*}
Thus, if the radius $\radius(\mc{X}) \defeq \sup_{x \in \mc{X}} \ltwo{x}$
is finite, for any unit vector $v$ and $t \ge 0$ we have the key
semidefinite lower bound
\begin{equation}
  \label{eqn:key-qsc-hessian}
  \ddot{\loss}_{\theta + t v} \succeq
  \hinge{1 - \concordantfunc(t \cdot \radius(\mc{X}))} \ddot{\loss}_{\theta}
\end{equation}
for any parameter $\theta$. The self-bounding
inequality~\eqref{eqn:key-qsc-hessian} is the key that allows us more
precise control on the error matrices in the heuristic derivation
of stability in Section~\ref{sec:heuristic-motivation}.

We begin with a proposition that applies to any lower q.s.c.\ loss with
$\concordantfunc(t) = \selftwoconst t$; as
Examples~\ref{example:robust-regression-sc}
and~\ref{example:binary-logistic-sc} show, this applies with $\selftwoconst
= 1$ for robust regression with the log loss and binary logistic regression.
(See Section~\ref{sec:proof-basic-parameter-self-bounding}
for a proof.)
\begin{proposition}
  \label{proposition:basic-parameter-self-bounding}
  Define $\radius(\mc{X}) = \sup_{x \in \mc{X}} \ltwo{x}$ and
  the loss $\loss$ be lower q.s.c.\ with 
  $\concordantfunc(t) \le \selftwoconst t$.
  Let $\rho \in (0, 1)$, and define
  \begin{equation*}
    \paramchange(P_n) \defeq \frac{4 \lipobj}{ n} \frac{1}{
      \rho \lambdamin(P_n)
      + \lambdareg - \lipgrad / n}.
  \end{equation*}
  Then
  $\ltwo{\theta(P_n) - \theta(P_n')}
  \le \paramchange(P_n)$ so long as $\paramchange(P_n) \le
  \frac{1 - \rho}{ \selftwoconst \radius(\mc{X})}$.
\end{proposition}

As in Section~\ref{sec:generic-parameter-stability}, we can leverage
Proposition~\ref{proposition:basic-parameter-self-bounding} to obtain
sharper guarantees by iterating its implied bound on $\ltwo{\theta(P_n) -
  \theta(P_n')}$ in a more careful Taylor expansion of $P_n'
\dot{\loss}_{\theta(P_n')}$ around $P_n \dot{\loss}_{\theta(P_n)}$. We
assume Condition~\eqref{eqn:lambda-self-bounded-big-enough} that
$\lambdamin(P_n) + \frac{1}{\rho} \lambdareg \ge \frac{4 \lipobj
  \selftwoconst \radius(\mc{X})}{ \rho(1 - \rho) n} + \frac{\lipgrad}{\rho
  n}$.  Fixing $\theta = \theta(P_n)$, we define the shorthands
\begin{align*}
  \tau_1 & \defeq \sup_{P_n', \theta'}
  \ltwo{(P_n \ddot{\loss}_\theta
    + \lambdareg I)^{-1}
    (P_n - P_n') \dot{\loss}_{\theta'}}
  \le \frac{2 \lipobj}{n (\lambdamin(P_n) + \lambdareg)}
  ~~ \mbox{and} \\
  \tau_2 & \defeq \sup_{P_n', \theta'}
  \ltwo{(P_n - P_n') \dot{\loss}_{\theta'}}
  \le \frac{2 \lipobj}{n}.
\end{align*}
We prove the following guarantee in
Section~\ref{sec:proof-advanced-parameter-self-bounding}.

\begin{proposition}
  \label{proposition:advanced-parameter-self-bounding}
  Let Condition~\eqref{eqn:lambda-self-bounded-big-enough}
  hold for a given $\rho \in (0, 1)$ and
  let $\loss$ be
  be $\concordantfunc$-q.s.c.~\eqref{eqn:functional-self-bounding}, where
  $\concordantfunc$ satisfies inequality~\eqref{eqn:bound-sc-by-linear}
  that $\concordantfunc(t) \le \selftwoconst t$ for
  $0 \le t \le \frac{1 - \rho}{\selftwoconst}$. Then
  \begin{equation*}
    \ltwo{\theta(P_n) - \theta(P_n')}
    \le \frac{1 - \sqrt{1 - \frac{4 \selftwoconst \radius(\mc{X}) \tau_2}{
          \lambdamin(P_n) + \lambdareg}}}{2 \selftwoconst \radius(\mc{X})}.
  \end{equation*}
\end{proposition}
\noindent
Recalling the parameter $\tparamchange(\lambda)$ from its
definition~\eqref{eqn:t-param-change},
Proposition~\ref{proposition:advanced-parameter-self-bounding} shows
that
\begin{equation*}
  \ltwo{\theta(P_n) - \theta(P_n')}
  \le \tparamchange(\lambdamin(P_n) + \lambdareg)
  = \frac{1 - \sqrt{1 - \frac{8 \selftwoconst \radius(\mc{X})
        \lipobj}{(\lambdamin(P_n) + \lambdareg) n}}}{2 \selftwoconst \radius(\mc{X})}.
\end{equation*}

A few more careful Taylor approximations show that
Proposition~\ref{proposition:advanced-parameter-self-bounding} sharpens
Proposition~\ref{proposition:basic-parameter-self-bounding}.
Recognize that $\sqrt{1 - \epsilon} \ge 1 - \frac{\epsilon}{2} -
\frac{\epsilon^2}{2}$ for $\epsilon \in [0, 1]$, so that for $\lambda =
\lambdamin(P_n) + \lambdareg$, the right hand side of the bound in the
proposition satisfies
\begin{equation*}
  \tparamchange(\lambda)
  \le \frac{2 \lipobj}{n \lambda}
  + \frac{16 \selftwoconst \radius(\mc{X}) \lipobj^2}{n^2 \lambda^2}
  \stackrel{\eqref{eqn:lambda-self-bounded-big-enough}}{\le}
  \frac{2 \lipobj}{n (\lambdamin(P_n) + \lambdareg)}
  + \frac{4 \rho(1 - \rho)\lipobj}{n (\lambdamin(P_n) + \lambdareg)},
\end{equation*}
where the second inequality holds under
Condition~\eqref{eqn:lambda-self-bounded-big-enough},

\subsection{Stability bounds for minimal eigenvalues}
\label{sec:eigenvalue-stability}

With the parameter stability bounds in the preceding section,
we can obtain corollaries about the perturbation
stability of minimal (and maximal) eigenvalues of
the empirical Hessian matrix $P_n \ddot{\loss}_\theta$.
As these are all relatively quick, we give
them sequentially and include proofs.

\begin{corollary}
  \label{corollary:generic-eigenvalue-change}
  Let the conditions of
  Proposition~\ref{proposition:double-generic-parameter-recursion} hold.
  Then for all $P_n'$ neighboring $P_n$,
  \begin{align*}
    \lefteqn{\frac{\lambdamin(P_n) + \lambdareg}{2}
      \left[1 + \sqrt{1 - \frac{8 \lipobj \liphess}{n (\lambdamin(P_n) + \lambdareg)^2}}
        \right] - \frac{\lipgrad}{n}} \\
    & \quad \le \lambdamin(P_n') + \lambdareg
    \le \frac{\lambdamin(P_n) + \lambdareg}{2}
    \left[3 - \sqrt{1 - \frac{8 \lipobj \liphess}{n (\lambdamin(P_n)
          + \lambdareg)^2}}
      \right] + \frac{\lipgrad}{n}.
  \end{align*}
\end{corollary}
\begin{proof}
  Let $\theta = \theta(P_n)$ and $\theta' = \theta(P_n')$ for shorthand.
  We prove the lower bound; the upper bound is completely similar.
  Then we have the semidefinite ordering inequalities
  \begin{equation*}
    P_n' \ddot{\loss}_{\theta'}
    = P_n \ddot{\loss}_{\theta'}
    + (P_n' - P_n) \ddot{\loss}_{\theta'}
    \succeq P_n \ddot{\loss}_\theta
    - \liphess \ltwo{\theta - \theta'}
    - \frac{\lipgrad}{n} I,
  \end{equation*}
  because of the assumptions that $\ddot{\loss}_\theta$ is
  $\liphess$-Lipschitz and that
  $\theta \mapsto \dot{\loss}_\theta$ is
  $\lipgrad$-Lipschitz, so that
  $\opnorms{\ddot{\loss}_\theta} \le \lipgrad$ for any parameter
  $\theta$. Substituting the bound of
  Proposition~\ref{proposition:double-generic-parameter-recursion}
  for $\ltwos{\theta(P_n) - \theta(P_n')}$ then gives the
  corollary.
\end{proof}

The cleaner behavior of Hessians for self-concordant GLMs allows sharper
recursive guarantees. Recall the parameter change
quantity~\eqref{eqn:t-param-change} of Section~\ref{sec:main-release}, that
is,
\begin{equation*}
  \tparamchange(\lambda)
  = \frac{1 - \sqrt{1 - \frac{8 \selftwoconst \radius(\mc{X})
        \lipobj}{\lambda n}}}{2 \selftwoconst \radius(\mc{X})}.
\end{equation*}

\begin{corollary}
  \label{corollary:eigenvalue-qsc-change}
  Let the conditions of
  Proposition~\ref{proposition:advanced-parameter-self-bounding}
  hold, so that $\loss$ is $\concordantfunc$-q.s.c., and let
  $\lambda_j(P_n) = \lambda_j(P_n \ddot{\loss}_{\theta(P_n)})$ denote
  the $j$th eigenvalue of $P_n \ddot{\loss}_{\theta(P_n)}$. Then
  for all $P_n'$ neighboring $P_n$,
  \begin{equation*}
    \left|\lambda_j(P_n') - \lambda_j(P_n)\right|
    \le \lambda_j(P_n) \concordantfunc\big(\tparamchange(
    \lambdamin(P_n) + \lambdareg)
    \cdot \radius(\mc{X})\big)
    + \frac{\lipgrad}{n}.
  \end{equation*}
\end{corollary}
\begin{proof}
  Let $\theta' = \theta(P_n')$ and $\theta = \theta(P_n)$ for shorthand
  as usual. Then we have
  \begin{equation*}
    P_n' \ddot{\loss}_{\theta'}
    = P_n \ddot{\loss}_{\theta'} + (P_n' - P_n) \ddot{\loss}_{\theta'}.
  \end{equation*}
  As $\ddot{\loss}_{v}(z) \succeq 0$ for all $z, v$ and
  $\opnorms{\ddot{\loss}_v} \le \lipgrad$, we thus have
  \begin{equation*}
    -\frac{\lipgrad}{n} I + P_n \ddot{\loss}_{\theta'}
    \preceq P_n' \ddot{\loss}_{\theta'}
    \preceq P_n \ddot{\loss}_{\theta'} + \frac{\lipgrad}{n} I.
  \end{equation*}
  Let $t = \ltwo{\theta - \theta'}$ and $r = \radius(\mc{X})$ for
  shorthand. Then
  for the upper bound, note that
  $\ddot{\loss}_{\theta'} \preceq \ddot{\loss}_\theta (1 + \concordantfunc(t r))$,
  so that
  $P_n' \ddot{\loss}_{\theta'}
  \preceq P_n \ddot{\loss}_{\theta} (1 + \concordantfunc(t r))
  + \frac{\lipgrad}{n} I$.
  A similar derivation gives
  $P_n' \ddot{\loss}_{\theta'}
  \succeq P_n \ddot{\loss}_{\theta} \hinge{1 - \concordantfunc(t r)}
  - \frac{\lipgrad}{n} I$.
  Applying Weyl's inequalitites then gives
  \begin{equation*}
    \hinge{1 - \concordantfunc(t r)}
    \lambda_j(P_n)
    - \frac{\lipgrad}{n}
    \le \lambda_j(P_n')
    \le (1 + \concordantfunc(t r)) \lambda_j(P_n)
    + \frac{\lipgrad}{n},
  \end{equation*}
  and rearranging this yields the corollary.
\end{proof}

Unpacking Corollary~\ref{corollary:eigenvalue-qsc-change}, consider
the ``standard'' scalings in which
$\radius(\mc{X}) \lesssim \sqrt{d}$ and $\lipobj \lesssim \sqrt{d}$,
for example, in the context of Examples~\ref{example:robust-regression-sc}
and~\ref{example:binary-logistic-sc}. Then taking $\concordantfunc(t)
= e^t - 1 \approx t$ and assuming $n \gg d$, we have
$\tparamchange(\lambda)
= \frac{2 \lipobj}{n \lambda} + O(n^{-2})$ and
Corollary~\ref{corollary:eigenvalue-qsc-change} unpacks to the bound
that (roughly)
\begin{align*}
  |\lambda_j(P_n') - \lambda_j(P_n)|
  & \le \frac{\lambda_j(P_n)}{\lambdamin(P_n) + \lambdareg}
  \frac{2 \lipobj \radius(\mc{X})}{n}
  + \frac{\lipgrad}{n} + O(n^{-2})
  \lesssim \frac{\lambda_j(P_n)}{\lambdamin(P_n) + \lambdareg} \frac{d}{n}.
\end{align*}
So we have the guarantees that
\begin{equation*}
  |\lambdamin(P_n) - \lambdamin(P_n')| \lesssim \frac{d}{n}
  ~~ \mbox{and} ~~
  |\lambdamax(P_n) - \lambdamax(P_n')|
  \lesssim \frac{\lambdamax(P_n) + \lambdareg}{\lambdamin(P_n) + \lambdareg}
  \frac{d}{n},
\end{equation*}
and the eigenvalues are (quite) stable to perturbations. In passing,
we note that a first-order expansion of $\lambdamax(P_n')$ suggests
these bounds are likely hard to improve.

\subsection{Proofs of parameter stability}
\label{sec:proofs-parameter-stability}

We collect our omitted proofs from this section.

\subsubsection{Proof of Proposition~\ref{proposition:basic-parameter-stability}}
\label{sec:proof-basic-parameter-stability}

Let $\theta = \theta(P_n)$ for shorthand and $\reg(\theta) =
\frac{\lambdareg}{2} \ltwo{\theta - \theta_0}^2$ be the regularization.
Fixing an arbitrary unit vector $v$, let $\theta' = \theta + tv \in \Theta$
and $t \ge 0$ be any value.  By the first-order conditions for optimality of
convex optimization we have $(P_n \dot{\loss}_\theta + \nabla
\reg(\theta))^T v \ge 0$ for any such setting.  Then our various Lipschitz
continuity assumptions give
\begin{align*}
  \lefteqn{P_n' \loss_{\theta'}
    + \reg(\theta')
    \ge P_n' \loss_\theta + R(\theta) + t (P_n' \dot{\loss}_\theta + \nabla
    \reg(\theta))^T v
    + \frac{t^2}{2} (v^T P_n'\ddot{\loss}_\theta v
    + \lambdareg)
    - \frac{\liphess}{6} t^3} \\
  & = P_n' \loss_\theta + \reg(\theta)
  + t \underbrace{(P_n \dot{\loss}_\theta + \nabla \reg(\theta))^T v}_{\ge 0}
  + \frac{t^2}{2} (v^T P_n \ddot{\loss}_\theta v + \lambdareg)
  - \frac{\liphess}{6} t^3
  + (P_n' - P_n)
  \left[t \dot{\loss}_\theta^T v
    + \frac{t^2}{2} v^T \ddot{\loss}_\theta v \right] \\
  & \ge P_n' \loss_\theta + \reg(\theta)
  + \frac{\lambdamin(P_n) + \lambdareg}{2} t^2 - \frac{\liphess}{6} t^3
  - \frac{2\lipobj}{n} t
  - \frac{\lipgrad}{2n} t^2.
\end{align*}
That is, if $\theta' = \theta + tv$ we have for
$\lambda = \lambdamin(P_n) + \lambdareg$ that
\begin{equation}
  \label{eqn:basic-loss-inequality}
  P_n' \loss_{\theta'} + \reg(\theta')
  \ge P_n' \loss_\theta + \reg(\theta) + \frac{\lambda}{2} t^2
  - t \left[
    \frac{2 \lipobj}{n} + \frac{\lipgrad}{2n} t
    + \frac{\liphess}{6} t^2 \right].
\end{equation}
Because $t \mapsto P_n' \loss_{\theta + tv} + \reg(\theta + tv)$ is convex,
if for some $t_0 > 0$ the sum of the final two terms
on the right hand side
of inequality~\eqref{eqn:basic-loss-inequality} is positive,
then we evidently have $P_n' \loss_{\theta + tv}
+ \reg(\theta + tv) >
P_n' \loss_\theta + \reg(\theta)$ for all $t \ge t_0$.
We now proceed to find a fairly gross upper bound on this critical
radius $t_0$. By condition~\eqref{eqn:lambda-is-large-enough-to-start},
$n$ is large enough
that $\frac{\lipgrad}{2n} \le \frac{\lambda}{6}$. Then
for any $t \le \frac{\lambda}{\liphess}$, we have
$\frac{\liphess}{6} t^3
\le \frac{\lambda}{6} t^2$, and under these conditions
inequality~\eqref{eqn:basic-loss-inequality} implies
\begin{align*}
  P_n' \loss_{\theta + tv} + \reg(\theta + tv)
  & \ge P_n' \loss_\theta + \reg(\theta)
  + \frac{\lambda}{3} t^2
  - \frac{\lipgrad}{2n} t^2
  - \frac{2 \lipobj}{n} t
  \ge P_n' \loss_\theta + \reg(\theta)
  + \frac{\lambda}{6} t^2 - \frac{2 \lipobj}{n} t.
\end{align*}
Then if $t > t_0 = \frac{12 \lipobj}{n \lambda}$, we have $P_n'
\loss_{\theta + tv} + \reg(\theta + tv) > P_n \loss_\theta
+ \reg(\theta)$, and it is possible to find such a
$t$ so long as $\frac{\lambda}{\liphess} > \frac{12 \lipobj}{n \lambda}$,
that is, $\lambda > \sqrt{12 \lipobj \liphess / n}$, which holds
per~\eqref{eqn:lambda-is-large-enough-to-start}.  Restating this, whenever
Condition~\eqref{eqn:lambda-is-large-enough-to-start} holds, we necessarily
have $\ltwo{\theta(P_n) - \theta(P_n')} \le \frac{12 \lipobj}{n \lambda}$ as
Proposition~\ref{proposition:basic-parameter-stability} requires.

\subsubsection{Proof of
  Proposition~\ref{proposition:double-generic-parameter-recursion}}
\label{sec:proof-double-generic-parameter-recursion}

Proposition~\ref{proposition:basic-parameter-stability}
guarantees the existence of a solution $\theta(P_n')$
minimizing $P_n' \loss_v$ in $v$; letting
$\theta' = \theta(P_n')$ and
$\theta = \theta(P_n)$ for shorthand, and $\reg(\theta)
= \frac{\lambdareg}{2} \ltwo{\theta - \theta_0}^2$ be the
$\ell_2$-regularization, we therefore
can perform a series of Taylor approximations
to obtain
\begin{align*}
  0 = P_n' \dot{\loss}_{\theta'}
  + \nabla \reg(\theta')
  & = P_n \dot{\loss}_{\theta'}
  + \nabla \reg(\theta')
  + (P_n' - P_n) \dot{\loss}_{\theta'} \\
  & = P_n \dot{\loss}_\theta
  + \nabla \reg(\theta)
  + (P_n \ddot{\loss}_\theta + \lambdareg I + E)
  (\theta' - \theta)
  + (P_n' - P_n) \dot{\loss}_{\theta'},
\end{align*}
where the error matrix $E$ satisfies
$\opnorm{E} \le \liphess \ltwo{\theta - \theta'}$
and we have used $\nabla \reg(\theta) - \nabla \reg(\theta')
= \lambdareg(\theta - \theta')$.
Under Condition~\eqref{eqn:lambda-is-large-enough-to-start},
we know that $P_n\ddot{\loss}_\theta + \lambdareg I + E$ is invertible
(even more, $\lambdamin(P_n) = \lambdamin(P_n\ddot{\loss}_\theta)
+ \lambdareg I
> \opnorm{E}$).
Thus we obtain
\begin{align*}
  \theta - \theta' = (P_n \ddot{\loss}_\theta + \lambdareg I + E)^{-1}
  (P_n - P_n')\dot{\loss}_{\theta'}.
\end{align*}
Defining
the shorthand $H = P_n \ddot{\loss}_\theta + \lambdareg I$ for the Hessian,
because $\opnorm{E} < \lambdamin(P_n)$,
we have $(H + E)^{-1} = H^{-1} + \sum_{i = 1}^\infty (-1)^i
(H^{-1} E)^i H^{-1}$, which in turn satisfies
\begin{align*}
  \opnormbigg{\sum_{i = 1}^\infty (-1)^i
    (H^{-1} E)^i H^{-1}}
  & \le \frac{\opnorm{E}}{(\lambdamin(P_n) + \lambdareg)^2}
  \frac{1}{1 - \opnorm{E} / (\lambdamin(P_n) + \lambdareg)} \\
  & = \frac{\opnorm{E}}{\lambdamin(P_n) + \lambdareg}
  \frac{1}{\lambdamin(P_n) + \lambdareg - \opnorm{E}}.
\end{align*}
Substituting above and using the shorthand
$\lambda = \lambdamin(P_n) + \lambdareg$, we obtain
\begin{align*}
  \ltwo{\theta - \theta'}
  \le \ltwo{H^{-1} (P_n - P_n')\dot{\loss}_{\theta'}}
  + \frac{\opnorm{E}}{\lambda}
  \frac{1}{\lambda - \opnorm{E}}
  \ltwo{(P_n - P_n') \dot{\loss}_{\theta'}}
\end{align*}
As $\opnorm{E} \le \liphess \ltwo{\theta - \theta'}$
and $\ltwos{(P_n - P_n') \dot{\loss}_{\theta'}}
\le \frac{2 \lipobj}{n}$,
if we let $t = \ltwo{\theta - \theta'}$,
then $t$ necessarily satisfies
\begin{equation*}
  t \le
  \frac{2 \lipobj}{n \lambda}
  + \frac{2 \lipobj \liphess t}{n \lambda(\lambda - \liphess t)}
  ~~ \mbox{or} ~~
  \lambda t - \liphess t^2 \le
  \frac{2 \lipobj}{n \lambda}
  + \frac{2 \lipobj}{n \lambda} \liphess t.
\end{equation*}
Solving the implied quadratic yields
\begin{equation*}
  \ltwo{\theta(P_n) - \theta(P_n')}
  = t \le \frac{\lambda - \sqrt{\lambda^2 - \frac{8 \lipobj \liphess}{n}}}{
    2 \liphess}.
\end{equation*}

\subsubsection{Proof of
  Proposition~\ref{proposition:basic-parameter-self-bounding}}
\label{sec:proof-basic-parameter-self-bounding}

As usual, we let $P_n$ and $P_n'$ be neighboring datasets, $\reg(\theta) =
\frac{\lambdareg}{2} \ltwo{\theta - \theta_0}^2$, and let $\theta =
\theta(P_n) = \argmin_\theta P_n \loss_\theta + \reg(\theta)$.  Then for any
vector $v$, we have
\begin{equation*}
  \ddot{\loss}_{\theta + v}(x, y)
  = h''(\<\theta + v, x\>, y) xx^T
  \succeq h''(\<\theta, x\>, y) xx^T
  \hinge{1 - \selftwoconst[h] |\<v, x\>|}.
\end{equation*}
Rewriting this in with the typical shorthand of suppressing
$x$ and $y$ and the dependence of $\selftwoconst[h]$ on $h$,
we have
$\ddot{\loss}_{\theta + v} \succeq
\ddot{\loss}_\theta \hinge{1 - \selftwoconst |\<v, x\>|}$.
Thus for any $v$, there is some $s \in [0, 1]$ for which we have
\begin{align*}
  P_n'\loss_{\theta + v}
  + \reg(\theta + v)
  & = P_n' \loss_{\theta}
  + \reg(\theta)
  + (P_n' \dot{\loss}_\theta + \nabla \reg(\theta))^T v
  + \half v^T (P_n' \ddot{\loss}_{\theta + s v} + \lambdareg I) v \\
  & \ge P_n'\loss_{\theta} + \reg(\theta)
  + (P_n' - P_n) \dot{\loss}_\theta^T v
  + \half v^T (P_n \ddot{\loss}_{\theta + s v} + \lambdareg I) v
  - \frac{\lipgrad}{2 n} \ltwo{v}^2
\end{align*}
for some $s \in [0, 1]$, where the inequality follows because
$\opnorms{\ddot{\loss}} \le \lipgrad$ and $P_n \dot{\loss}_\theta + \nabla
\reg(\theta) = 0$. Using the assumption~\eqref{eqn:bound-sc-by-linear} on
the losses and its consequence~\eqref{eqn:key-qsc-hessian}, we then have
\begin{align}
  \lefteqn{P_n' \loss_{\theta + v}
    + \reg(\theta + v)} \nonumber \\
  & \ge P_n'\loss_\theta
  + \reg(\theta)
  + (P_n' - P_n) \dot{\loss}_\theta^T v
  + \half v^T (P_n \ddot{\loss}_\theta \hinge{1 - \selftwoconst |\<v, X\>|}
  + \lambdareg I)
  v - \frac{\lipgrad}{2 n} \ltwo{v}^2
  \label{eqn:self-lower-bound} \\
  & \ge P_n' \loss_\theta + \reg(\theta)
  - \frac{2 \lipobj}{n} \ltwo{v}
  + \left(\lambdamin(P_n)(1 - \selftwoconst
  \ltwo{v} \radius(\mc{X})) + \lambdareg I
  - \frac{\lipgrad}{n} \right) \frac{\ltwo{v}^2}{2}.
  \nonumber
\end{align}
Let $t = \ltwo{v}$ for shorthand. Then by convexity, if
\begin{equation*}
  -\frac{2 \lipobj}{n} t
  + \left(\lambdamin(P_n) (1 - \selftwoconst t \cdot
  \radius(\mc{X}))
  + \lambdareg
  - \frac{\lipgrad}{n} \right) \frac{t^2}{2} > 0,
\end{equation*}
then we necessarily have
$P_n' \loss_{\theta + v}
+ \reg(\theta + v) > P_n' \loss_\theta + \reg(\theta)$, and
moreover, $P_n' \loss_{\theta + u}
+ \reg(\theta + u) > P_n' \loss_\theta + \reg(\theta)$
whenever $\ltwo{u} > t$, so that
$\ltwo{\theta(P_n) - \theta(P_n')} \le t$.

Notably, whenever $t \selftwoconst \radius(\mc{X}) \le 1 - \rho$,
it suffices to find a $t$ satisfying
\begin{equation*}
  -\frac{2 \lipobj}{n} t + \left(\lambdamin(P_n) \rho
  + \lambdareg - \frac{\lipgrad}{n}\right) \frac{t^2}{2} = 0
  ~~~ \mbox{and} ~~~
  t \le \frac{1 - \rho}{\selftwoconst \radius(\mc{X})}.
\end{equation*}
This occurs whenever
$t = \frac{4 \lipobj}{n} \frac{1}{\rho \lambdamin(P_n)
  + \lambdareg
  - \lipgrad / n} < \frac{1 - \rho}{\selftwoconst \radius(\mc{X})}$,
which is the claim of the proposition.


\subsubsection{Proof of
  Proposition~\ref{proposition:advanced-parameter-self-bounding}}
\label{sec:proof-advanced-parameter-self-bounding}

The proof of the proposition requires some manipulations of
Hessian error terms, so we provide a matrix inequality to
address them.
\begin{lemma}
  \label{lemma:advanced-matrix-inverse-perturbation}
  Let $A \succ 0$ satisfy
  $-\delta A \preceq E \preceq \delta A$. Then
  for each $k \in \N$,
  there exists a symmetric $D$
  satisfying $-\frac{\delta^{k+1}}{1 - \delta} \preceq D \preceq
  \frac{\delta^{k + 1}}{1 - \delta}$ and for which
  \begin{equation*}
    (A + E)^{-1}
    = A^{-1} + \sum_{i = 1}^k (-1)^i (A^{-1} E)^i A^{-1}
    + A^{-1/2} D A^{-1/2}.
  \end{equation*}
\end{lemma}
\begin{proof}
  We have
  $(A^{-1} E)^i A^{-1} = A^{-1/2}(A^{-1/2} E A^{-1/2})^i A^{-1/2}$,
  and $-\delta I \preceq A^{-1/2} E A^{-1/2} \preceq \delta I$
  by assumption. Thus $\opnorms{A^{-1/2} E A^{-1/2}}^i \le \delta^i$.
  In particular, we can therefore perform the standard matrix
  inverse expansion that
  \begin{equation*}
    (A + E)^{-1}
    = A^{-1} + \sum_{i = 1}^\infty (-1)^i (A^{-1} E)^i A^{-1}
    = A^{-1} + A^{-1/2} \sum_{i = 1}^\infty (-1)^i (A^{-1/2} E A^{-1/2})^i
    A^{-1/2}.
  \end{equation*}
  Now note that
  \begin{equation*}
    \opnormbigg{\sum_{i = k + 1}^\infty
      (-1)^i (A^{-1/2} E A^{-1/2})^i}
    \le \sum_{i = k + 1}^\infty
    \opnorm{A^{-1/2} E A^{-1/2}}^i
    \le \sum_{i = k + 1}^\infty \delta^i
    = \frac{\delta^{k + 1}}{1 - \delta}.
  \end{equation*}
  Letting $D = \sum_{i = k + 1}^\infty (-1)^i (A^{-1/2} E A^{-1/2})^i$ completes the proof.
\end{proof}

With Lemma~\ref{lemma:advanced-matrix-inverse-perturbation}, we can perform
the manipulations of the gradient conditions for optimality of $\theta(P_n)$
with the necessary Hessian perturbations.
\begin{lemma}
  \label{lemma:perturbation-with-errors}
  Let $P_n, P_n'$ be neighboring samples and the conditions of
  Proposition~\ref{proposition:advanced-parameter-self-bounding} hold.  Then
  $\theta = \theta(P_n)$ and $\theta' = \theta(P_n')$ exist, and $\gamma
  \defeq \selftwoconst \ltwo{\theta - \theta'} \radius(\mc{X}) < 1$.
  Additionally, there is a symmetric matrix $D$ satisfying $-\frac{\gamma}{1
    - \gamma} \preceq D \preceq \frac{\gamma}{1 - \gamma}$ for which
  \begin{equation*}
    \theta' - \theta
    = (P_n \ddot{\loss}_\theta
    + \lambdareg I)^{-1} (P_n - P_n') \dot{\loss}_{\theta'}
    + (P_n \ddot{\loss}_\theta + \lambdareg I)^{-1/2} D
    (P_n \ddot{\loss}_\theta + \lambdareg I)^{-1/2}
    (P_n - P_n') \dot{\loss}_{\theta'}.
  \end{equation*}
\end{lemma}
\begin{proof}
  Let $\radconst =
  \radius(\mc{X})$ and $\lambda = \lambdamin(P_n)$
  for shorthand.
  By
  Proposition~\ref{proposition:basic-parameter-self-bounding}, for any $\rho
  \in (0, 1)$ such that $\rho\lambda + \lambdareg
  \ge \frac{4 \selftwoconst \lipobj r}{(1 -
    \rho) n} + \frac{\lipgrad}{n}$, we have $\ltwo{\theta(P_n) -
    \theta(P_n')} \le \frac{4 \lipobj }{n} \frac{1}{\lambda \rho
    + \lambdareg - \lipgrad / n} \le
  \frac{1 - \rho}{\selftwoconst r}$.  The solution $\theta' = \theta(P_n')$
  then exists, and so a Taylor expansion gives
  \begin{align*}
    0 = P_n' \dot{\loss}_{\theta'}
    + \nabla \reg(\theta')
    & = P_n \dot{\loss}_{\theta'} + (P_n' - P_n) \dot{\loss}_{\theta'}
    + \nabla \reg(\theta') \\
    & = P_n \dot{\loss}_\theta
    + \nabla \reg(\theta)
    + (P_n \ddot{\loss}_\theta + \lambdareg I +  E) (\theta' - \theta)
    + (P_n' - P_n) \dot{\loss}_{\theta'},
  \end{align*}
  where the error matrix
  $E$ satisfies
  \begin{equation*}
    - \selftwoconst P_n \ddot{\loss}_\theta
    |(\theta - \theta')^T X|
    \preceq E \preceq \selftwoconst P_n \ddot{\loss}_\theta
    |(\theta - \theta')^T X|
  \end{equation*}
  by assumption, as
  $|(\theta - \theta')^T x| \le
  (1 - \rho) / \selftwoconst$, and so
  the self-bounding conditions apply.

  Define the Hessian $H = P_n \ddot{\loss}_\theta + \lambdareg I$ for shorthand.
  So long as $\sup_{x \in \mc{X}} (\theta - \theta')^T x < 1 / \selftwoconst$,
  for which it is sufficient that $\ltwo{\theta - \theta'} r < 1 /
  \selftwoconst$, $H + E$ is invertible,
  because in this case
  $-H \prec E \prec H$.
  With the choice
  $\gamma = \selftwoconst \ltwo{\theta - \theta'} \radius(\mc{X})$, we have
  $\gamma \le 1 - \rho < 1$, where $\rho \in (0, 1)$ is
  as in Condition~\eqref{eqn:lambda-self-bounded-big-enough}.
  Lemma~\ref{lemma:advanced-matrix-inverse-perturbation} therefore
  gives that
  \begin{equation*}
    (H + E)^{-1}
    = H^{-1}
    + H^{-1/2}
    D H^{-1/2}
  \end{equation*}
  for a symmetric matrix $D$ satisfying $-\gamma / (1 - \gamma)
  \preceq D \preceq \gamma / (1 - \gamma)$, and
  rewriting the Taylor expansion above then gives
  \begin{align*}
    (\theta' - \theta)
    & = (H + E)^{-1} (P_n - P_n') \dot{\loss}_{\theta'}
    \nonumber \\
    & = H^{-1}
    (P_n - P_n') \dot{\loss}_{\theta'}
    + H^{-1/2} D H^{-1/2} (P_n - P_n') \dot{\loss}_{\theta'},
  \end{align*}
  as desired.
\end{proof}

Taking norms of both sides in Lemma~\ref{lemma:perturbation-with-errors}
yields the inequality
\begin{equation}
  \label{eqn:nicer-recursive-bound}
  \ltwo{\theta - \theta'}
  \le \ltwo{(P_n \ddot{\loss}_\theta + \lambdareg I)^{-1}
    (P_n - P_n') \dot{\loss}_{\theta'}}
  + \frac{\opnorm{D}}{
    \lambdamin(P_n) + \lambdareg} \ltwo{(P_n - P_n') \dot{\loss}_{\theta'}}.
\end{equation}
Let $t = \ltwo{\theta(P_n) - \theta(P_n')}$ for shorthand;
we will find bounds on $t$ so that the bound~\eqref{eqn:nicer-recursive-bound}
holds. Substituting in the earlier bounds on the error matrix $D$, we have
$\opnorm{D} \le \selftwoconst t r
/ (1 - \selftwoconst t r)$, and
we see that $t$ necessarily satisfies
\begin{equation*}
  t \le \ltwo{(P_n \ddot{\loss}_\theta + \lambdareg I)^{-1}
    (P_n - P_n') \dot{\loss}_{\theta'}}
  + \frac{\selftwoconst r t}{(\lambdamin(P_n) + \lambdareg)
    (1 - \selftwoconst r t)}
  \ltwo{(P_n - P_n')\dot{\loss}_{\theta'}}.
\end{equation*}
As we note above,
Condition~\eqref{eqn:lambda-self-bounded-big-enough} is sufficient to
guarantee that $1 - \selftwoconst r t \ge \rho > 0$, and
so
for $\tau_2 = \ltwos{(P_n -
  P_n')\dot{\loss}_{\theta'}}$ and $\lambda = \lambdamin(P_n) + \lambdareg$,
so the preceding display implies that
\begin{equation*}
  t \le \frac{\tau_2}{\lambda} + \frac{\selftwoconst r t}{1 - \selftwoconst r t}
  \frac{\tau_2}{\lambda}
  ~~ \mbox{i.e.} ~~
  0 \le \selftwoconst r t^2 - t + \frac{\tau_2}{\lambda}.
\end{equation*}
Solving the quadratic, this implies
\begin{equation*}
  t \le \frac{1 - \sqrt{1 - 4 \frac{\selftwoconst r \tau_2}{\lambda}}}{
    2 \selftwoconst r}.
\end{equation*}


\section{Releasing private quantities via recursive bounds}
\label{sec:private-recursions}

As we outline in Section~\ref{sec:prelim-ideas}, the first stage in our
algorithms is to privately release a lower bound on $\lambdamin(P_n)$. In
this section, we develop the tools to do so by introducing new algorithms
for privately releasing statistics whose values on neighboring samples can
be bounded recursively, that is, by functionals of the statistic itself.
The prototypical example on which we focus is the minimal eigenvalue
$\lambdamin(P_n)$, which satisfies a number of recursive bounds. From
Corollary~\ref{corollary:generic-eigenvalue-change}, for example, with
$\lambdareg = 0$, we see
that
\begin{equation*}
  \lambdamin(P_n')
  \ge \frac{\lambdamin(P_n)}{2}
  \left[1 + \sqrt{1 - \frac{8 \lipobj \liphess}{n \lambdamin^2(P_n)}}
    \right] - \frac{\lipgrad}{n},
\end{equation*}
so long as Condition~\eqref{eqn:lambda-is-large-enough-to-start}
holds, while
Corollary~\ref{corollary:eigenvalue-qsc-change} gives a sharper
guarantee for generalized linear models with quasi-self-concordance.

To develop the mechanisms, we leverage \citeauthor{AsiDu20nips}'s
approximate inverse sensitivity mechanism~\citep{AsiDu20nips}, which
releases a real-valued statistic $f(P_n)$ with $(\diffp,
\delta)$-differential privacy and high accuracy. We first recapitulate their
mechanism, then show how to apply it to eigenvalues in
Section~\ref{sec:lambda-release}.  Recall the modulus of
continuity~\eqref{eqn:modulus-continuity} for a function $f$ acting on
finitely supported measures,
\begin{equation*}
  \modcont_f(P_n; k) \defeq \sup_{P_n' \in \mc{P}_n} \left\{ |f(P_n) - f(P_n')|
  ~\mbox{s.t.}~ n \tvnorm{P_n - P_n'} \le k \right\}.
\end{equation*}
The mechanism requires a set of upper bounding functions
$\upperls_i : \mc{P}_n \to \R_+$, $i = 1, \ldots, n$, acting on the sample
measures $\mc{P}_n$ and satisfying
\begin{equation*}
  \modcont_f(P_n; 1) \le \upperls_1(P_n)
\end{equation*}
and the local upper bounding condition that
\begin{equation}
  \upperls_k(P_n) \le \upperls_{k + 1}(P_n')
  ~ \mbox{for~all~} k \in [n]
  ~ \mbox{and}~ P_n, P_n' ~ \mbox{with}~
  \tvnorm{P_n - P_n'} \le \frac{1}{n}.
  \label{eqn:local-upper-bound-cond}
\end{equation}
The upper inverse modulus of continuity
\begin{equation}
  \label{eqn:upper-inverse-modulus}
  \upperinvmodcont_f(P_n; t)
  \defeq \min\left\{ k \in \N \mid \sum_{i = 1}^k \upperls_i(P_n) \ge
  |t - f(P_n)| \right\}
\end{equation}
then defines the approximate inverse sensitivity mechanism
\begin{equation}
  \label{eqn:approximate-mechanism}
  \tag{\textsc{M.a}}
  \P(M(P_n) \in A)
  = \frac{\int_A e^{-\diffp \upperinvmodcont_f(P_n; t) / 2} d\mu(t)}{
    \int_{\mc{T}} e^{-\diffp \upperinvmodcont_f(P_n; t) / 2} d\mu(t)}.
\end{equation}
\citet[Thm.~1]{AsiDu20nips} show that the mechanism is
differentially private:
\begin{corollary}
  The length function~\eqref{eqn:upper-inverse-modulus} satisfies
  $|\upperinvmodcont_f(P_n; t) - \upperinvmodcont_f(P_n'; t)| \le n
  \tvnorm{P_n - P_n'}$, and hence the
  mechanism~\eqref{eqn:approximate-mechanism} is $\diffp$-differentially
  private.
\end{corollary}

Using the mechanism~\eqref{eqn:approximate-mechanism}, we show how to
release one-dimensional quantities whose stability is governed by the
quantity itself, after which (in Sec.~\ref{sec:lambda-release}) we show
how this applies more concretely to releasing minimal (and maximal)
eigenvalues.

\subsection{Private one-dimensional statistics via recursive bounds}
\label{sec:private-one-dim-recursions}

Let $C \subset \R$ be a closed convex set (typically, this will be
$\openright{0}{\infty}$ or an interval $[a, b]$). A mapping $\recurse : \R
\to \R$ is an \emph{accelerating decreasing recursion} on $C$
if for all $\lambda
\le \lambda' \in C$, we have
$\recurse(\lambda) \le \lambda$ and the acceleration condition
\begin{equation}
  \label{eqn:acceleration}
  \lambda - \recurse(\lambda)
  \ge \lambda' - \recurse(\lambda')
  ~~ \mbox{whenever}~~ \recurse(\lambda) > \inf C,
\end{equation}
so that the recursion $\lambda \mapsto \recurse(\lambda)$ accelerates
toward the lower limit $\inf C$.  If
$\recurse$ is differentiable, then
condition~\eqref{eqn:acceleration} is equivalent to
\begin{equation*}
  \recurse'(\lambda) \ge 1
  ~~~ \mbox{whenever} ~~~
  \recurse(\lambda) > \inf C.
\end{equation*}
(To see this, set $\lambda' = \lambda + \delta$ and take
$\delta \downarrow 0$.) Additionally, if
$\recurse$ is accelerating and $H_\rho$ is the hard-thresholding operator
$H_\rho(t) = t \indic{t \ge \rho}$, then $H_\rho \circ \recurse$ is
also an accelerating decreasing recursion by inspection.
Previewing our applications to eigenvalues,
examples include the linear mapping
$\recurse(\lambda) = \hinge{\lambda - a}$,
or the mapping
$\recurse(\lambda) = \lambda - a/\lambda - b$ with $a, b \ge 0$, which
are both accelerating over $\lambda \in \R_+$.
Define the $k$-fold composition
\begin{equation*}
  \recurse^k \defeq
  \underbrace{\recurse \circ \cdots \circ \recurse}_{k~\textup{times}}
\end{equation*}

Assume we have a statistic $\lambda : \mc{P}_n \to C$ satisfying
the one-step recursive guarantee that $\lambda(P_n') \ge
\recurse(\lambda(P_n))$ whenever $P_n', P_n$ are neighboring.
Define the upper bound sequence
\begin{equation}
  \label{eqn:general-recurse-upper}
  U_k(P_n) \defeq
  \begin{cases}
    \recurse^{k-1}(\lambda(P_n)) - \recurse^k(\lambda(P_n))
    & \mbox{if~} \recurse^k(\lambda(P_n)) > \inf C \\
    +\infty & \mbox{otherwise}.
  \end{cases}
\end{equation}
We claim the following
lemma, which shows that this sequence upper bounds
the local modulus of continuity
as required in the definition~\eqref{eqn:upper-inverse-modulus}.

\begin{lemma}
  \label{lemma:basic-recursion}
  Let $C$ be closed convex,
  $\lambda : \mc{P}_n \to C$,
  and $\recurse : C \to \R$ be an accelerating decreasing recursion.
  Assume that $\lambda$ satisfies the
  bounds
  \begin{equation*}
    \lambda(P_n) - \recurse(\lambda(P_n))
    \ge \lambda(P_n') - \lambda(P_n)
    \ge \recurse(\lambda(P_n)) - \lambda(P_n)
  \end{equation*}
  for all neighboring $P_n, P_n' \in \mc{P}_n$. Then the
  mapping~\eqref{eqn:general-recurse-upper} satisfies
  \begin{equation*}
    U_1(P_n) \ge \modcont_\lambda(P_n; 1) ~~ \mbox{and} ~~
    U_k(P_n) \le U_{k + 1}(P_n')
    ~ \mbox{for all} ~
    k \in \N.
  \end{equation*}
\end{lemma}
\begin{proof}
  Fix $\lambda_0 = \lambda(P_n)$,
  and let $\lambda_0' = \lambda(P_n')$ for some
  $P_n, P_n'$ with $\tvnorms{P_n - P_n'} \le 1/n$.
  We have the recursions
  \begin{equation*}
    \lambda_{k + 1} \defeq \recurse(\lambda_k)
    ~~ \mbox{and} ~~
    \lambda'_{k + 1} = \recurse(\lambda_k'),
  \end{equation*}
  and define
  $u_k \defeq U_k(P_n) = \lambda(P_n) - \recurse_k(\lambda(P_n))$
  and $u_k' \defeq U_k(P_n') = \lambda(P_n') - \recurse_k(\lambda(P_n'))$
  for shorthand.
  The first claim that
  $u_1 \ge \modcont_\lambda(P_n; 1)$ is immediate,
  as $u_1 = \recurse(\lambda(P_n)) - \lambda(P_n)
  \ge |\lambda(P_n') - \lambda(P_n)|$.
  So we need only show that
  $u_k \le u_{k + 1}'$ for all $k$, which we do via induction,
  demonstrating both that $u_k \le u_{k+1}'$ and
  that 
  $\lambda_k \ge \lambda_{k + 1}'$ for all $k$.
  
  \emph{Base case.} The case $k = 0$ is that
  $\lambda_0 \ge \lambda_1'$,
  which is equivalent to the claim that
  \begin{equation*}
    \lambda_0 \ge \lambda_1' = \recurse(\lambda_0'),
    ~~ \mbox{i.e.} ~~
    \lambda_0 - \lambda_0' \ge \recurse(\lambda_0') - \lambda_0',
  \end{equation*}
  which is immediate by the assumed bounds on $\lambda(\cdot)$.

  \emph{Induction.} Assume that the inequalities
  $u_i \le u_{i + 1}'$ and
  $\lambda_i \ge \lambda_{i+1}'$ hold for all $i < k$. We wish to show
  they hold for $i = k$. The monotonicity of the recursive
  mapping guarantees that
  $\lambda_k = \recurse(\lambda_{k - 1})
  \ge \recurse(\lambda_k') = \lambda_{k + 1}'$ by the
  assumption that $\lambda_{k - 1} \ge \lambda_k'$.
  If $\lambda_k = \inf C$ then $\lambda_{k+1}' = \inf C$, and
  so $u_k = u_{k + 1}' = +\infty$. Otherwise,
  we have $\lambda_k > \inf C$ and then
  \begin{equation*}
    u_k = \lambda_{k-1} - \lambda_k
    = \lambda_{k - 1} - \recurse(\lambda_{k-1})
    \stackrel{(i)}{\le} \lambda_k' - \recurse(\lambda_{k-1}')
    \stackrel{(ii)}{\le} u'_{k + 1},
  \end{equation*}
  where inequality~$(i)$ is the acceleration
  condition~\eqref{eqn:acceleration}
  and~$(ii)$  is an
  equality unless $\recurse(\lambda_{k-1}') \le \inf C$, in which case
  $u_{k+1}' = +\infty$. This gives the induction and the lemma.
\end{proof}

Inverting $\recurse^k$ provides a clean approach to releasing lower bounds on
$\lambda(P_n)$. Define the inverse
\begin{equation*}
  (\recurse^k)^{-1}(\gamma)
  \defeq \sup \left\{\lambda \ge \inf C \mid \recurse^k(\lambda) \le \gamma
  \right\}.
\end{equation*}
If we have a high probability guarantee on a (random) $N$ that
$\recurse^N(\lambda(P_n)) > \inf C$, then the monotonicity of
$\recurse$ guarantees that $\lambda(P_n) \ge
(\recurse^N)^{-1}(\inf C)$, leading to the following algorithm.

\algbox{
  \label{alg:generic-lambda-release} A private lower bound on $\lambda(P_n)$
}{%
  \textbf{Require:}
  Privacy parameters $\diffp \ge 0$ and $\delta \in (0, 1)$
  and
  an accelerating decreasing recursion $\recurse : C \to \R$
  satisfying
  $|\lambda(P_n) - \lambda(P_n')| \le \lambda(P_n) - \recurse(\lambda(P_n))$
  for all neighboring empirical distributions $P_n$, $P_n'$.
  \begin{enumerate}[i.]
  \item Set
    \begin{equation*}
      \what{N} \defeq \min\left\{N \in \N \mid \recurse^N(\lambda(P_n))
      \le \inf C \right\} + \frac{1}{\diffp} \laplace(1).
    \end{equation*}
  \item Set $k(\diffp, \delta) = \frac{1}{\diffp} \log\frac{1}{2\delta}$,
    then return $\what{N}$ and
    \begin{equation*}
      \what{\lambda} = \left(\recurse^{\what{N} - k(\diffp,\delta)}
      \right)^{-1}(\inf C).
    \end{equation*}
  \end{enumerate}
}

The discussion above and that
if $W \sim \laplace(1)$, we have
$\P(W \ge \log \frac{1}{2 \delta})
= \half \int_{\log\frac{1}{2 \delta}} e^{-t} dt
= \delta$,
then immediately yield the following proposition.
\begin{proposition}
  \label{proposition:good-lambda}
  Algorithm~\ref{alg:generic-lambda-release} is $\diffp$-differentially
  private, and 
  $\lambda(P_n) \ge \what{\lambda}$
  with probability at least $1 - \delta$.
\end{proposition}
\begin{proof}
  By the construction of the upper
  bound mapping~\eqref{eqn:general-recurse-upper},
  we have
  \begin{equation*}
    \upperinvmodcont_\lambda(P_n; \inf C)
    = \min\{N \in \N \mid \recurse^N(\lambda(P_n))
    \le \inf C\}.
  \end{equation*}
  Thus
  for $W \sim \laplace(1)$, we have
  $\what{N} \eqdist \upperinvmodcont_\lambda(P_n; \inf C) + \frac{1}{\diffp}W$,
  and $\what{N}$ is $\diffp$-differentially private (and so
  is $\what{\lambda}$ by post-processing).
  To obtain that $\P(\lambda(P_n) \ge \what{\lambda}) \ge 1 - \delta$, note that
  $\P(\what{N} < \upperinvmodcont_\lambda(P_n; \inf C)
  + k(\diffp, \delta))
  = \P(W < \log\frac{1}{2 \delta}) = 1 - \delta$.
  On the event $\what{N} < \upperinvmodcont_\lambda(P_n; \inf C)
  + k(\diffp, \delta)$, 
  we know that $\recurse^{\what{N} - k(\diffp, \delta)}(\lambda(P_n))
  > \inf C$, and so monotonicity of
  $\recurse$ guarantees
  $\lambda(P_n) \ge \what{\lambda}$.
\end{proof}

\subsection{Releasing eigenvalues for M-estimation problems}
\label{sec:lambda-release}

We finish this section by giving explicit algorithms for releasing minimal
eigenvalues for M-estimation problems~\eqref{eqn:basic-m-estimator} as well
as proving Corollaries~\ref{corollary:lambda-qsc-accuracy-privacy}
and~\ref{corollary:lambda-max-qsc-guarantee}.
Algorithm~\ref{alg:generic-lambda-release} guarantees
privacy, by Proposition~\ref{proposition:good-lambda}, so if we can
demonstrate a recursion $\recurse$ for $\lambdamin(P_n)$ satisfying
\begin{equation*}
  |\lambdamin(P_n) - \lambdamin(P_n')| \le \lambdamin(P_n)
  - \recurse(\lambdamin(P_n)),
\end{equation*}
then we may simply apply Algorithm~\ref{alg:generic-lambda-release}.

We begin with a generic lemma, which gives two somewhat more sophisticated
recursions, based (respectively) on
Corollaries~\ref{corollary:generic-eigenvalue-change}
and~\ref{corollary:eigenvalue-qsc-change}.
(See Appendix~\ref{sec:proof-acceleration} for a proof.)
\begin{lemma}
  \label{lemma:accelerating-recursions}
  Let $a, b, c \ge 0$ and $\lambda_0 \ge 0$. Then
  the functions
  \begin{equation*}
    \recurse(\lambda) = \frac{\lambda}{2}
    \left[1 + \sqrt{1 - \frac{a}{\lambda^2}}\right]
    - b
    ~~~ \mbox{or} ~~~
    \recurse(\lambda) = \lambda
    \left[2 - \exp\left(b \left(1 - \sqrt{1 - \frac{a}{\lambda + \lambda_0}}
      \right)\right)\right] - c
  \end{equation*}
  are accelerating decreasing recursions for, respectively,
  $\lambda > a$ and $\lambda + \lambda_0 > a$.
\end{lemma}


\subsubsection{Releasing the minimal eigenvalue for a general smooth loss}

We now revisit Corollary~\ref{corollary:generic-eigenvalue-change},
which applies when all we know are that the losses
$\loss$ have Lipschitz derivatives.
In this case, we have the following corollary.
\begin{corollary}
  \label{corollary:general-lambda-release}
  Let the loss $\loss$ have
  $\lipletter_i$-Lipschitz continuous $i$th derivative
  for $i = 0, 1, 2$.
  Define the recursion
  \begin{equation*}
    \recurse(\lambda) \defeq
    \max\left\{\frac{\lambda}{2}
    \left[1 + \sqrt{1 - \frac{8 \lipobj \liphess}{n \lambda^2}}
      \right] - \frac{\lipgrad}{n}, \lambdareg\right\},
  \end{equation*}
  where $\sqrt{x} = -\infty$ for $x \le 0$. Then
  Algorithm~\ref{alg:generic-lambda-release} applied with
  this recursion
  releases an \mbox{$\diffp$-differentially}
  private $\what{\lambda}$. With probability at least $1 - \delta$,
  $\what{\lambda}$ satisfies both
  $\what{\lambda} \le \lambdamin(P_n) + \lambdareg$ and
  \begin{equation*}
    \what{\lambda} \ge \lambdamin(P_n) + \lambdareg
    - O(1) \frac{1}{\diffp} \log \frac{1}{\delta}
    \left[\frac{\lipobj \liphess}{(\lambdamin(P_n) + \lambdareg) n}
      + \frac{\lipgrad}{n} \right].
  \end{equation*}
\end{corollary}
\begin{proof}
  The first claim of the corollary is immediate by combining
  Proposition~\ref{proposition:good-lambda},
  Lemma~\ref{lemma:accelerating-recursions} that $\recurse$ is accelerating,
  and the deviation bounds in
  Corollary~\ref{corollary:generic-eigenvalue-change}.

  The guarantees on the relationship
  between $\what{\lambda}$ and $\lambdamin(P_n)$ require more work.
  Let $\lambda\opt = \lambdamin(P_n) + \lambdareg$ for shorthand.
  That $\what{\lambda} \le \lambda\opt$ with probability
  $1 - \delta$ is immediate by definition of $\what{N}$.
  To obtain the lower bound $\what{\lambda} \ge \lambda\opt
  - O(\frac{1}{n \diffp})$, introduce the shorthands
  $a = \frac{8 \lipobj \liphess}{n}$ and $b = \frac{\lipgrad}{n}$,
  where we assume $\max\{\sqrt{a}, b\} \ll \lambda\opt$. (Otherwise, the
  guarantee is vacuous.)
  Recall the definition
  $(\recurse^k)^{-1}(0) = \inf\{\lambda \mid \recurse^k(\lambda) = \lambdareg\}$,
  and let $N$ be the smallest value necessary to obtain
  $\recurse^N(\lambda\opt) = \lambdareg$,
  so $\recurse^{N - 1}(\lambda\opt) > 0$.
  Consider a single iteration of the recursion
  \begin{equation*}
    \lambda \mapsto \recurse(\lambda)
    = \frac{\lambda}{2}
    \left(1 + \sqrt{1 - \frac{a}{\lambda^2}}\right) - b
    = \frac{\lambda}{2}
    \left(2 - \frac{a}{\lambda^2} + O(a^2 / \lambda^4)\right) - b
    = \lambda - \frac{a}{\lambda}
    - b - O\left(\frac{a^2}{\lambda^3}\right).
  \end{equation*}
  Then for some (numerical) constant $c$ the recursion
  $\recurse^N(\lambda\opt - \frac{a}{\lambda\opt} - b - c
  \frac{a^2}{{\lambda\opt}^3}) = \lambdareg$, so that
  $(\recurse^N)^{-1}(\lambdareg) \ge \lambda\opt - \frac{a}{\lambda\opt} - b -
  O(\frac{a^2}{{\lambda\opt}^3})$.  Applying $k$ steps of the recursion with the
  above linearization, we obtain
  \begin{equation*}
    \recurse^k(\lambda\opt) = \lambda\opt
    - k \left(\frac{a}{\lambda\opt} + b\right)
    - O\left(k \frac{a^2}{{\lambda\opt}^3}
    \right).
  \end{equation*}
  For $k = k(\diffp, \delta) = \frac{1}{\diffp} \log \frac{1}{2 \delta}$
  we have $\what{N} \ge N - k(\diffp, \delta)$ with probability
  at least $1 - \delta$,
  so recognizing that $a^2 / {\lambda\opt}^2 \ll a/\lambda\opt$
  gives
  \begin{equation*}
    (\recurse^{\what{N}})^{-1}(\lambdareg)
    \ge
    \recurse^{k + 1}\left((\recurse^N)^{-1}(\lambdareg)\right)
    \ge \lambda\opt - O(1) k \left(\frac{a}{\lambda\opt} + b\right)
    + O\left(k \frac{a^2}{{\lambda\opt}^3}\right).
  \end{equation*}
  Substituting for $a$ and $b$ gives the corollary
  once we recognize that it is vacuous whenever
  $k a / \lambda\opt \gtrsim \lambda\opt$.
\end{proof}

\subsubsection{Proof of Corollary~\ref{corollary:lambda-qsc-accuracy-privacy}}
\label{sec:proof-lambda-qsc-accuracy-guarantee}

We can revisit Corollary~\ref{corollary:eigenvalue-qsc-change} to apply
to (quasi) self-concordant losses.
Define
\begin{equation*}
  a = \frac{4 \lipobj \selftwoconst \radius(\mc{X})}{n},
  ~~
  b = \frac{1}{2 \selftwoconst}, 
  ~~
  c = \frac{\lipgrad}{n}.
\end{equation*}
Then
the defined recursion satisfies
\begin{equation*}
  \recurse(\lambda) =
  \lambda \left(2 - \exp\left(b
  \left(1 - \sqrt{1 - \frac{a}{\lambda + \lambdareg}}
  \right)\right)\right) - c
\end{equation*}
as in Lemma~\ref{lemma:accelerating-recursions}
(so long
as $\lambda$ satisfies Condition~\eqref{eqn:lambda-self-bounded-big-enough},
and hard-thresholding to 0 otherwise) so that it is an accelerating
and decreasing recursion.
Corollary~\ref{corollary:eigenvalue-qsc-change}
shows that $\recurse$ bounds the changes in $\lambdamin(P_n)$ to
$\lambdamin(P_n')$.   Proposition~\ref{proposition:good-lambda} thus
gives the differential privacy.

For the claimed lower bound on $\what{\lambda}$, we consider
the behavior of $\recurse$ for $\lambda$ near $\lambdamin(P_n)$.
Let $\radconst = \radius_2(\mc{X})$ for shorthand.
Under the assumption that
$C \frac{\lipobj \radconst}{n} \le \lambdamin(P_n) + \lambdareg$
for a suitably large numerical constant $C$, we have
\begin{equation*}
  b \left(1 - \sqrt{1 - \frac{a}{\lambda + \lambdareg}}\right)
  = b \left(\frac{a}{2 (\lambda + \lambdareg)} + O(a^2 / (\lambda
  + \lambdareg)^2)\right)
  = \frac{2 \lipobj \radconst}{n (\lambda + \lambdareg)}
  + O\left(\frac{\lipobj^2 \radconst^2}{
    n^2 (\lambda + \lambdareg)^2}\right)
\end{equation*}
assuming that $\selftwoconst$ and $\rho$ are numerical constants. Ignoring
the higher order terms and using that $e^t = 1 + t + O(t^2)$, we thus obtain
\begin{align*}
  \recurse(\lambda)
  & = \lambda \left(1 - \frac{2 \lipobj \radconst}{n (\lambda + \lambdareg)}
  - O\left(\frac{\lipobj^2 \radconst^2}{n^2 (\lambda + \lambdareg)^2}
  \right)\right)
  - \frac{\lipgrad}{n} \\
  & =
  \lambda - \frac{2 \lipobj \radconst}{n} \frac{\lambda}{\lambda + \lambdareg}
  - \frac{\lipgrad}{n}
  - O\left(\frac{\lipobj^2 \radconst^2}{n^2 (\lambda + \lambdareg)}\right).
\end{align*}
Following the same strategy as that in the proof of
Corollary~\ref{corollary:general-lambda-release}, we see that
$k$ steps of this linearization yields
\begin{equation*}
  \recurse^k(\lambda) = \lambda
  - k \left(\frac{2 \lipobj \radconst}{n}
  \frac{\lambda}{\lambda + \lambdareg} - \frac{\lipgrad}{n}\right)
  - O\left(\frac{k \lipobj^2 \radconst^2}{n^2 (\lambda + \lambdareg)}\right).
\end{equation*}
Then if $N$ is the smallest value necessary to obtain $\recurse^N(\lambda)
= 0$ for $\lambda = \lambdamin(P_n)$, we have $\recurse^{N-1}(\lambda)
> 0$, and $(\recurse^N)^{-1}(0)
\ge \lambda - \frac{2 \lipobj \radconst}{n}
\frac{\lambda}{\lambda + \lambdareg} - \frac{\lipgrad}{n}
- O(\frac{\lipobj^2 \radconst^2}{n^2 (\lambda + \lambdareg)})$. Setting
$k = k(\diffp, \delta) = \frac{1}{\diffp}\log \frac{1}{2 \delta}$,
we have $\what{N} \ge N - k(\diffp, \delta)$ with probability
at least $1 - \delta$, and as $\frac{\lipobj^2 \radconst^2}{n^2 (\lambda
+ \lambdareg)}
\lesssim \frac{\lipobj \radconst}{n}$ under the settings of the corollary,
we have
\begin{equation*}
  (\recurse^{\what{N}})^{-1}(0) \ge \lambda
  - O(1) k(\diffp, \delta) \left(\frac{\lipobj \radconst}{n}
  \frac{\lambda}{\lambda + \lambdareg}
  + \frac{\lipgrad}{n}
  \right)
\end{equation*}
as desired.

\subsubsection{Proof of Corollary~\ref{corollary:lambda-max-qsc-guarantee}}
\label{sec:proof-lambda-max-qsc-guarantee}


We first recognize that given any $\lambda \le \lambdamin(P_n) + \lambdareg$,
Corollary~\ref{corollary:eigenvalue-qsc-change}
guarantees that
\begin{equation*}
  \lambdamax(P_n') \le \lambdamax(P_n)
  \left(1 + \concordantfunc(\tparamchange(\lambda) \cdot \radius(\mc{X}))
  \right)
  + \frac{\lipgrad}{n}.  
\end{equation*}
So the bounds required for recursive algorithms to provide privacy hold.
For the actual privacy guarantee, we rely on the composition
guarantee of Lemma~\ref{lemma:conditional-composition}, and
privacy follows from Proposition~\ref{proposition:good-lambda}
as in the proof of Corollary~\ref{corollary:lambda-qsc-accuracy-privacy}.

The proof of accuracy is also similar to that of
Corollary~\ref{corollary:lambda-qsc-accuracy-privacy}.  Let $\radconst =
\radius_2(\mc{X})$ as before
and $\what{\lambda} = \what{\lambda}_{\min}(P_n) + \lambdareg$.
Using the assumption that $\frac{\lipobj
  \radconst}{n (\lambdamin(P_n) + \lambdareg)}
\le 1/C$ for a suitably large numerical
constant $C$, the output $\what{\lambda}_{\min}$ of
Algorithm~\ref{alg:self-concordant-glm-lambda-release} satisfies
$\what{\lambda}_{\min} \ge \lambdamin(P_n) - O(1) k(\diffp, \delta)
(\frac{\lipobj \radconst}{n} + \frac{\lipgrad}{n}) \gtrsim \lambdamin(P_n)$
with probability at least $1 - \delta$.  So on this event, we obtain
\begin{align*}
  \recurse(\lambda)
  & = \lambda \exp\left(\tparamchange(\what{\lambda})
  \radconst\right) + \frac{\lipgrad}{n} \\
  & = \lambda + \frac{2 \lipobj \radconst}{n}
  \cdot \frac{\lambda}{\lambdamin(P_n) + \lambdareg}
  + \frac{\lipgrad}{n}
  + O(1) \frac{\lipobj^2 \radconst^2}{n^2 (\lambdamin(P_n) + \lambdareg)}
  \cdot \frac{\lambda}{\lambdamin(P_n) + \lambdareg},
\end{align*}
where we have used Corollary~\ref{corollary:lambda-qsc-accuracy-privacy}
so that $t(\what{\lambda})
= \frac{2 \lipobj}{n (\lambdamin(P_n) + \lambdareg)}(1 + o(1))$,
that $e^t = 1 + t + O(t^2)$ for $t$ small.
By assumption $\frac{\lipobj \radconst}{n (\lambdamin(P_n) + \lambdareg)}
\le C^{-1}$, we obtain
\begin{equation*}
  \recurse(\lambda)
  \le \lambda + O(1) \frac{\lipobj \radconst}{n} \cdot \frac{\lambda}{
    \lambdamin(P_n) + \lambdareg} + \frac{\lipgrad}{n}.
\end{equation*}
As in the proof of Corollary~\ref{corollary:lambda-qsc-accuracy-privacy}
(\emph{mutatis mutandis}), if $N$ is the smallest value such that
$\recurse^N(\lambdamax(P_n)) = \lipgrad$, we have
$\recurse^{N-1}(\lambdamax(P_n)) < \lipgrad$ and
$(\recurse^N)^{-1}(\lipgrad) \le \lambdamax(P_n) + O(1) \frac{\lipobj
  \radconst}{n} \cdot \frac{\lambdamax(P_n)}{\lambdamin(P_n) + \lambdareg} +
\frac{\lipgrad}{n}$.
Iterating this $k = k(\diffp, \delta)$ times from $\lambdamax(P_n)$
yields
\begin{equation*}
  \inf\left\{\lambda \mid \recurse^{N - k(\diffp, \delta)}(\lambda)
  \ge \lipgrad \right\}
  \le \lambdamax(P_n) + O(1) k(\diffp, \delta)
  \frac{\lipobj \radconst}{n} \frac{\lambdamax(P_n)}{\lambdamin(P_n)
    + \lambdareg}
  + k(\diffp, \delta) \frac{\lipgrad}{n}.
\end{equation*}


\section{Private algorithms for parameter release}

When we wish to release a full parameter vector
$\theta(P_n)$, we focus on the more basic composition approaches
from Section~\ref{sec:composition-test-release}.
Letting
\begin{equation*}
  \modcont_\theta(P_n; 1) \defeq \sup\left\{\ltwo{\theta(P_n)
    - \theta(P_n')} \mid n \tvnorm{P_n - P_n'} \le 1\right\}
\end{equation*}
be the modulus of continuity of $\theta(P_n)$ for the $\ell_2$-norm with
respect to changing a single example (the local sensitivity),
Observation~\ref{observation:comp-comp} shows that
if a private random variable $W$ satisfies
$W \ge \modcont_\theta(P_n; 1)$ with high probability,
then
\begin{equation*}
  \theta(P_n) + \normal(0, W^2 \cdot \sigma^2(\diffp, \delta) I_d)
\end{equation*}
is differentially private.  We apply this insight
to the two main cases we consider: generic smooth losses and for
quasi-self-concordant (q.s.c.) generalized linear models (GLMs). For both,
we focus on the unregularized case that the parameter set $\Theta = \R^d$.

\subsection{General smooth losses without regularization}

Focusing on generic smooth losses,
Proposition~\ref{proposition:double-generic-parameter-recursion} shows that
\begin{equation*}
  \modcont_\theta(P_n; 1) \le
  \frac{1}{2 \liphess} \left[\lambdamin(P_n) + \lambdareg
    - \sqrt{(\lambdamin(P_n) + \lambdareg)^2
      - \frac{8 \lipobj \liphess}{n}}\right]
\end{equation*}
so long as $\lambdamin(P_n) + \lambdareg \ge \max\{3 \lipgrad / n, \sqrt{12
  \lipobj \liphess / n}\}$, as in
Condition~\eqref{eqn:lambda-is-large-enough-to-start}.  Thus, the following
algorithm is differentially private and releases an approximation to
$\theta(P_n)$.  \algbox{
  \label{alg:basic-parameter-release}
  Parameter release for generic smooth losses
}{%
  \textbf{Require:} privacy level $(\diffp, \delta)$ and
  Lipschitz constants $\lipletter_i$, $i = 0, 1, 2$, of
  loss $\loss$

  \begin{enumerate}[i.]
  \item Let $\what{\lambda}$ be the output of
    Alg.~\ref{alg:generic-lambda-release}
    with privacy $(\diffp/2, \delta/2)$, statistic
    $\lambda(P_n) = \lambdamin(P_n) + \lambdareg$, and
    recursion
    \begin{equation*}
      \recurse(\lambda) =
      \max\left\{\frac{\lambda}{2}
      \left(1 + \sqrt{1 - \frac{8 \lipobj \liphess}{n \lambda^2}}
      \right) - \frac{\lipgrad}{n}, \lambdareg \right\}.
    \end{equation*}
  \item If $\what{\lambda}$ satisfies
    Condition~\eqref{eqn:lambda-is-large-enough-to-start},
    set
    \begin{equation*}
      W \defeq \frac{1}{2 \liphess}
      \left[\what{\lambda} - \sqrt{\what{\lambda}^2
          - \frac{8 \lipobj \liphess}{n}}\right]
    \end{equation*}
    and return
    \begin{equation*}
      \what{\theta}
      = \theta(P_n) + \normal\left(0, W^2 \sigma^2\Big(
      \frac{\diffp}{2}, \frac{\delta}{2}\Big) \cdot I_d\right).
    \end{equation*}
  \end{enumerate}
}

By combining the pieces of our results together, we obtain the following
proposition on the accuracy and privacy of
Algorithm~\ref{alg:basic-parameter-release}.
\begin{proposition}
  \label{proposition:basic-parameter-release}
  The output $\what{\theta}$ of Alg.~\ref{alg:basic-parameter-release} is
  $(\diffp, \delta)$-differentially private. Additionally,
  there exists a numerical constant $C < \infty$ such that if
  \begin{equation*}
    \lambdamin(P_n)
    \ge C \max\left\{\frac{\lipgrad}{n \diffp}
    \log \frac{1}{\delta},
    \sqrt{\frac{\lipobj \liphess}{n \diffp}
      \log \frac{1}{\delta}}\right\}
  \end{equation*}
  then with probability at least $1 - \delta - \gamma$,
  \begin{equation*}
    \ltwobig{\theta(P_n) - \what{\theta}}
    \le C \frac{\lipobj}{n \diffp \lambdamin(P_n)}
    \sqrt{\log\frac{1}{\delta}}
    \left[\sqrt{d} + \sqrt{\log\frac{1}{\gamma}}\right].
  \end{equation*}
\end{proposition}
\begin{proof}
  The privacy guarantee is nearly immediate via
  Corollary~\ref{corollary:general-lambda-release}, which gives that
  $\lambdamin(P_n) + \lambdareg \ge \what{\lambda}$ with probability
  at least $1 - \delta/2$ and $\what{\lambda}$ is $\diffp/2$-differentially
  private.
  Then Observation~\ref{observation:comp-comp}
  guarantees that $\what{\theta}$ is $(\diffp, \delta)$-differentially
  private.

  Corollary~\ref{corollary:general-lambda-release} guarantees
  $\what{\lambda} \ge \lambdamin(P_n)
  - O(1) (\frac{\lipobj \liphess}{\diffp \lambdamin(P_n) n}
  + \frac{\lipgrad}{n \diffp})\log\frac{1}{\delta}$ with
  probability at least $1 - \delta$.
  On this event, so long as the lower bound
  on $\lambdamin(P_n)$ in the statement of the proposition holds
  (for suitably large numerical constant $C$), the random variable
  \begin{equation*}
    W \lesssim \frac{1}{2 \liphess}
    \left[\lambdamin(P_n) - \sqrt{\lambdamin(P_n)^2 - \frac{\lipobj \liphess}{n}
      }\right]
    \lesssim \frac{\lipobj}{n \lambdamin(P_n)}
  \end{equation*}
  by a Taylor approximation of $\sqrt{1 - \gamma} = 1 - \gamma/2 +
  O(\gamma^2)$, valid for $\gamma$ small.  Noting that for any $0 < \gamma <
  1$, a random Gaussian $Z \sim \normal(0, \sigma^2 I)$ satisfies $\ltwo{Z}
  \le \sigma(\sqrt{d} + O(1)\sqrt{\log(1/\gamma)})$ with probability at
  least $1 - \gamma$ (cf.~\cite[Thm.~3.1.1]{Vershynin19}), then because
  $\sigma^2(\diffp, \delta) \lesssim \diffp^{-1} \log \frac{1}{\delta}$, we
  have the proposition.
\end{proof}

\subsection{Quasi-self-concordant GLMs and the
  proof of Theorem~\ref{theorem:self-concordant-release}}
\label{sec:proof-self-concordant-release}

For
q.s.c.\ GLMs, Proposition~\ref{proposition:advanced-parameter-self-bounding}
shows that so long as $\lambdamin(P_n)$ and $\lambdareg$ satisfy
inequality~\eqref{eqn:lambda-self-bounded-big-enough},
then
\begin{equation*}
  \modcont_\theta(P_n; 1) \le
  \tparamchange(\lambdamin(P_n) + \lambdareg),
\end{equation*}
where Eq.~\eqref{eqn:t-param-change} defines the parameter change constant
$\tparamchange(\lambda) = \frac{2 \lipobj}{n \lambda}(1 + o(1))$.  In
this case, by leveraging
Corollary~\ref{corollary:lambda-qsc-accuracy-privacy}, we
can prove the claimed deviation guarantee on $\what{\theta}$ relative
to $\theta(P_n)$.
The privacy guarantee of the theorem follows by combining
Observation~\ref{observation:comp-comp} with
Proposition~\ref{proposition:good-lambda}.

For the accuracy guarantee,
Corollary~\ref{corollary:lambda-qsc-accuracy-privacy}
shows that under the conditions on $\lambdamin(P_n)$ and $\lambdareg$
in the statement
of Theorem~\ref{theorem:self-concordant-release}, we have
$\what{\lambda} \ge \lambdamin(P_n) - O(\diffp^{-1} \log\frac{1}{\delta})
\max\{\frac{\lipgrad}{n}, \frac{\lipobj \radius(\mc{X})}{n}\}$
with probability at least $1 - \delta$, and a Taylor approximation
yields that the parameter change quantity~\eqref{eqn:t-param-change}
satisfies
\begin{equation*}
  \tparamchange(\what{\lambda} + \lambdareg)
  \lesssim \frac{\lipobj}{n (\what{\lambda} + \lambdareg)}
  \lesssim \frac{\lipobj}{n (\lambdamin(P_n) + \lambdareg)}
\end{equation*}
on this event.
Setting $\tparamchange = \tparamchange(\what{\lambda} + \lambdareg)$
and $\sigma^2 = \sigma^2(\diffp, \delta)$,
the quantity $\what{\theta} = \theta(P_n)
+ \normal(0, \tparamchange^2 \sigma^2 I_d)$
satisfies
$\ltwos{\what{\theta} - \theta(P_n)}
\lesssim \tparamchange \sigma \sqrt{d}
(1 + \sqrt{\log\frac{1}{\gamma}})$ with probability
at least $1 - \gamma$.

\subsection{Dimension and accuracy scaling}


As in Section~\ref{sec:dimension-dependence}, let us briefly discuss the
scaling of the accuracy with dimension in
Proposition~\ref{proposition:basic-parameter-release} and
Theorem~\ref{theorem:self-concordant-release} and when these scalings
apply. We focus on the case of a ``typical'' generalized linear modeling
scenario, where we have a loss of the form $\loss_\theta(x, y) = h(y - \<x,
\theta\>)$ or $\loss_\theta(x, y) = h(y\<x, \theta\>)$, as in the robust
regression or (binary) logistic regression
Examples~\ref{example:robust-regression}
and~\ref{example:logistic-regression}, where the covariate vectors $x \in
[-1, 1]^d$ and $h$ has Lipschitz zeroth, first, and second derivatives. Then
we have the scalings
\begin{equation*}
  \lipobj \asymp \sqrt{d},
  ~~ \lipgrad \asymp d,
  ~~ \liphess \asymp d^{3/2},
  ~~
  \radius(\mc{X}) \asymp \sqrt{d}.
\end{equation*}
In both cases, if either of
Algorithms~\ref{alg:self-concordant-glm-lambda-release}
or~\ref{alg:basic-parameter-release} releases an estimate $\what{\theta}$,
\begin{equation}
  \label{eqn:full-param-accuracy}
  \ltwos{\what{\theta} - \theta(P_n)}
  \lesssim \frac{\lipobj}{n \lambdamin(P_n)} \cdot \frac{\sqrt{d \log
      \frac{1}{\delta}}}{\diffp}
\end{equation}
with high probability by
Proposition~\ref{proposition:basic-parameter-release} and
Theorem~\ref{theorem:self-concordant-release}. As in our discussion in the
introduction, for $n$ large, at $\theta = \theta(P_n)$ the local
modulus~\eqref{eqn:modulus-continuity} has scaling
\begin{equation*}
  \modcont_\theta(P_n; 1) \asymp \sup_{x \in \mc{X}, y}
  \frac{1}{n} \ltwo{(P_n \ddot{\loss}_\theta + \lambdareg I)^{-1}
    \dot{\loss}_\theta(x, y)}
  \stackrel{(\star)}{\le} \frac{1}{n} \frac{\lipobj}{\lambdamin(P_n)},
\end{equation*}
where inequality~$(\star)$ holds with
(approximate) equality when the Hessian
$P_n \ddot{\loss}_\theta$ is near a scaled identity matrix
or $\mc{X}$ is a scaled $\ell_2$-ball. By
\citeauthor{CaiWaZh21}'s score attack~\citep{CaiWaZh21},
the additional scaling with
$\sqrt{d} / \diffp$ is unavoidable, making the accuracy
of these algorithms unimprovable in a worst-case sense, though
they adapt to the particular (local) strong convexity of the problem.

At the grossest level, then, the main difference between the algorithms is
when they may actually release parameters, as the accuracy
guarantees~\eqref{eqn:full-param-accuracy} they provide are
indistinguishable. The basic
Algorithm~\ref{alg:basic-parameter-release} states that as soon as
\begin{equation*}
  \lambdamin(P_n) + \lambdareg \gg \max\left\{\frac{d}{n \diffp},
  \frac{d}{\sqrt{n \diffp}}\right\},
\end{equation*}
so that $n \gg d^2$, the algorithm applies, while
Algorithm~\ref{alg:self-concordant-glm-lambda-release}
requires the weaker condition that
\begin{equation*}
  \lambdamin(P_n) + \lambdareg \gg \frac{d}{n \diffp},
\end{equation*}
so that $n \gg d$. (In both cases, we ignore the logarithmic scaling with
$\frac{1}{\delta}$.)
Such a requirement is, at least in the worst case, unavoidable
under differential privacy.


\section{Releasing linear functionals of the parameter}
\label{sec:releasing-linear-functionals}

By combining the algorithms we have developed for releasing minimal
eigenvalues in
Section~\ref{sec:lambda-release},
the privacy guarantees of the propose-test-release
framework in Section~\ref{sec:test-release}, and the stability
bounds in Section~\ref{sec:parameter-stability}, we can finally
return to one of our original motivations: releasing a single
coordinate of the vector $\theta(P_n)$, or, more generally,
releasing
\begin{equation*}
  u^T \theta(P_n)
\end{equation*}
for a unit vector $u$. To develop the methodology, we will require a few
more sophisticated deviation bounds on the parameter $\theta$ and
functionals of $\theta$.  Note from
Lemma~\ref{lemma:perturbation-with-errors} in the proof of
Proposition~\ref{proposition:advanced-parameter-self-bounding}
that under the conditions of the proposition,
for
\begin{equation*}
  \gamma = \gamma(P_n)
  \defeq \selftwoconst \cdot \tparamchange(\lambdamin(P_n) + \lambdareg)
  \radius(\mc{X}) < 1,
\end{equation*}
there exists a symmetric $D$ with
$\opnorm{D} \le \frac{\gamma}{1 - \gamma}$
such that the Hessian $H \defeq P_n \ddot{\loss}_\theta + \lambdareg I$
satisfies
\begin{equation*}
  \theta(P_n') - \theta(P_n)
  = H^{-1} (P_n - P_n') \dot{\loss}_{\theta'} +
  H^{-1/2} D H^{-1/2}
  (P_n - P_n') \dot{\loss}_{\theta'},
\end{equation*}
where we use $\theta' = \theta(P_n')$ and $\theta = \theta(P_n)$.
Then for
a any $\ell_2$-unit vector $u$,
inequality~\eqref{eqn:directional-modulus-bound} holds:
\begin{equation*}
  \left|u^T (\theta(P_n') - \theta(P_n))\right|
  \le
  \stdonedim \defeq
  \diffu(P_n, u)
  + \frac{2 \lipobj}{n (\lambdamin(P_n) + \lambdareg)}
  \cdot \frac{\gamma(P_n)}{1 - \gamma(P_n)}
\end{equation*}
where we recall the directional
sensitivity~\eqref{eqn:directional-difference}
\begin{equation*}
  \diffu(P_n, u)
  = \frac{1}{n} \sup_{g_0, g_1 \in \mc{G}}
  u^T (P_n \ddot{\loss}_{\theta(P_n)} + \lambdareg I)^{-1} (g_0 - g_1).
\end{equation*}
We use the propose-test-release scheme to argue that releasing
\begin{equation*}
  u^T \theta(P_n) + \stdonedim \cdot Z
\end{equation*}
for a Gaussian $Z$ with variance scaling as $\frac{1}{\diffp^2} \log
\frac{1}{\delta}$ is private so long as we can privately certify that
$\lambdamin(P_n)$ is large enough and $\lambdamax(P_n)$ is small enough.

\subsection{Propose-test-release for the local modulus of continuity}
\label{sec:modulus-ratio-bounds}

The approach to the (somewhat) naive release above introduces subtleties,
however, because neighboring samples $P_n, P_n'$ may have different
directional sensitivies $\diffu$, so that even if $u^T \theta(P_n) - u^T
\theta(P_n')$ is small, the magnitude of the noise added may leak
information.
We therefore adopt an approach building out of literature on private mean
estimation algorithms that adapt to the covariance of the underlying
data~\cite{BiswasDoKaUl20, BrownGaSmUlZa21, BrownHoSm23, DuchiHaKu23}.
Thus, we control the ratio
\begin{equation}
  \label{eqn:super-ratio}
  \frac{\stdonedim}{\stdonedim[P_n']}
  = \frac{\diffu(P_n, u) +
    \frac{2 \lipobj}{n (\lambdamin(P_n) + \lambdareg)}
    \frac{\gamma(P_n)}{1 - \gamma(P_n)}
  }{\diffu(P_n', u)
    + \frac{2 \lipobj}{n (\lambdamin(P_n') + \lambdareg)}
      \frac{\gamma(P_n')}{1 - \gamma(P_n')}}.
\end{equation}
We can control this ratio as soon as we have (high
probability) lower bounds $\what{\lambda}_0 \le \lambdamin(P_n)$ and
$\what{\lambda}_1 \ge \lambdamax(P_n)$. To that end, assume there exists a
ratio bounding term $\ratio(u, \lambda)$ for $\lambda = (\lambda_0,
\lambda_1)$ such that whenever $0 \le \lambda_0 \le \lambdamin(P_n)$ and
$\lambdamax(P_n) \le \lambda_1$, we have
\begin{equation}
  \label{eqn:ratio-to-control}
  \frac{1}{1 + \ratio^2(u, \lambda)}
  \le \frac{\stdonedim^2}{\stdonedim[P_n']^2}
  \le 1 + \ratio^2(u, \lambda)
  ~~ \mbox{for~all~neighboring}~P_n, P_n'.
\end{equation}
Once we have such a guarantee,
then so long as
$\diffp \ge \half (1 + \Phi^{-1}(1 - \delta/2)^2) \ratio^2(u, \what{\lambda})$,
we release
\begin{equation}
  \label{eqn:non-heuristic-release}
  T \defeq u^T \theta(P_n) +
  \normal\left(0, \varpriv \cdot \stdonedim^2\right)
\end{equation}
and $T = \perp$ otherwise.  The following result, based on the test-release
framework (Algorithm~\ref{alg:test-release}), guarantees privacy.

\begin{proposition}
  \label{proposition:release-once-ratio}
  Let $\diffp \ge \half (1 + \Phi^{-1}(1 - \delta/2)^2) \ratio^2(u,
  \what{\lambda})$ and $\delta > 0$.  Then $T$ is $(3 \diffp, (1 + e^\diffp
  + e^{2 \diffp}) \delta)$-differentially private.
\end{proposition}
\noindent
Because the proposition is more or less a consequence of the
propose-test-release scheme, we prove it in
Appendix~\ref{sec:proof-release-once-ratio}.

We provide two main results that allow us to apply
Proposition~\ref{proposition:release-once-ratio}. First, we address the case
in which the gradient set is a scaled $\ell_2$-ball, and in the second, when
it is a scaled $\ell_\infty$-ball.  In either case, we must specify several
constants to allow (private) certification that the
ratio~\eqref{eqn:ratio-to-control} is bounded.
Assume that the loss $\loss_\theta(x, y) = h(\<\theta, x\>, y)$ where $h$
satisfies the self-concordance guarantees~\eqref{eqn:bound-sc-by-linear}
with self-bounding parameter $\selftwoconst$ satisfying
$\concordantfunc(t) \le \selftwoconst t$.
For covariate domain $\mc{X}$, let $r =
\radius_2(\mc{X})$ be the $\ell_2$-radius of the data, and recall the
definition~\eqref{eqn:t-param-change} of $\tparamchange(\lambda)$, which
guarantees that $\ltwo{\theta(P_n) - \theta(P_n')} \le
\tparamchange(\lambdamin(P_n) + \lambdareg)$ under
Condition~\eqref{eqn:lambda-self-bounded-big-enough}.
Recall additionally the recursion
$\recurse$ defined by the cases~\eqref{eqn:self-bounded-lambda-recursion}.

Now, we control the ratio~\eqref{eqn:ratio-to-control}.  Fix $\lambda_0$ and
$\lambda_1$ to be any positive values (in the sequel, we take them to
estimate the minimal and maximal eigenvalues $\lambdamin(P_n) + \lambdareg$ and
$\lambdamax(P_n) + \lambdareg$).  Recall the
definitions~\eqref{eqn:all-the-ratio-constants} of the constants
\begin{equation*}
  \begin{split}
    \tparamchange & \defeq \tparamchange(\lambda_0),
    ~~~
    r \defeq \radius_2(\mc{X}),
    ~~~ \annoyingconst
    \defeq \frac{\linf{h''}}{\hinge{1 - \selftwoconst \tparamchange}}
    \frac{r^2}{n \lambda_0},
    ~~~
    \gamma \defeq \selftwoconst r \cdot \tparamchange,
    ~~~
    \gamma' \defeq \selftwoconst r \cdot
    \tparamchange(\recurse(\lambda_0)) \\
    \simconst_1 & \defeq \frac{1}{\hinge{1 - \selftwoconst
        r \tparamchange}} - 1,
    ~~~
    \simconst_2 \defeq \frac{1}{n (1 - \annoyingconst)}
    \frac{\linf{h''}}{\hinge{1 - \selftwoconst r \tparamchange}},
    ~~~
    \kappa \defeq \frac{\lambda_1}{\lambda_0}.
  \end{split}
\end{equation*}
We then have the following guarantees.

\begin{proposition}
  \label{proposition:certifiable-ratio-bound}
  Let the preceding conditions hold,
  assume that
  $\lambda_0 \le \lambdamin(P_n) + \lambdareg$ and
  $\lambdamax(P_n) + \lambdareg \le \lambda_1$, that
  Condition~\eqref{eqn:lambda-self-bounded-big-enough} holds,
  and define $\kappa = \frac{\lambda_1}{\lambda_0}$.
  Let the
  gradient set
  \begin{equation*}
    \mc{G} = \left\{g \in \R^d \mid \ltwo{g} \le \lipobj\right\}.
  \end{equation*}
  Then
  \begin{equation*}
    1 \le
    \frac{\stdonedim}{\diffu(P_n, u)}
    \le 1 + \kappa \frac{\gamma}{1 - \gamma}
  \end{equation*}
  and
  \begin{equation*}
    1 - \kappa (\simconst_1 + \simconst_2 \radconst)
    \le \frac{\stdonedim[P_n']}{\diffu(P_n, u)}
    \le 
    1 + \kappa(\simconst_1 + \simconst_2 \radconst)
    + \kappa \frac{\lambda_0}{\recurse(\lambda_0)} \frac{\gamma'}{1 - \gamma'}.
  \end{equation*}
\end{proposition}

\begin{proposition}
  \label{proposition:certifiable-linf-ratio}
  Let the conditions of
  Proposition~\ref{proposition:certifiable-ratio-bound} hold, but define $d_p =
  d^{1 - 2/p}$ and let the gradient set
  \begin{equation*}
    \mc{G} = \left\{g \in \R^d \mid \norm{g}_p \le \lipobj\right\}.
  \end{equation*}
  Then
  \begin{equation*}
    1 \le
    \frac{\stdonedim}{\diffu(P_n, u)}
    \le 1 + \sqrt{d_p} \kappa \frac{\gamma}{1 - \gamma}
  \end{equation*}
  and
  \begin{equation*}
    1 - \sqrt{d_p} \simconst_1 \kappa - 2 \frac{2 d_p \simconst_2}{\lambda_0}
    \le \frac{\stdonedim[P_n']}{\diffu(P_n, u)}
    \le 
    1 + \sqrt{d_p} \kappa \simconst_1
    + \frac{2 \simconst_2 d_p}{\lambda_0}
    + \frac{\sqrt{d_p} \kappa \lambda_0}{\recurse(\lambda_0)}
    \frac{\gamma'}{1 - \gamma'}.
  \end{equation*}
\end{proposition}

Propositions~\ref{proposition:certifiable-ratio-bound}
and~\ref{proposition:certifiable-linf-ratio} rely on fairly careful
control over non-symmetric quadratic forms.  After
giving some commentary on the results, we build up to them over the
remainder of this section, beginning in
Section~\ref{sec:self-similar-matrices}, which addresses the similarity of
Hessians for neighboring samples $P_n$ and $P_n'$, with the proofs
of the propositions following in
Sections~\ref{sec:proof-certifiable-ratio-bound}
and~\ref{sec:proof-certifiable-linf-ratio}.


\subsubsection{Proof of Corollary~\ref{corollary:u-t-theta-private}}
\label{sec:proof-u-t-theta-private}

To obtain privacy using these results, we apply
Proposition~\ref{proposition:release-once-ratio}. We need to guarantee that
the ratio of the (local) moduli of continuity satisfy the appropriate bounds
on the ratio $\ratio(u, \lambda)$ in
inequality~\eqref{eqn:ratio-to-control}.
The inequalities~\eqref{eqn:ratio-for-epsilon}
guarantee that
\begin{equation*}
  \frac{1}{1 + \ratio^2}
  \le \left(\frac{\stdonedim}{\stdonedim[P_n']}\right)^2
  \le 1 + \ratio^2
  ~~ \mbox{for some}~~
  \ratio^2 \le \frac{2 \diffp}{1 + \Phi^{-1}(1 - \delta/2)^2}.
\end{equation*}
Then Corollary~\ref{corollary:u-t-theta-private}
follows as an immediate corollary to
Propositions~\ref{proposition:certifiable-ratio-bound}
and~\ref{proposition:certifiable-linf-ratio},
coupled with Proposition~\ref{proposition:release-once-ratio}.


\subsection{Self-similar matrices and Hessians}
\label{sec:self-similar-matrices}

Proposition~\ref{proposition:release-once-ratio} makes it clear that
what is essential to releasing a statistic accurately
is to provide sufficient bounds on the ratio~\eqref{eqn:ratio-to-control}.
This turns out to be a fairly subtle question, and
we develop a few tools to bound ratios of matrix-vector products
here to address the issue.
We provide proofs of the results in Section~\ref{sec:proofs-ratio-stuff},
which require a few auxiliary results as well.
Abstractly---treating $H_0$ and $H_1$ as
the Hessians $P_n\ddot{\loss}_{\theta(P_n)}$ and
$P_n'\ddot{\loss}_{\theta(P_n')}$, respectively---we will
control ratios of
\begin{equation*}
  \sup_{v \in \mc{V}} u^T H_0^{-1} v
  ~~ \mbox{to} ~~
  \sup_{v \in \mc{V}} u^T H_1^{-1} v,
\end{equation*}
where $\mc{V}$ is a symmetric convex body.
In evaluating the ratios of these quantities,
we consider matrices $H_0$ and $H_1$ that
we term \emph{$(\simconst_1, \simconst_2)$-self-similar relative
to $\mc{X}$}, meaning that
$H_0$ and $H_1$ satisfy
\begin{equation}
  \label{eqn:hessian-error-relation}
  H_1^{-1} = H_0^{-1} + E_1 + E_2,
\end{equation}
where the error matrices $E_1$ and $E_2$ satisfy
that there exist vectors $x_0, x_1 \in \mc{X}$ such that
\begin{equation*}
  -\simconst_1 H_0^{-1} \preceq E_1 \preceq \simconst_1 H_0^{-1}
  ~~ \mbox{and} ~~
  -\simconst_2 H_0^{-1} x_0 x_0^T H_0^{-1}
  \preceq E_2 \preceq
  \simconst_2 H_0^{-1} x_1 x_1^T H_0^{-1}.
\end{equation*}
The key to applying these similarity results
is that the Hessians $H_0 = P_n \ddot{\loss}_{\theta(P_n)}
+ \lambdareg I$
and $H_1 = P_n' \ddot{\loss}_{\theta(P_n')} + \lambdareg I$ are self-similar.
\begin{lemma}
  \label{lemma:self-inverse-hessian-expansion}
  Assume that $\ltwo{\theta - \theta'} \le t$ and
  $\radius(\mc{X}) \le r$. Define
  $\annoyingconst = \frac{\linf{h''}}{1 - \selftwoconst r t}
  \frac{1}{\lambdamin(P_n\ddot{\loss}_\theta) + \lambdareg} \frac{r^2}{n}$.
  If $\annoyingconst < 1$, there
  there exist $x_0, x_1 \in \mc{X}$ such that
  \begin{align*}
    \lefteqn{\frac{1}{1 + \selftwoconst r t}
      H_0^{-1}
      - \frac{1}{n(1 - \annoyingconst)}
      \frac{\linf{h''}}{(1 + \selftwoconst r t)^2}
      H_0^{-1}
      x_0 x_0^T 
      H_0^{-1}}
    \\
    &
    \qquad \qquad \preceq
    H_1^{-1}
    \preceq
    \frac{1}{1 - \selftwoconst r t}
    H_0^{-1}
    + \frac{1}{n(1 - \annoyingconst)}
    \frac{\linf{h''}}{(1 - \selftwoconst r t)^2}
    H_0^{-1} x_1 x_1^T H_0^{-1}.
  \end{align*}
\end{lemma}
\noindent
See Section~\ref{sec:proof-self-inverse-hessian-expansion} for a proof of
Lemma~\ref{lemma:self-inverse-hessian-expansion}.
Rewriting the result in a more modular form,
Lemma~\ref{lemma:self-inverse-hessian-expansion} shows the following:
\begin{lemma}
  \label{lemma:simple-self-similar-hessian}
  Let $H_0 = P_n \ddot{\loss}_{\theta(P_n)} + \lambdareg I$ and
  $H_1 = P_n' \ddot{\loss}_{\theta(P_n')} + \lambdareg I$. Assume
  the bounds of Lemma~\ref{lemma:self-inverse-hessian-expansion}
  that $\ltwo{\theta(P_n) - \theta(P_n')} \le t$.
  Then $H_0$ and $H_1$ are $(\simconst_1, \simconst_2)$-self-similar
  with  
  \begin{align*}
    \simconst_1
    = \frac{1}{1 - \selftwoconst r t} - 1
    ~~ \mbox{and} ~~
    \simconst_2
    = \frac{1}{n(1 - \annoyingconst)}
    \frac{\linf{h''}}{(1 - \selftwoconst r t)^2}.
  \end{align*}
\end{lemma}

We specialize these results to bound the ratios of $\stdonedim /
\stdonedim[P_n']$ when the $\mc{V}$ is a norm ball.  We begin by capturing
the case in which $\mc{V}$ is an $\ell_2$ ball, so that $\sup_{v \in \mc{V}}
u^T H_0^{-1} v = \ltwo{H_0^{-1} u}$. (Scalings of the $\ell_2$-ball follow
trivially.)
\begin{lemma}
  \label{lemma:self-similar}
  Let $H_0$ and $H_1$ be $(\simconst_1,
  \simconst_2)$-self-similar~\eqref{eqn:hessian-error-relation}
  relative to $\mc{X}$.  Then
  \begin{equation*}
    \ltwo{H_1^{-1} u}
    \le \left(1 + \frac{\simconst_1}{2}
    \left(1 + \frac{\lambdamax(H_0)}{\lambdamin(H_0)}\right)
    + \simconst_2 \sup_{x \in \mc{X}} \lambdamax(H_0) \ltwo{H_0^{-1} x}^2\right)
    \ltwo{H_0^{-1} u}.
  \end{equation*}
  Similarly,
  \begin{equation*}
    \ltwo{H_1^{-1} u}
    \ge \left(1 - \frac{\simconst_1}{2}\left(
    1 + \frac{\lambdamax(H_0)}{\lambdamin(H_0)}\right)
    - \simconst_2 \sup_{x \in \mc{X}} \lambdamax(H_0) \ltwo{H_0^{-1} x}^2\right)
    \ltwo{H_0^{-1} u}.
  \end{equation*}
\end{lemma}
\noindent
See Section~\ref{sec:proof-self-similar} for the proof of the lemma.

We can also consider more general sets.
Let the set $\mc{V}$ be a symmetric convex body as before. Define
the maximal $\ell_p$-inscribed- and $\ell_p$-radii by
\begin{equation*}
  \radius_p(\mc{V}) \defeq \sup\{\norm{v}_p \mid v \in \mc{V}\}
  ~~ \mbox{and} ~~
  \inscribed_p(\mc{V}) \defeq \sup\{t \mid t \ball_p^d
  \subset \mc{V} \},
\end{equation*}
so that the ratio $\radius_p(\mc{V}) / \inscribed_p(\mc{V})$ gives
a type of condition number for $\mc{V}$. For example,
$\mc{V} = [-1, 1]^d$ has $\radius_2(\mc{V}) = \sqrt{d}$ and
$\inscribed_2(\mc{V}) = 1$.
For a matrix $A$ we define the
$\mc{V}$-relative conditioning
\begin{equation*}
  \kappa(A, \mc{V}) \defeq
  \frac{\sup_{v \in \mc{V}} v^T A v / \ltwo{v}}{
    \inf_{\ltwo{u} = 1} \sup_{v \in \mc{V}} u^T A v}.
\end{equation*}
A quick calculation gives the following bound on
this condition number.
\begin{lemma}
  \label{lemma:relative-conditioning}
  The $\mc{V}$-relative condition number satisfies
  \begin{equation*}
    \kappa(A, \mc{V}) \le \frac{\radius_2(\mc{V})}{\inscribed_2(\mc{V})}
    \kappa(A).
  \end{equation*}
  Additionally, for $\mc{V} = \ball_2^d$,
  $\kappa(A, \mc{V}) = \kappa(A)$, and for
  $\mc{V} = [-1, 1]^d$, for any condition
  number $\kappa \ge 1$ there exist matrices $A$ with
  $\kappa(A) = \kappa$ and
  \begin{equation*}
    \frac{1}{\sqrt{2}}
    \cdot \frac{\radius_2(\mc{V})}{\inscribed_2(\mc{V})}
    \kappa(A)
    \le \kappa(A, \mc{V})
    \le \frac{\radius_2(\mc{V})}{\inscribed_2(\mc{V})}
    \kappa(A).
  \end{equation*}
\end{lemma}
\noindent
See Section~\ref{sec:proof-relative-conditioning} for a proof.
With this lemma, we can provide an analogue of Lemma~\ref{lemma:self-similar}
when $\mc{V}$ is not an $\ell_2$-ball.
\begin{lemma}
  \label{lemma:self-similar-general}
  Let $H_0$ and $H_1$ be $(\simconst_1,
  \simconst_2)$-self-similar~\eqref{eqn:hessian-error-relation}, and let $u$
  be a unit vector.  Let $\mc{V} = c \mc{X} = \{cx \mid x \in \mc{X}\}$,
  where $c > 0$ is a fixed constant.  Then for any unit vector $u$,
  \begin{equation*}
    \sup_{v \in \mc{V}}
    u^T H_1^{-1} v
    \le \left(1 + \simconst_1 \frac{\radius_2(\mc{V})}{\inscribed_2(\mc{V})}
    \kappa(H_0)
    + \simconst_2 \frac{2 \radius_2^2(\mc{V})}{c^2 \lambdamin(H_0)}
    \right) \sup_{v \in \mc{V}} u^T H_0^{-1} v
  \end{equation*}
  and
  \begin{equation*}
    \sup_{v \in \mc{V}}
    u^T H_1^{-1} v
    \ge \left(1 - \simconst_1 \frac{\radius_2(\mc{V})}{\inscribed_2(\mc{V})}
    \kappa(H_0)
    - \simconst_2 \frac{2 \radius_2^2(\mc{V})}{c^2 \lambdamin(H_0)}
    \right) \sup_{v \in \mc{V}} u^T H_0^{-1} v.
  \end{equation*}
  In the case
  that $\mc{V} = [-1, 1]^d$ and $H_0 = I$, these results are sharp
  for any standard basis vector $u \in \{e_1, \ldots, e_d\}$, in that
  there exists $H_1$ satisfying self-similarity and
  \begin{equation*}
    \sup_{v \in \mc{V}}
    u^T H_1^{-1} v
    \ge \left(1 + \simconst_1 \sqrt{d} + \simconst_2 d \right)
    = \left(1 + \simconst_1 \frac{\radius_2(\mc{V})}{\inscribed_2(\mc{V})}
    \kappa(H_0)
    + \simconst_2 \frac{\radius_2^2(\mc{V})}{\lambdamin(H_0)}
    \right) \sup_{v \in \mc{V}} u^T H_0^{-1} v.
  \end{equation*}
\end{lemma}
\noindent
See Section~\ref{sec:proof-self-similar-general}
for the proof.

\subsection{Proof of Proposition~\ref{proposition:certifiable-ratio-bound}}
\label{sec:proof-certifiable-ratio-bound}

Observe first that, by Lemma~\ref{lemma:simple-self-similar-hessian}, if we
have any quantity ${\lambda}_0 \le \lambdamin(P_n) + \lambdareg$ then the
Hessians $H_0 = P_n \ddot{\loss}_{\theta(P_n)} + \lambdareg I$ and $H_1 = P_n'
\ddot{\loss}_{\theta(P_n')} + \lambdareg I$ are $({\simconst}_1,
     {\simconst}_2)$-self-similar, where for ${\tparamchange} =
     \tparamchange({\lambda}_0)$, we recall our parameter definitions
\begin{equation*}
  {\simconst}_1 \defeq \frac{1}{\hinge{1 - \selftwoconst r
    {t}}} - 1,
  ~~
  \simconst_2 \defeq \frac{1}{n(1 - {\annoyingconst})}
  \frac{\linf{h''}}{\hinge{1 - \selftwoconst r {\tparamchange}}^2},
  ~~ \mbox{and} ~~
  {\annoyingconst}
  = \frac{\linf{h''}}{\hinge{1 - \selftwoconst r {\tparamchange}}}
  \frac{r^2}{n {\lambda}_0},
\end{equation*}
Recalling the definition $\gamma(P_n) = \selftwoconst
\tparamchange(\lambdamin(P_n) + \lambdareg) \radius(\mc{X})$, because
$\tparamchange(\lambda)$ is decreasing in $\lambda$, once we have the lower
bound ${\lambda}_0 \le \lambdamin(P_n) + \lambdareg$ we obtain an upper bound
$\gamma = \selftwoconst \tparamchange \radconst \ge \gamma(P_n)$.

We first provide an upper bound on $\stdonedim
= \diffu(P_n, u) + \frac{2 \lipobj}{n \lambda_0} \frac{\gamma}{1 - \gamma}$ and
a lower bound on $\stdonedim[P_n'] \ge \diffu(P_n', u)$.
We may use the bounds from Lemmas~\ref{lemma:self-similar}
and~\ref{lemma:self-similar-general} to control $\diffu(P_n', u)$.  So long
as ${\lambda}_1 \ge \lambdamax(P_n) + \lambdareg$, the condition number estimate
${\kappa} = \frac{{\lambda}_1}{{\lambda}_0} \ge \kappa(H_0)$.
Recalling the gradient set $\mc{G} = \{g \in \R^d \mid \ltwo{g} \le
\lipobj\}$, Lemma~\ref{lemma:self-similar} coupled with the recognition
that
\begin{equation*}
  \diffu(P_n, u) = \frac{2}{n} \sup_{g \in \mc{G}_2} u^T H_0^{-1} g
  = \frac{2 \lipobj}{n} \ltwo{H_0^{-1} u}
\end{equation*}
implies
\begin{equation*}
  \diffu(P_n', u)
  \ge \left(1 - {\simconst}_1
  {\kappa}
  - {\simconst}_2 r {\kappa} \right)
  \diffu(P_n, u).
\end{equation*}
For the upper bound on $\stdonedim$, note that
$\ltwo{H_0^{-1} u} \ge {\lambda}_1^{-1}$ for any unit vector $u$,
so that
$\frac{2 \lipobj}{n \lambda_0} \le
\kappa \diffu(P_n, u)$.

To obtain the lower bound on $\stdonedim$ and
upper bound on $\stdonedim[P_n']$, note that
$\gamma(P_n') = \selftwoconst \tparamchange(\lambdamin(P_n'))
\radius(\mc{X}) \le \selftwoconst \tparamchange(\recurse({\lambda}_0))
r \eqdef {\gamma}'$. So proceeding as above, we obtain
trivially that $\stdonedim \ge \diffu(P_n, u)$, and
\begin{equation*}
  \stdonedim[P_n']
  \le \diffu(P_n', u)
  + \frac{2 \lipobj}{n \recurse({\lambda}_0)}
    \frac{{\gamma}'}{1 - {\gamma}'}
  \le \diffu(P_n, u)
  (1 + {\kappa}({\simconst}_1 + {\simconst}_2 r))
  + \frac{2 \lipobj}{n \recurse({\lambda}_0)}
  \frac{{\gamma}'}{1 - {\gamma}'}
\end{equation*}
by Lemma~\ref{lemma:self-similar} again.
Now again use that
$\frac{2 \lipobj}{n \lambda_0}
\le \kappa \diffu(P_n, u)$.

\subsection{Proof of Proposition~\ref{proposition:certifiable-linf-ratio}}
\label{sec:proof-certifiable-linf-ratio}

In the case of $\ell_p$, $p > 2$-bounded gradients, we must control the error
terms somewhat differently than we did in the proof of
Proposition~\ref{proposition:certifiable-ratio-bound}. We still have that
$H_0 = P_n \ddot{\loss}_{\theta(P_n)} + \lambdareg I$ and $H_1 = P_n'
\ddot{\loss}_{\theta(P_n')} + \lambdareg I$ are ${\simconst}_1,
{\simconst}_2$-similar, with the same definitions of the constants as
in Proposition~\ref{proposition:certifiable-ratio-bound}.

For $\mc{V} = \{v \in \R^d \mid \norm{v}_p \le 1\}$, we have relative
condition bound $\kappa(A, \mc{V}) \le \sqrt{d_p} \kappa(A)$, which is
(nearly) as sharp as possible by Lemma~\ref{lemma:relative-conditioning}.
We therefore have via Lemma~\ref{lemma:self-similar-general} that
\begin{equation*}
  \diffu(P_n', u)
  = \frac{2}{n} \sup_{g \in \mc{G}_p} u^T H_1^{-1} g
  \ge \left(1 - \simconst_1 \sqrt{d_p} \kappa(H_0)
  - \simconst_2 \frac{2 d_p}{\lambdamin(H_0)}\right)
  \diffu(P_n, u).
\end{equation*}
Using that $\diffu(P_n, u) = \frac{2 \lipobj}{\sqrt{d_p} n}
\norm{H_0^{-1} u}_q
\ge \frac{2 \lipobj}{\sqrt{d_p} n \lambda_1}$,
where $q = \frac{p}{p - 1} < 2$ is conjugate to $p$,
we rearrange to obtain
$\frac{2 \lipobj}{n \lambda_0}
\le \sqrt{d_p} \kappa \diffu(P_n, u)$, so that
\begin{equation*}
  \stdonedim
  \le \diffu(P_n, u)\left(1 + \sqrt{d_p} \kappa \frac{\gamma}{1 - \gamma}
  \right).
\end{equation*}

To obtain the converse bounds,
note that applying Lemma~\ref{lemma:self-similar-general}
gives
\begin{equation*}
  \diffu(P_n', u)
  \le
  \left(1 + \simconst_1 \sqrt{d_p} \kappa(H_0)
  + \simconst_2 \frac{2 d_p}{\lambdamin(H_0)}\right)
  \diffu(P_n, u).
\end{equation*}
and similar calculations thus yield
\begin{align*}
  \stdonedim[P_n']
  & \le \diffu(P_n', u)
  + \frac{2 \lipobj}{n \recurse(\lambda_0)} \frac{\gamma'}{1 - \gamma'}
  \\
  & \le
  \diffu(P_n, u) \left(1 + \sqrt{d_p} \simconst_1 \kappa
  + \frac{2 \simconst_2 d_p}{\lambda_0}\right)
  + \frac{\sqrt{d_p} \kappa \lambda_0}{\recurse(\lambda_0)}
  \frac{\gamma'}{1 - \gamma'} \diffu(P_n, u)
\end{align*}
where we used again
that $\frac{2 \lipobj}{n \lambda_0}
\le \sqrt{d_p} \kappa \diffu(P_n, u)$.

\subsection{Proof of Lemma~\ref{lemma:self-inverse-hessian-expansion}}
\label{sec:proof-self-inverse-hessian-expansion}

Recalling our notation
that $\ltwo{\theta - \theta'} \le t$ and
$\radius(\mc{X}) \le r$, we have
\begin{align*}
  P_n' \ddot{\loss}_{\theta'}
  & = P_n \ddot{\loss}_{\theta'}
  + (P_n' - P_n) \ddot{\loss}_{\theta'}
  = P_n \ddot{\loss}_\theta
  + P_n(\ddot{\loss}_{\theta'} - \ddot{\loss}_\theta)
  + (P_n' - P_n) \ddot{\loss}_{\theta'}.
\end{align*}
Let $E_1 = P_n \ddot{\loss}_{\theta'} - P_n \ddot{\loss}_\theta$, so that
$-\selftwoconst r t P_n \ddot{\loss}_\theta \preceq E_1 \preceq
\selftwoconst r t P_n\ddot{\loss}_\theta$. Let $E_2 = (P_n' - P_n)
\ddot{\loss}_{\theta'}$, so leveraging that
$\ddot{\loss}_\theta = h''(\<\theta, x\>, y) xx^T$,
there exist $x_0, x_1, y_0, y_1$ such that
\begin{equation*}
  n (P_n' - P_n) \ddot{\loss}_{\theta'}
  = n E_2
  = \ddot{\loss}_{\theta'}(x_0, y_0)
  - \ddot{\loss}_{\theta'}(x_1, y_1)
  = h''(\<\theta', x_0\>, y_0) x_0 x_0^T
  - h''(\<\theta', x_1\>, y_1) x_1 x_1^T,
\end{equation*}
that is, there exist $x_0, x_1$ and 
$\linf{h''} < \infty$ such that
$-x_1 x_1^T \linf{h''} \preceq n  E_2 \preceq x_0 x_0^T \linf{h''}$.
Define
\begin{equation*}
  H_0 = P_n \ddot{\loss}_\theta + \lambdareg I,
  ~~ \mbox{and} ~~
  H_1 = P_n' \ddot{\loss}_{\theta'} + \lambdareg I.
\end{equation*}
Then using the operator monotonicity properties of the matrix inverse,
we obtain that for some $x \in \mc{X}$,
\begin{align*}
  H_1^{-1}
  & \preceq \left(P_n \ddot{\loss}_\theta (1 - \selftwoconst \radconst
  \tparamchange)
  + (1 - \selftwoconst \radconst \tparamchange) \lambdareg I
  - n^{-1} \linf{h''} x x^T \right)^{-1} \\
  & = \frac{1}{1 - \selftwoconst r t}
  H_0^{-1}
  + \frac{\linf{h''}}{
    n (1 - \selftwoconst r t)^2
    (1 - \frac{\linf{h''}}{(1 - \selftwoconst r t)n}
    x^T H_0^{-1} x)}
  H_0^{-1} x x^T H_0^{-1}
\end{align*}
by the Sherman-Morrison inversion formula.
Similarly,
there exists $x \in \mc{X}$ such that
\begin{align*}
  H_1^{-1}
  & \succeq
  \left(
  P_n \ddot{\loss}_\theta (1 + \selftwoconst r t)
  + (1 + \selftwoconst \radconst \tparamchange) \lambdareg I
  + n^{-1} \linf{h''} xx^T \right)^{-1} \\
  & = \frac{1}{1 + \selftwoconst \radconst \tparamchange}
  H_0^{-1}
  - \frac{\linf{h''}}{n
    (1 + \selftwoconst r t)^2
    (1 + \frac{\linf{h''}}{n (1 + \selftwoconst r t)}
    x^T H_0^{-1} x)}
  H_0^{-1} xx^T H_0^{-1}.
\end{align*}
Defining $\annoyingconst = \frac{\linf{h''}}{1 - \selftwoconst r t}
\frac{r^2}{(\lambdamin(P_n \ddot{\loss}_\theta) + \lambdareg) n}$, we thus obtain
\begin{align*}
  \lefteqn{\frac{1}{1 + \selftwoconst r t}
    H_0^{-1}
    - \frac{\linf{h''}}{n (1 + \selftwoconst t r)^2
      (1 - \annoyingconst)}
    H_0^{-1} xx^T H_0^{-1}} \\
  & \qquad \preceq
  H_1^{-1}
  \preceq
  \frac{1}{1 - \selftwoconst r t}
  H_0^{-1}
  + \frac{\linf{h''}}{n (1 - \selftwoconst t r)^2
    (1 - \annoyingconst)}
  H_0^{-1} xx^T H_0^{-1}
\end{align*}
as desired.


\section{Discussion}

The original motivation for this paper was in service to a hypothesis that
we entertain, which is that to improve adoption of privacy-preserving
procedures in sciences will require effective and practical methods.
We admit that we have, perhaps, strayed from a simple set of procedures
via detours through some nontrivial mathematical machinery,
which still leaves us unable to easily test our hypothesis.
In spite of this, our experimental results are promising: for large sample
sizes, Algorithm~\ref{alg:release-u-t-theta} nearly achieves optimal
performance, to within small numerical constant factors.

Nonetheless, there are several avenues for future work, which
we hope that others will tackles.
First, as we discuss in Section~\ref{sec:dimension-dependence}, even for
well-conditioned problems, it appears that
Algorithm~\ref{alg:release-u-t-theta} becomes most effective when $n \gtrsim
d^{3/2}$, at which point it more or less
releases $u^T \theta(P_n) + \normal(0, \diffu^2(P_n, u))$,
which is optimal scaling.
Identifying the precise dimension dependence at which this ``local
modulus''-dependent release is possible will be interesting.
One plausible avenue here would be to develop procedures that rely not on
the minimal eigenvalue $\lambdamin(P_n)$, which governs the worst-case gross
behavior of $\ltwos{\theta(P_n) - \theta(P_n')}$ but instead on a more
nuanced quantity relating directly to the differences $u^T \theta(P_n) - u^T
\theta(P_n')$.
Corollaries~\ref{corollary:directional-modulus-bound}
and~\ref{corollary:eigenvalue-qsc-change} both rely on this global bound in
the change of $\theta(P_n)$ rather than the particular directionality that
$\theta(P_n') \approx \theta(P_n) - (P_n\ddot{\loss}_\theta)^{-1} (P_n -
P_n') \dot{\loss}_\theta$, so that more careful tracking there could allow
better dimension-dependence.
Of course, our approaches here may be simply mis-directed, and
a more direct attempt to implement Asi and Duchi's~\cite{AsiDu20}
inverse sensitivity, which is instance optimal, may be more
sensible.
Regardless, we hope that continued interest in practicable procedures for
private estimation continues.

\appendix

\section{Technical appendices}


\subsection{Proofs of basic privacy building blocks}

In this appendix, we collect the proofs
of the privacy building blocks in Section~\ref{sec:composition-test-release}.

\subsubsection{Proof of Lemma~\ref{lemma:conditional-composition}}
\label{sec:proof-conditional-composition}

We wish to show that for any $A \subset \mc{T} \times \mc{W}$
and neighboring sample $P_n'$,
we have
\begin{equation}
  \label{eqn:super-conditional-composition}
  \P((M(P_n, W), W) \in A)
  \le e^{\diffp + \diffp_0}
  \P((M(P_n', W'), W') \in A) + \delta_0 + \delta + \gamma,
\end{equation}
where $W' \sim \mu(\cdot \mid P_n')$ is the mechanism
$W$ on input $P_n'$.
Define the slices
$A_w = \{t \mid (t, w) \in A\}$ and
projection $A^{\mc{W}} = \{w \mid
\mbox{there~exist}~ (t, w) \in A\}$, which are
measurable as $A$ is~\cite[Ch.~12.4]{Royden88}.
By standard conditional probability and
(dis)integration arguments~\cite{ChangPo97},
we have
\begin{align*}
  \P(M(P_n, W) \in A)
  & = \int_{A^{\mc{W}}}
  \P(M(P_n, w) \in A_w) d\mu(w \mid P_n) \\
  & \stackrel{(i)}{\le} \int_{G(P_n) \cap A^{\mc{W}}}
  \P(M(P_n, w) \in A_w) d\mu(w \mid P_n)
  + \P(W \not \in G(P_n) \mid P_n) \\
  & \stackrel{(ii)}{\le} \int_{G(P_n) \cap A^{\mc{W}}}
  \min\left\{(e^\diffp \P(M(P_n', w) \in A_w)
  + \delta), 1 \right\} d\mu(w \mid P_n) + \gamma \\
  & \stackrel{(iii)}{\le}
  \int_{A^{\mc{W}}}
  \min\left\{e^\diffp \P(M(P_n', w) \in A_w), 1 \right\}
  d\mu(w \mid P_n) + \delta + \gamma,
\end{align*}
where inequality~$(i)$ follows because $\int f d\mu \le 1$ whenever $0 \le
f \le 1$, inequality~$(ii)$ by the assumptions that $M(P_n, w)$ is
$(\diffp, \delta)$-differentially private when $w \in G(P_n)$ and that
$\P(W \not\in G(P_n) \mid P_n) \le \gamma$, and inequality~$(iii)$ follows
because $\min\{a + b, 1\} \le \min\{a, 1\} + b$ whenever $a, b \ge
0$.

For shorthand define the (measurable)
function $f(w) \defeq \min\{e^\diffp \P(M(P_n', w) \in
A_w), 1\}$, noting that $0 \le f \le 1$. Then by the
definition of the integral $\int f d\mu$ as a supremum over simple
functions $0 \le \varphi \le f$ (e.g.~\cite[Ch.~11.3]{Royden88}),
we obtain
$\int f(w) d\mu(w \mid P_n)
\le e^{\diffp_0} \int f(w) d\mu(w \mid P_n') + \delta_0$
by the assumption that $W$ is $(\diffp_0, \delta_0)$-DP. Substituting
above gives
\begin{align*}
  \P(M(P_n, W) \in A)
  & \le e^{\diffp_0} \int_{A^{\mc{W}}} \min\{e^\diffp \P(M(P_n', w) \in A_w),
  1 \} d\mu(w \mid P_n') + \gamma + \delta_0 + \delta \\
  & \le e^{\diffp_0 + \diffp}
  \int_{A^{\mc{W}}} \P(M(P_n', w) \in A_w) d\mu(w \mid P_n')
  + \gamma + \delta_0 + \delta \\
  & = e^{\diffp_0 + \diffp} \P((M(P_n', W'), W') \in A)
  + \gamma + \delta_0 + \delta,
\end{align*}
which is inequality~\eqref{eqn:super-conditional-composition}.

\subsubsection{Proof of Lemma~\ref{lemma:private-normal-variance}}
\label{sec:proof-private-normal-variance}

Without loss of generality by translation, we assume
$\mu_0 = 0$ and let $\mu = \mu_1$. Let
$p_i(z) = \frac{1}{2 \sigma^2} \exp(-\frac{1}{2 \sigma^2}
\ltwo{z - \mu_i}^2)$ for
$i = 0, 1$, and 
define the
log likelihood ratio
$\ell(z) = \log \frac{p_0(z)}{p_1(z)}
= \frac{1}{2 \sigma^2}
(\ltwo{\mu}^2 + 2 \<\mu, z\>)$. Then
we have $Z_0 \eqdiffp Z_1$ if 
$\P(|\ell(Z)| \ge \diffp) \le \delta$ when
$Z \sim \normal(0, \sigma^2 I)$. 
A bit of linear algebra and the rotational invariance
of the Gaussian distribution shows that
if $W \sim \normal(0, 1)$, then
\begin{align*}
  \P(|\ell(Z)| \ge \diffp)
  & = \P\left(\left|\frac{\ltwo{\mu}^2}{2 \sigma^2}
  + \frac{\ltwo{\mu}}{\sigma} W \right| \ge \diffp\right) \\
  & = \P\left(W \ge \frac{\sigma}{\ltwo{\mu}}
  \left(\diffp - \frac{\ltwo{\mu}^2}{2 \sigma^2}\right)
  \right)
  + \P\left(W \le - \frac{\sigma}{\ltwo{\mu}}
  \left(\diffp + \frac{\ltwo{\mu}^2}{2 \sigma^2}\right)
  \right).
\end{align*}
The homogeneity of $\sigma / \ltwo{\mu}$ gives the result.

\subsubsection{Proof of Lemma~\ref{lemma:test-release}}
\label{sec:proof-test-release}

For any set $B$ not including $\perp$, we have
that $\{M(P_n) \in B\} = \{M_0(P_n) \in A, M_1(P_n) \in B\}$.
Consider two cases: in the first,
we have $\lambda(P_n) \in G$. Then
$M_1(P_n) \eqdiffp M_1(P_n')$, and so
standard $(\diffp, \delta)$-composition gives
\begin{align*}
  \P(M(P_n) \in B)
  & = \P(M_0(P_n) \in A, M_1(P_n) \in B) \\
  & \le e^{\diffp_0 + \diffp} \P(M_0(P_n') \in A, M_1(P_n') \in B)
  + \delta_0 + \delta.
\end{align*}
In the second, $\lambda(P_n) \not \in G$. Then
$\P(M_0(P_n) \in A) \le \delta_0$, and by
$(\diffp_0, \delta_0)$-differential privacy, we have
$\P(M_0(P_n') \in A) \le e^{\diffp_0} \delta_0$, so that
\begin{align*}
  \P(M(P_n) \in B) & =
  \P(M_0(P_n) \in A, M_1(P_n) \in B)
  \le \P(M_0(P_n) \in A)
  \le \delta_0 \\
  \P(M(P_n') \in B)
  & \le \P(M_0(P_n') \in A)
  \le e^{\diffp_0} \delta_0,
\end{align*}
and so combining the two guarantees gives
\begin{equation*}
  \P(M(P_n) \in B) \le e^{\diffp_0 + \diffp} \P(M(P_n') \in B)
  + e^{\diffp_0} \delta_0 + \delta.
\end{equation*}

Lastly, we consider the case that $B$ may contain
$\perp$. Note that
\begin{equation*}
  \P(M(P_n) = \perp)
  = \P(M_0(P_n) \not \in A)
  \le e^{\diffp_0} \P(M_0(P_n') \not\in A) + \delta_0
  = e^{\diffp_0} \P(M(P_n') = \perp) + \delta_0.
\end{equation*}
Combining this display with the preceding derivation gives
the result.

\subsubsection{Proof of Proposition~\ref{proposition:release-once-ratio}}
\label{sec:proof-release-once-ratio}

Recapitulating a few results on Gaussian closeness, we say random variables
$X$ and $Y$ satisfy
\begin{equation*}
  X \eqdiffp Y ~~~ \mbox{if} ~~~
  \P(X \in A) \le e^\diffp \P(Y \in A) + \delta
  ~~ \mbox{and} ~~
  \P(Y \in A) \le e^\diffp \P(X \in A) + \delta
\end{equation*}
for all measurable $A$. The following lemma gives sufficient
conditions for closeness of
Gaussian distributions, where
we recall the nuclear norm $\norm{A}_* = \sum_i \sigma_i(A)$,
Mahalanobis norm $\norm{v}_\Sigma^2 = v^T \Sigma^{-1} v$, and
use the distance-like
function on positive definite
matrices
\begin{equation*}
  \dpd(A, B) = \max\left\{\norm{A^{-1/2} (B - A) A^{-1/2}}_*,
  \norm{B^{-1/2} (A - B) B^{-1/2}}_* \right\}.
\end{equation*}
We also recall $\varpriv$, the
variance~\eqref{eqn:private-normal-variance} necessary for Gaussians to
provide $(\diffp, \delta)$-privacy.

\begin{lemma}
  \label{lemma:gaussian-closeness}
  Let $\diffp, \delta > 0$,
  and let $X \sim \normal(\mu_1, \Sigma_1)$ and
  $Y \sim \normal(\mu_2, \Sigma_2)$.
  Then $X$ and $Y$ satisfy $X \eqdiffp Y$ in the following
  cases.
  \begin{enumerate}[i.]
  \item If $\Sigma_1 = \Sigma_2 = \sigma^2 \Sigma$,
    where
    $\sigma \ge \stdpriv \norm{\mu_1 - \mu_2}_\Sigma$.
  \item If $\mu_1 = \mu_2$ and $\diffp \ge 6 \dpd(\Sigma_1,
    \Sigma_2) \log \frac{2}{\delta}$.
  \item  \label{item:one-dim-variance-private}
    In one dimension if $\Sigma_1 = \sigma_1^2$,
    $\Sigma_2 = \sigma_2^2$, and $\mu_1 = \mu_2 = \mu \in \R$,
    then $X \eqdiffp Y$ if
    $\diffp \ge \half(1 + \Phi^{-1}(1 - \delta/2)^2)
    \dpd(\sigma_1^2, \sigma_2^2)$.
  \end{enumerate}
\end{lemma}
\begin{proof}
  The first claim is a trivial modification
  of Lemma~\ref{lemma:private-normal-variance}.
  For the second, see, e.g., \cite[Lemmas
    2.5--2.6]{DuchiHaKu23}. For the last (in one dimension), w.l.o.g.\ let
  $\mu = 0$, let $p_1, p_2$ denote the densities of $X$ and $Y$, and
  consider $X \sim \normal(0, \sigma_1^2)$; it suffices to show that $|\log
  \frac{p_1(X)}{p_2(X)}| \le \diffp$ with probability at least $1 -
  \delta$. To that end, note that
  \begin{align*}
    \left|\log\frac{p_1(x)}{p_2(x)}\right|
    & =
    \left|\half \log \frac{\sigma_2^2}{\sigma_1^2}
    + \half x^2 \left(\frac{1}{\sigma_2^2} - \frac{1}{\sigma_1^2}
    \right)\right| \\
    & \le
    \half \max\left\{\log\frac{\sigma_2^2}{\sigma_1^2},
    \log \frac{\sigma_1^2}{\sigma_2^2} \right\}
    + \half \frac{x^2}{\sigma_1^2}
    \left|\frac{\sigma_1^2}{\sigma_2^2} - 1 \right| \\
    & \le \half \max \left\{\frac{\sigma_2^2}{\sigma_1^2} - 1,
    \frac{\sigma_1^2}{\sigma_2^2} - 1 \right\}
    + \half \frac{x^2}{\sigma_1^2}
    \left|\frac{\sigma_1^2}{\sigma_2^2} - 1 \right|,
  \end{align*}
  where we use that $\log t = \log(1 + t - 1) \le t - 1$ for all $t
  \ge 0$.
  As $X / \sigma_1 \sim \normal(0, 1)$,
  it becomes sufficient to upper bound
  $\half \dpd(\sigma_1^2, \sigma_2^2)
  + \half Z^2 \dpd(\sigma_1^2, \sigma_2^2)$ for
  $Z \sim \normal(0, 1)$. But of course, we have
  $|Z| \le \Phi^{-1}(1 - \delta/2)$ with probability
  at least $1 - \delta$, giving the result.
\end{proof}

We leverage part~\ref{item:one-dim-variance-private} of
Lemma~\ref{lemma:gaussian-closeness} to prove the proposition once the ratio
of the directional modulus~\eqref{eqn:directional-modulus-bound} quantities
$\stdonedim$ and $\stdonedim[P_n']$ is bounded.
Let $\sigma^2$ be as in the statement of the proposition, and 
define random variables
\begin{align*}
  Z_0 \sim \normal(u^T \theta(P_n), \sigma^2 \cdot \stdonedim^2),
  ~~
  & Z_1 \sim \normal(u^T \theta(P_n'), \sigma^2 \cdot \stdonedim[P_n]^2), \\
  & Z_2 \sim \normal(u^T \theta(P_n'), \sigma^2 \cdot \stdonedim[P_n']^2).
\end{align*}
We know that
because $\what{\lambda} \le \lambdamin(P_n)$, we have
$|u^T\theta(P_n) - u^T\theta(P_n')|
\le \stdonedim$, and so
\begin{equation*}
  Z_0 \eqdiffp Z_1
\end{equation*}
by Lemma~\ref{lemma:gaussian-closeness}.
By assumption
on the ratio~\eqref{eqn:ratio-to-control},
we have $\dpd(\stdonedim^2, \stdonedim[P_n']^2)
\le \ratio^2(u, \what{\lambda})$, and applying
Lemma~\ref{lemma:gaussian-closeness}.\ref{item:one-dim-variance-private},
\begin{equation*}
  Z_1 \eqdiffp Z_2
\end{equation*}
so long as
$\diffp \ge \half (1 + \Phi^{-1}(1 - \delta/2)^2) \ratio^2(u,\what{\lambda})$.
By standard composition guarantees, we thus have
$Z_0 \eqdist_{2\diffp,\delta + e^\diffp \delta} Z_2$.
Applying Lemma~\ref{lemma:test-release} gives the proposition.

\subsection{Proof of Lemma~\ref{lemma:self-concordance}}
\label{sec:proof-self-concordance}

We prove each statement in the lemma in turn.
\begin{enumerate}[(i)]
\item Define
  $g(t) = \log f''(t)$. Then
  $g'(t) = \frac{f'''(t)}{f''(t)}$, so that
  $|g(t + s) - g(t)| = |\int_t^{t + s} g'(u) du|
  \le c|s|$, giving the inequality.
  The choices of $\concordantfunc$ are immediate.
\item This is standard~\cite[Eq.~(9.46)]{BoydVa04}.
  Without loss of generality, let $t = 0$.
  Self-concordance is equivalent to the statement that
  \begin{equation*}
    \left|\frac{d}{ds} \left(f''(s)^{-1/2}\right)\right| \le 1,
    ~~~ \mbox{as} ~~~
    \frac{d}{ds} \left(f''(s)^{-1/2}\right)
    = \half \frac{f'''(s)}{f''(s)^{3/2}},
  \end{equation*}
  and the latter quantity has magnitude at most $1$ if and only if
  $f$ is self-concordant.
  Assuming w.l.o.g.\ that $s \ge 0$, then
  integrating from $0$ to $s$ then yields
  \begin{equation*}
    -s \le \frac{1}{\sqrt{h''(s)}} - \frac{1}{\sqrt{h''(0)}}
    \le s.
  \end{equation*}
  Solving the upper and lower bounds gives
  claim~\eqref{item:self-concordant-hessian}.
\item This is immediate from the upper bound of
  part~\eqref{item:self-concordant-hessian}.
\end{enumerate}


\subsection{Proofs about recursions}

\subsubsection{Proof of Lemma~\ref{lemma:accelerating-recursions}}
\label{sec:proof-acceleration}

That each satisfies $\recurse(\lambda) \le \lambda$ is
immediate by convexity: we have
$\sqrt{1 - \delta} \le 1 - \delta/2$ and $\exp(\delta) \ge 1 + \delta$,
respectively.

For $\recurse(\lambda) = \frac{\lambda}{2}
+ \half \sqrt{\lambda^2 - a} - b$,
we observe that
$\recurse'(\lambda)
= \half + \frac{\lambda}{2 \sqrt{\lambda^2 - a}}
\ge 1$ for all $\lambda \ge a$.

\newcommand{\lambdazero}{\lambda_0}

Now we show that for any $\lambdazero \ge 0$,
\begin{equation*}
  \recurse(\lambda) \defeq \lambda
  \left[2 - \exp\left(b\left(1 - \sqrt{1 - \frac{a}{\lambda + \lambdazero}}
    \right)\right)\right] - c
\end{equation*}
is an accelerating recursion, for which it suffices to show that
$\recurse'(\lambda) \ge 1$ for all $\lambda + \lambdazero \ge a$.  Taking
derivatives and using that $\frac{\partial}{\partial \lambda} b(1 - \sqrt{1
  - a/(\lambda + \lambdazero)}) = -\frac{ab}{2 (\lambda + \lambdazero)^2
  \sqrt{1 - a/(\lambda + \lambdazero)}}$, we have
\begin{align*}
  \lefteqn{\recurse'(\lambda)} \\
  & = 2 - \exp\left(b(1 - \sqrt{1 - a/(\lambda + \lambdazero)})\right)  
  + \exp\left(b(1 - \sqrt{1 - a/(\lambda + \lambdazero)})\right)
  \frac{b a}{2 (\lambda + \lambdazero) \sqrt{1 - a/(\lambda + \lambdazero)}} \\
  & =
  2 + \exp\left(b(1 - \sqrt{1 - a/(\lambda + \lambdazero)})\right)
  \left[\frac{b a}{2 (\lambda + \lambdazero) \sqrt{1 - a/(\lambda + \lambdazero)}} - 1 \right].
\end{align*}
Let $\delta = \frac{a}{\lambda + \lambdazero} < 1$ (as $\lambda +
\lambdazero > a$). Then $\recurse'(\lambda) \ge 1$ if and only if
\begin{align*}
  \frac{b \delta}{2 \sqrt{1 - \delta}} - 1
  & \ge -\exp\left(-b(1 - \sqrt{1 - \delta})\right)
  ~~ \mbox{if~and~only~if} \\
  \exp\left(-b(1 - \sqrt{1 - \delta})\right)
  + \frac{b \delta}{2 \sqrt{1 - \delta}} & \ge 1.
\end{align*}
At $\delta = 0$ this inequality trivially holds. Define $f(\delta) =
\exp(-b + b\sqrt{1 - \delta})) + \frac{b \delta}{2 \sqrt{1 -
    \delta}}$. Then
\begin{equation*}
  f'(\delta) = \exp\left(-b(1 - \sqrt{1 - \delta})\right)
  \left(\frac{-b}{2 \sqrt{1 - \delta}}\right)
  + \frac{b}{2 \sqrt{1 - \delta}}
  + \frac{b \delta}{4 (1 - \delta)^{3/2}}
  > \frac{b \delta}{4(1 - \delta)^{3/2}} \ge 0,
\end{equation*}
so $f(\delta) \ge 1$ for all $\delta = \frac{a}{\lambda} \in [0, 1]$,
and $\recurse'(\lambda) \ge 1$.

\subsection{Proofs about matrix ratios}
\label{sec:proofs-ratio-stuff}

\providecommand{\symm}{\mathbf{S}}
\providecommand{\neghinge}[1]{\left[{#1}\right]_-}

We consider a few technical results that form
useful building blocks for many of our results. Throughout,
we let $\symm^d = \{A \in \R^{d \times d} \mid A = A^T\}$ denote the
symmetric matrices. The basic
form of results in this section is as follows:
for a (symmetric, positive definite) matrix $X$
belonging to a set $\mc{C}$ of PSD matrices, we wish to provide
bounds on quantities of the form
\begin{equation}
  \label{eqn:non-symmetric-quadratic}
  \sup_{X \in \mc{C}} \<u, X v\>
  ~~ \mbox{and} ~~
  \sup_{X \in \mc{C}} \norm{Xu}
\end{equation}
where $u$ and $v$ are given vectors. For a symmetric matrix $A \in \symm^d$,
we let $\hinge{A}$ be its Euclidean projection onto the positive
semidefinite matrices, so that if $A = U\Lambda U^T$
with $\Lambda = \diag(\lambda)$, then $\hinge{A} = U
\hinge{\Lambda} U^T = U \diag(\hinge{\lambda}) U^T$. Similarly,
we let $\neghinge{A} = -\hinge{-A}$ be the Euclidean
projection of $A$ onto the negative semidefinite matrices, or its negative
semidefinite part, so that $A = \hinge{A} - \hinge{-A}
= \hinge{A} + \neghinge{A}$.

\subsubsection{Suprema of matrix inner products with semidefinite box
  constraints}

Our main focus is on situations where the set $\mc{C}$ is of the form
$\mc{C} = \{X \in \symm^d \mid A \preceq X \preceq B\}$. In this case, the
following lemma provides guidance in the solution of the first problem
in~\eqref{eqn:non-symmetric-quadratic}. In the lemma, we say that a matrix
$X$ is invariant in the eigenspaces of $Y$ and $Z$ if for the spectral
decompositions $Y = U \Lambda U^T$ and $Z = V D V^T$, where we include only
the nonzero eigenvalues in $\Lambda$ and $D$, we have $X UU^T = UU^T X UU^T$
and $X VV^T = VV^T X VV^T$.
\begin{lemma}
  \label{lemma:general-supremum-trace}
  Let $A \preceq B$ be symmetric matrices and
  $C \in \R^{d \times d}$.
  Then
  \begin{equation*}
    \sup_{A \preceq X \preceq B}
    \tr(X C) =
    \inf \left\{\half \<B, C_+\> - \half \<A, C_-\>
    \mid C_+ \succeq 0, C_- \succeq 0,
    \half(C + C^T) = C_+ - C_- \right\}.
  \end{equation*}
  Additionally,
  \begin{equation*}
    \sup_{A \preceq X \preceq B} \tr(XC)
    \le
    \half \<B, \hinge{C + C^T}\>
    + \half \<A, \neghinge{C + C^T}\>,
  \end{equation*}
  and equality holds if $A$ and $B$ are invariant in the eigenspaces
  of $\hinge{C + C^T}$ and $\neghinge{C + C^T}$.  
\end{lemma}
\begin{proof}
  Without loss of generality assume that $C = C^T$, because $\tr(XC) =
  \tr(XC^T)$ for $X$ symmetric; otherwise we simply replace
  $C$ with its symmetrization $\half(C + C^T)$.
  Introduce Lagrange multipliers $Y, Z \succeq 0$
  for the constraints
  $A \preceq X$ and $X \preceq B$, respectively. Then
  \begin{equation*}
    \mc{L}(X, Y, Z) = \<X, C\> + \<Y, X - A\>
    + \<Z, B - X\>
  \end{equation*}
  satisfies
  \begin{equation*}
    \sup_{X \in \symm^d}
    \mc{L}(X, Y, Z) = \begin{cases}
      \<B, Z\> - \<A, Y\> & \mbox{if} ~ C = Z - Y \\
      +\infty & \mbox{otherwise}.
    \end{cases}
  \end{equation*}
  As $Y \succeq 0$, $Z \succeq 0$, this gives associated dual problem
  \begin{equation}
    \label{eqn:funky-sdp-dual}
    \begin{array}{ll} \minimize & \<B, Z\> - \<A, Y\> \\
      \subjectto & Z \succeq 0, Y \succeq 0, C = Z - Y.
    \end{array}
  \end{equation}
  There exist $Z, Y \succ 0$ satisfying the constraints of the dual---so
  that Slater's condition holds---and strong duality obtains for the
  problem~\eqref{eqn:funky-sdp-dual}. The first claim of the lemma follows.

  Certainly the choices $Z = \hinge{C}$ and $Y = -\neghinge{C}$
  are feasible for the dual, giving the second
  claim of the lemma, though they may be suboptimal.
  For the special case attaining equality, let $C = U \Lambda U^T
  = C_+ - C_-$, where $C_+ \succeq 0$ and $C_- \succeq 0$ decompose
  $C$ into its positive and negative eigenvalues. Let
  $U_+$ and $U_-$ be the eigenvectors associated with the
  positive and negative eigenvalues of $C$, and let
  $\Pi_+ = U_+ U_+^T$ and $\Pi_- = U_- U_-^T$ be the
  associated projection matrices, and let $\Pi_0$ be the orthogonal
  projector to the null space of $\Pi_+ + \Pi_-$. If
  $A$ and $B$
  are invariant in these eigenspaces, in that
  $A \Pi_+ = \Pi_+ A \Pi_+$ and
  $A \Pi_- = \Pi_- A \Pi_-$ (and similarly for $B$),
  then the eigenspace invariances imply that
  \begin{equation*}
    X = \Pi_+ B \Pi_+ + \half \Pi_0(A + B) \Pi_0 + \Pi_- A \Pi_-
  \end{equation*}
  satisfies $A \preceq X \preceq B$ because
  we can write $A = \Pi_+ A \Pi_+ + \Pi_0 A \Pi_0 + \Pi_- A \Pi_-$, and
  similarly for $B$. Notably,
  the choices $Z = \hinge{C}$ and $Y = -\neghinge{C}$
  are feasible in the dual~\eqref{eqn:funky-sdp-dual},
  and
  \begin{equation*}
    \<C, X\> = \<\Pi_+ B \Pi_+, C\>
    + \<\Pi_- A \Pi_-, C\>
    = \<B, \hinge{C}\> + \<A, \neghinge{C}\>
    = \<B, Z\> - \<A, Y\>,
  \end{equation*}
  showing equality in the dual.
\end{proof}

\begin{lemma}
  \label{lemma:rank-one-supremum-trace}
  Let $A \preceq B$ be symmetric and $u, v$ be vectors.
  Then
  \begin{equation*}
    \sup_{A \preceq X \preceq B}
    u^T X v
    \le \frac{1}{4} \left[\<B, (u + v)(u + v)^T\>
      - \<A, (u - v) (u - v)^T\>\right].
  \end{equation*}
\end{lemma}
\begin{proof}
  Note that
  \begin{equation*}
    \half (uv^T + vu^T)
    = \frac{1}{4} \left[(u + v)(u + v)^T - (u - v)(u - v)^T\right],
  \end{equation*}
  a difference of the positive semidefinite matrices
  $(u + v)(u + v)^T$ and $(u - v) (u - v)^T$.
  Apply Lemma~\ref{lemma:general-supremum-trace}.
\end{proof}

As a consequence of Lemma~\ref{lemma:rank-one-supremum-trace},
we have the following inequality.
\begin{lemma}
  \label{lemma:rank-one-supremum-trace-normalized}
  Let $A \preceq B$ be symmetric and $u, v$ be vectors. Define
  $u_0 = u / \ltwo{u}$ and $v_0 = v / \ltwo{v}$. Then
  \begin{equation*}
    \sup_{A \preceq X \preceq B}
    u^T X v
    \le \ltwo{u} \ltwo{v}
    \frac{1}{4} \left[\<B, (u_0 + v_0)(u_0 + v_0)^T\>
      - \<A, (u_0 - v_0) (u_0 - v_0)^T \right].
  \end{equation*}
  If $u_0 + v_0$ and $u_0 - v_0$ are
  eigenvectors of $A$ and $B$, then equality holds.
\end{lemma}
\begin{proof}
  By homogeneity we have
  \begin{equation*}
    u^T X v = \ltwo{u} \ltwo{v} (u / \ltwo{u})^T X (v / \ltwo{v})
    = \ltwo{u} \ltwo{v} u_0^T X v_0.
  \end{equation*}
  Now apply Lemma~\ref{lemma:rank-one-supremum-trace},
  but note that
  as $\ltwo{u_0} = \ltwo{v_0} = 1$, we have
  $(u_0 + v_0)^T(u_0 - v_0) = 1 - 1 = 0$, so that
  Lemma~\ref{lemma:general-supremum-trace}
  gives the result as
  $u_0 v_0^T + v_0 u_0^T$
  has eigedecomposition
  $\half (u_0 + v_0)(u_0 + v_0)^T
  - \half (u_0 - v_0)(u_0 - v_0)^T$.

  The equality claim similarly follows from
  Lemma~\ref{lemma:general-supremum-trace}.
\end{proof}

\subsubsection{Proof of Lemma~\ref{lemma:self-similar}}
\label{sec:proof-self-similar}

We have
\begin{align*}
  \ltwo{H_1^{-1} u}
  & \le \ltwo{H_0^{-1} u}
  + \ltwo{E_1 u} + \ltwo{E_2 u} \\
  & \le \ltwo{H_0^{-1} u}
  + \opnorm{E_1 H_0} \ltwo{H_0^{-1} u}
  + \opnorm{E_2 H_0} \ltwo{H_0^{-1} u}.
\end{align*}
We control the two error terms
$\opnorm{E_1 H_0}$ and $\opnorm{E_2 H_0}$ in turn.
Let $w$ and $v$ be unit vectors. Then normalizing by $\ltwo{H_0 v}$,
Lemma~\ref{lemma:rank-one-supremum-trace-normalized} gives
\begin{equation*}
  \sup_{-\simconst_1 H_0^{-1} \preceq E_1
    \preceq \simconst_1 H_0^{-1}} \frac{w^T E_1 H_0 v}{\ltwo{H_0 v}}
  \le \frac{\simconst_1}{2}
  \left[
    \<w, H_0^{-1} w\> + \frac{1}{\ltwo{H_0 v}^2}
    \<v, H_0 v\> \right],
\end{equation*}
so that
\begin{align*}
  \sup_{\ltwo{w} = 1, \ltwo{v} = 1}
  w^T E_1 H_0 v
  \le \sup_{\ltwo{w} = \ltwo{v} = 1}
  \frac{\ltwo{H_0 v} \simconst_1}{2}
  \left[\<w, H_0^{-1} w\> + \frac{1}{\ltwo{H_0 v}^2}
    \<v, H_0 v\>\right].
\end{align*}
Note that
$\ltwo{H_0 v} \le \opnorm{H_0}$ and
$\<H_0 v, v\> / \ltwo{H_0 v} \le 1$,
so taking a supremum over $w$ yields
\begin{equation*}
  \sup_{\ltwo{w} = 1, \ltwo{v} = 1}
  w^T E_1 H_0 v = \opnorm{E_1 H_0} \le
  \frac{\simconst_1}{2}\left(\opnorm{H_0^{-1}} \opnorm{H_0} + 1 \right)
  = \frac{\simconst_1}{2} \left(1 + \frac{\lambdamax(H_0)}{\lambdamin(H_0)}
  \right).
\end{equation*}

For the  term involving $\opnorm{E_2 H_0}$,
the naive bound
\begin{equation*}
  \sup_{-\Delta_0 \preceq E \preceq \Delta_1}
  \opnorm{E H_0}
  \le \max\{\opnorm{\Delta_1}, \opnorm{\Delta_0}\}
  \opnorm{H_0}
  \le \sup_{x \in \mc{X}}
  \ltwo{H_0^{-1} x}^2 \lambdamax(H_0)
\end{equation*}
suffices.

The proof of the lower bound is, \emph{mutatis mutandis}, identical.

\subsubsection{Proof of Lemma~\ref{lemma:relative-conditioning}}
\label{sec:proof-relative-conditioning}

For any unit vector $u$, we have
$\sup_{v \in \mc{V}} u^T A v
\ge \inscribed_2(\mc{V}) \ltwo{A u}
\ge \lambdamin(A)$. Simultaneously we have
$\sup_{v \in \mc{V}} v^T A v / \ltwo{v}
\le \sup_{v \in \mc{V}} \ltwo{A v}
\le \radius_2(\mc{V}) \lambdamax(A)$. This yields the
claimed upper bound.

For the equalities and near equalities, note that
if $\mc{V} = \ball_2^d$, then
$\sup_{v \in \mc{V}} v^T A v = \lambdamax(A)$, while
$\inf_{\ltwo{u} = 1} \sup_{v \in \mc{V}} u^T A v
= \inf_u \ltwo{A u} = \lambdamin(A)$.
When $\mc{V} = [-1, 1]^d$, take the matrix
$A = \frac{1}{d} \ones\ones^T + \lambda (I - \frac{1}{d}
\ones\ones^T)$. Then
$A$ has one eigenvalue $1$ associated to the eigenvector
$\ones / \sqrt{d}$, and the rest are all $\lambda$, so that
$\kappa(A) = 1/\lambda$.
Then
$v^T A v = \frac{1 - \lambda}{d} \<\ones, v\>^2
+ \lambda \ltwo{v}^2$,
and the supremum is achieved by $v = \ones$, yielding
$\sup_v v^T A v / \ltwo{v} = \sqrt{d}$.
Now, take $u = (e_1 - e_2) / \sqrt{2}$. Then
$\sup_{v \in \mc{V}} u^T A v = \lone{A u}
= \lambda \lone{u} = \lambda \sqrt{2}$. Thus
we have $\kappa(A, [-1, 1]^d) \ge
\sqrt{d} / (\lambda \sqrt{2})
= \frac{\radius_2(\mc{V})}{\inscribed_2(\mc{V})}
\kappa(A) / \sqrt{2}$.

\subsubsection{Proof of Lemma~\ref{lemma:self-similar-general}}
\label{sec:proof-self-similar-general}

Fix $v \in \mc{V}$. We control each of the error terms in the expansion
\begin{equation*}
  u^T H_1^{-1} v = u^T H_0^{-1} v + u^T E_1 v + u^T E_2 v.
\end{equation*}
For the first, we 
use Lemma~\ref{lemma:rank-one-supremum-trace-normalized}: for
$v_0 = v / \ltwo{v}$, we have
\begin{align*}
  \sup_{-\simconst_1 H_0^{-1} \preceq E_1 \preceq \simconst_1 H_0^{-1}}
  u^T E_1 v
  & \le
  \frac{\ltwo{v}}{4}
  \left[\<\simconst_1 H_0^{-1}, (u + v_0)(u + v_0)^T
    + (u - v_0)(u - v_0)^T\>
    \right] \\
  & = \frac{\simconst_1 \ltwo{v}}{2}
  \<H_0^{-1}, u u^T + v_0 v_0^T\>.
\end{align*}
By Lemma~\ref{lemma:relative-conditioning},
we have
\begin{equation*}
  v_0^T H_0^{-1} v
  \le \kappa(H_0^{-1}, \mc{V}) \cdot \sup_{v \in \mc{V}}
  u^T H_0^{-1} v
  \le \frac{\radius_2(\mc{V})}{\inscribed_2(\mc{V})}
  \kappa(H_0) \cdot \sup_{v \in \mc{V}}
  u^T H_0^{-1} v.
\end{equation*}
Noting that the set $\{v / \inscribed_2(\mc{V})
\mid v \in \mc{V} \} \supset \ball_2^d$, we have
\begin{equation*}
  u^T H_0^{-1} u \ltwo{v}
  \le \ltwo{v} \cdot \sup_{v \in \mc{V}} u^T H_0^{-1} v /
  \inscribed_2(\mc{V})
  \le \frac{\radius_2(\mc{V})}{\inscribed_2(\mc{V})}
  \sup_{v \in \mc{V}} u^T H_0^{-1} v.
\end{equation*}
Combining the preceding inequalities then gives that
\begin{equation*}
  u^T E_1 v \le \simconst_1 \frac{\radius_2(\mc{V})}{\inscribed_2(\mc{V})}
  \left(\frac{\kappa(H_0) + 1}{2} \right)
  \sup_{v \in \mc{V}} u^T H_0^{-1} v.
\end{equation*}

We now control the second error term. Noting that $x_0x_0^T + x_1 x_1^T
\succeq x_0 x_0^T$ and $x_0x_0^T + x_1 x_1^T \succeq x_0 x_0^T$ we have
(again applying Lemma~\ref{lemma:rank-one-supremum-trace-normalized})
\begin{align*}
  \sup_{E_2} u^T E_2 v
  & \le \frac{\simconst_2 \ltwo{v}}{2}
  \<H_0^{-1} (x_0 x_0^T + x_1 x_1^T) H_0^{-1},
  u u^T + v_0 v_0^T \> \\
  & \le \simconst_2 \ltwo{v}
  \left[\sup_{x \in \mc{X}}
    \<H_0^{-1} xx^T H_0^{-1}, u u^T\>
    + \sup_{x \in \mc{X}}
    \<H_0^{-1} xx^T H_0^{-1}, v_0 v_0^T\>
    \right] \\
  & = \simconst_2 \ltwo{v}
  \left[\sup_{x \in \mc{X}}
    (u^T H_0^{-1} x)^2
    + \sup_{x \in \mc{X}} (v_0^T H_0^{-1} x)^2 \right]
\end{align*}
Now we use that $\mc{V}$ coincides with a scaled multiple
of $\mc{X}$ to obtain
\begin{align*}
  \sup_{\ltwo{u} = 1} \sup_{x \in \mc{X}} u^T H_0^{-1} x
  = c^{-1} \sup_{\ltwo{u} = 1}
  \sup_{v \in \mc{V}} u^T H_0^{-1} v
  \le \frac{\radius_2(\mc{V})}{c \lambdamin(H_0)},
\end{align*}
and thus
\begin{align*}
  u^T E_2 v
  \le \simconst_2 \cdot \frac{2 \radius_2^2(\mc{V})}{c^2 \lambdamin(H_0)}
  \sup_{v \in \mc{V}} u^T H_0^{-1} v.
\end{align*}
Combining the inequalities gives the lemma.
The lower bound calculation is completely similar.

To see the sharpness conditions, take $H_0 = I$ and $\mc{V} = [-1, 1]^d$.
Then the constraint $-\simconst I \preceq E
\preceq \simconst I$ is equivalent to the constraint that $\opnorm{E} \le
\simconst$, and so using
Lemma~\ref{lemma:rank-one-supremum-trace-normalized} with $v_0 = v /
\ltwo{v}$, we have
\begin{equation*}
  \sup_{\opnorm{E} \le \simconst} u^T E v
  = \frac{\simconst \ltwo{v}}{2}
  \<I, uu^T + v_0 v_0^T\>
  = \simconst \ltwo{v},
\end{equation*}
because $u - v_0$ and $u + v_0$ are certainly eigenvectors of $I$
and $\<I, v_0 v_0^T\> = \<I, uu^T\> = 1$ for $\ell_2$-unit vectors
$u, v_0$. Note also that $\radius_2(\mc{V}) / \inscribed_2(\mc{V})
= \sup_{v \in \mc{V}} \ltwo{v} = \sqrt{d}$. Taking
$u$ to be any standard basis vector and $v = \ones$ then
yields $u^T H_0^{-1} v = u^T v = 1 = \sup_{v \in \mc{V}} u^T v$.
In particular, we can choose $E_1$ such that
\begin{equation*}
  u^T (H_0^{-1} + E_1) v = u^T H_0^{-1} v + \simconst_1 \ltwo{v}
  = (u^T v) (1 + \simconst_1 \sqrt{d})
  = \sup_{v \in \mc{V}} (u^T v)
  (1 + \simconst_1 \sqrt{d}).
\end{equation*}
To control the second error term involving $u^T E_2 v$
take $x = v = \ones$ and $u$ to be any vector with nonnegative entries, so that
$u^T E_2 v = u^T \ones d
= d \cdot \sup_{\linf{v} \le 1} u^T v$. Thus, we have exhibited
error matrices $E_1$ and $E_2$, when $H_0 = I$, with
$-\simconst_1 H_0^{-1} \preceq E_1 \preceq \simconst_1 H_0^{-1}$ and
$-\simconst_2 xx^T \preceq E_2 \preceq \simconst_2 xx^T$ such that
\begin{equation*}
  \sup_{v \in \mc{V}}
  u^T H_1^{-1} v
  \ge \left(1 + \simconst_1 \sqrt{d} + \simconst_2 d \right)
  \sup_{v \in \mc{V}} u^T H_0^{-1} v
\end{equation*}
whenever $u$ is a standard basis vector.
As
$\radius_2^2(\mc{V}) = d$,
$\inscribed_2(\mc{V}) = 1$,
and $\kappa(H_0) = 1$ in this case, the proof is complete.

\bibliography{../bib}
\bibliographystyle{abbrvnat}

\end{document}